\documentclass[sigconf]{acmart}

\usepackage{fancyhdr}
\fancypagestyle{fancy}{%
  \fancyhf{}
  
}
\fancypagestyle{firstpagestyle}{%
  \fancyhf{}
  
}

\renewcommand\footnotetextcopyrightpermission[1]{}
\setcopyright{none}
\acmDOI{}
\acmISBN{}
\acmConference[KDD '26]{ACM SIGKDD Conference on Knowledge Discovery and Data Mining}{August 2026}{Jeju, Korea}
\acmYear{2026}
\copyrightyear{2026}

\usepackage{amsmath}

\usepackage{amssymb}
\usepackage{amsthm,mathtools,bm}

\usepackage{microtype}
\usepackage{booktabs}
\usepackage{graphicx}
\usepackage{subcaption}
\usepackage{enumitem}
\usepackage{xcolor}
\usepackage{multirow}
\newtheorem{definition}{Definition}
\newtheorem{lemma}{Lemma}
\newtheorem{theorem}{Theorem}
\newtheorem{remark}{Remark}

\newtheorem{condition}{Condition}

\newtheorem*{lemma*}{Lemma}
\newtheorem*{theorem*}{Theorem}

\newcommand{\NAME}{\textsc{ATLAS}}

\newcommand{\Xcal}{\mathcal{X}}
\newcommand{\Dcal}{\mathcal{D}}

\newcommand{\DeltaK}{\Delta^{K-1}}

\begin{document}

\title{Learning Demographic-Conditioned Mobility Trajectories\\with Aggregate Supervision}
\thanks{\textsuperscript{1} University of California, Berkeley;\;
\textsuperscript{2} Hong Kong University of Science and Technology (Guangzhou).}

\author{Jessie Zixin Li\textsuperscript{1}}
\email{zixin_li@berkeley.edu}
\affiliation{\city{Berkeley}
  \country{USA}}
\author{Zhiqing Hong\textsuperscript{1,2}}
\email{zhiqinghong@hkust-gz.edu.cn}
\affiliation{\city{Guangzhou}
  \country{China}}
\author{Toru Shirakawa\textsuperscript{1}}
\email{shirakawatoru@berkeley.edu}
\affiliation{\city{Berkeley}
  \country{USA}}
\author{Serina Chang\textsuperscript{1}}
\email{serinac@berkeley.edu}
\affiliation{\city{Berkeley}
  \country{USA}}

\begin{abstract}
Human mobility trajectories are widely studied in public health and social science, where different demographic groups exhibit significantly different mobility patterns.
However, existing trajectory generation models rarely capture this heterogeneity because most trajectory datasets lack demographic labels. 
To address this gap in data, we propose \NAME{}, a weakly supervised approach for demographic-conditioned trajectory generation using only (i) individual trajectories without demographic labels, (ii) region-level aggregated mobility features, and (iii) region-level demographic compositions from census data. \NAME{} trains a trajectory generator and fine-tunes it so that simulated mobility matches observed regional aggregates while conditioning on demographics. 
Experiments on real trajectory data with demographic labels show that \NAME{} substantially improves demographic realism over baselines (JSD ↓ 12\%–69\%) and closes much of the gap to strongly supervised training.
We further develop theoretical analyses for when and why \NAME{} works, identifying key factors including demographic diversity across regions and the informativeness of the aggregate feature, paired with experiments demonstrating the practical implications of our theory. We release our code at https://github.com/schang-lab/ATLAS.
\end{abstract}


\maketitle
\thispagestyle{plain}
\pagestyle{plain}

\section{Introduction}
\begin{figure*}
    \centering
    \includegraphics[width=\linewidth]{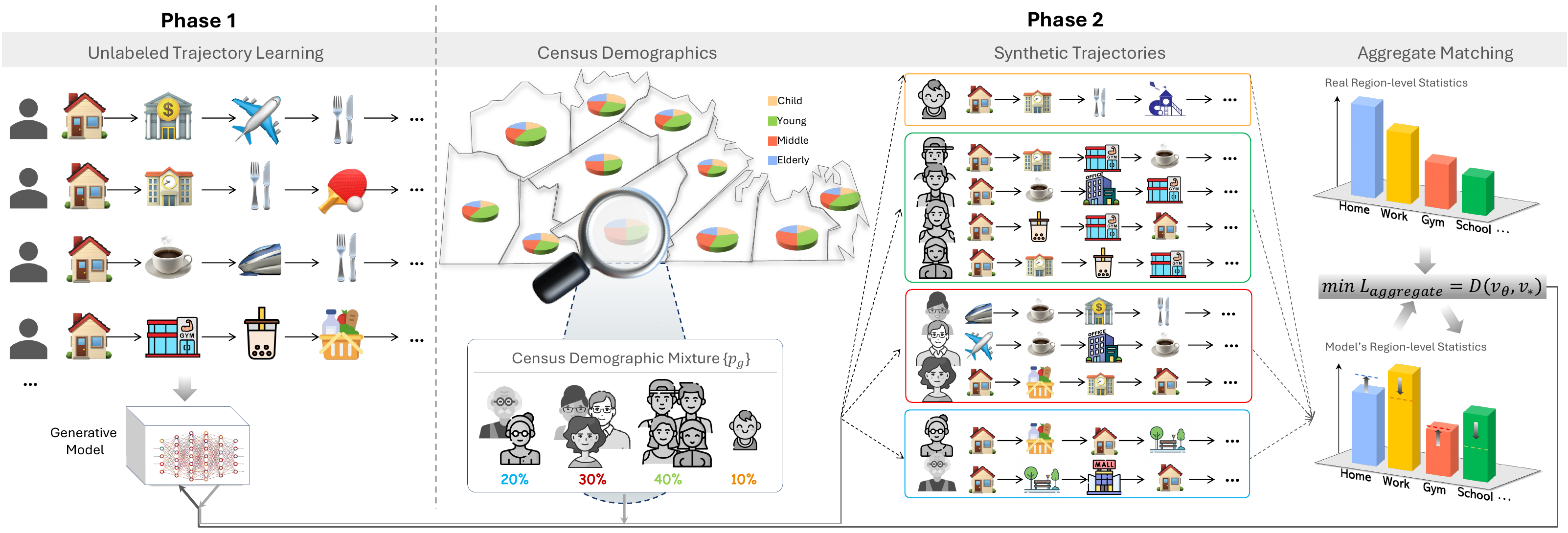}
    \caption{Overview of \NAME{}. Phase 1: Train a generative model on trajectories without demographic labels. Phase 2: Fine-tune with demographic conditioning by sampling groups from the region's demographic composition $p(\cdot\mid g)$ and optimizing to match region's observed aggregate features $\nu_\star(g)$.}
    \label{fig:overview}
\end{figure*}

Understanding human mobility is crucial to many societal problems, including modeling infectious diseases \citep{chang2021mobility}, designing transportation infrastructure \citep{pappalardo2023future}, and measuring social mixing \citep{nilforoshan2023segregation}. 
Mobility trajectories, i.e., a timestamped sequence of places visited by an individual, provide especially granular information about which people and places this individual encountered, enabling fine-grained health and behavioral estimates.
Since these trajectories are difficult to collect and share, researchers have developed generative models of mobility trajectories \citep{zhu2023difftraj, yuan2025learning}.
Synthetic trajectories, generated by these models, have become an important tool for privacy-preserving sharing, simulation, and policy evaluation \citep{osorio2024privacy, feng2020learning, li2021impact}.

However, existing trajectory generation models struggle to capture demographic heterogeneity.
Mobility varies significantly across demographic groups: for example, across age groups, such as students who go to school, working adults who go to their offices, and retirees who may spend more time in their residential communities.
These differences in mobility have important downstream implications, contributing to health disparities \citep{chang2021mobility}, differential access to resources \citep{xu2025using}, and socioeconomic segregation \citep{nilforoshan2023segregation}.
Existing trajectory generation models fail to capture this demographic heterogeneity due to a lack of data linking individual mobility trajectories with demographics. 
Publicly available trajectory datasets lack ground-truth demographics, including GeoLife \citep{zheng1geolife}, YJMob100K \citep{yabe2024yjmob100k}, and Veraset \citep{veraset2022visits} (Section~\ref{sec:app-datasets}), due to privacy constraints and challenges in data collection, as mobility trajectories are typically passively collected but collecting ground-truth demographics requires surveys. 

\paragraph{The present work.}
Given the importance of capturing demographic heterogeneity yet this gap in data, in this work we propose a novel solution: \textbf{a weakly supervised approach that leverages available \textit{aggregate} data to learn demographic-conditioned trajectory generation}.
Specifically, we leverage three key ingredients that are more readily available: (1) individual-level trajectories without demographic information, such as the aforementioned datasets, (2) aggregated mobility features released at the regional level, such as total visits to each point-of-interest (POI) \citep{advan2025foottraffic}, and (3) 
the composition of demographics within each region, typically available from census data \cite{uscensus2022agesexa}.
We present our method \NAME{} (tr\textbf{A}jec\textbf{T}ory \textbf{L}earning from \textbf{A}ggregate\textbf{S}), which uses these three ingredients to learn demographic-conditioned trajectory generation.
\NAME{} proceeds in two phases: first, we train a generative model on the trajectories without demographic information, then we add demographic conditioning to the model and fine-tune it to minimize the loss between the regional aggregated features provided in the data vs. generated by the model. 
Notably, our method is model-agnostic and could apply to any trajectory generation model, including diffusion models \citep{zhu2023difftraj,song2024controllable,zhu2024controltraj,wei2024diff,guo2025leveraging}, large language models \citep{li2024geo,li2024poi,wang2024urban}, and variational autoencoders \citep{long2023practical,huang2019variational}.

We begin by presenting theoretical foundations for when and why \NAME{} succeeds.
Our analysis identifies two key factors that influence the ability of \NAME{} to learn the true demographic-conditioned trajectory distributions from regional aggregated features: first, the degree to which demographic compositions differ across regions, and second, the informativeness of the feature map (e.g., POI visit counts) applied to each trajectory before aggregation.
We analyze the relationship between these factors and \NAME{}' success, and support these theoretical analyses with experiments where we systematically vary the demographic diversity of the regions or the choice of the feature. 
These analyses provide a principled foundation revealing when \NAME{} is expected to work well, with actionable guidance for practitioners.

To empirically test \NAME{}, we utilize real mobility trajectories linked to ground-truth demographics, such as age and gender, from the private Embee dataset \citep{bouzaghrane2025tracking}. 
In our experiments, we instantiate \NAME{} using a BART autoencoder-diffusion model architecture as a state-of-the-art approach for trajectory generation \citep{lewis2020bart, ho2022classifier, lovelace2023latent, guo2025leveraging}.
We compare \NAME{} to the baseline, a model trained on trajectories without demographic information, and the ceiling, a strongly supervised model that is directly trained on demographic-trajectory pairs.
Using the Embee dataset, we show that \NAME{} strongly outperforms the baseline, reducing Jensen-Shannon Divergence (JSD) to the real trajectories per demographic group by 12\%-69\% across trajectory statistics. 
Furthermore, \NAME{} approaches the performance of the strongly supervised model, frequently closing the gap between the baseline and strongly supervised model.
These results demonstrate that, in the absence of linked trajectory and demographics data, our weakly supervised approach can improve significantly at capturing real demographic mobility patterns, and can come close to the ideal strongly supervised case. 

In summary, our work presents the following contributions:
\begin{itemize}[leftmargin=*, itemsep=0.2em]
  \item \textbf{Learning from aggregates.} We introduce \NAME{}, a model-agnostic framework for learning individual demographic-conditioned trajectory generation from region-level aggregate mobility features and region-level demographic compositions (Section~\ref{sec:overview}).
  \item \textbf{Theoretical foundation.} We provide theoretical analyses of when and why \NAME{} succeeds, identifying two key factors: (i) the diversity in demographic compositions across regions and (ii) the informativeness of the aggregate mobility feature (Section~\ref{sec:theory}).
  \item \textbf{Experiments with real trajectories and demographics.} Using a dataset of real mobility trajectories linked to demographics, we show that \NAME{} significantly outperforms models not conditioned on demographics (by 12--69\%) and approaches the performance of a strongly supervised model (Sections~\ref{sec:empirical}-\ref{sec:experiments}).
\end{itemize}
Our work reconciles the lack of demographic information in real-world mobility data with the need to incorporate demographic heterogeneity into mobility-based models.
By providing a new avenue to learn demographic distributions from aggregate data, our work balances data constraints with decision-making needs, enabling more accurate models and equitable decisions across domains. 
\section{Learning Demographic Distributions from Aggregates}
\label{sec:overview}

\subsection{Problem Definition}
We begin by formally defining the key components of our learning-from-aggregates framework.

\begin{definition}[Trajectory Space]
Let $(\Xcal,\mathcal{B})$ denote the trajectory space, where $\Xcal$ represents the space of mobility trajectories and $\mathcal{B}$ is the $\sigma$-algebra on $\Xcal$ (the collection of measurable subsets of trajectories).
A trajectory $x\in\Xcal$ is characterized by a sequence of visited locations, timestamps, and associated spatial attributes.
\end{definition}

\begin{definition}[Demographic Groups and Conditional Distributions]
Let $\Dcal=\{1,\dots,K\}$ index the set of demographic groups (e.g., age $\times$ gender bins).
For each demographic group $d\in\Dcal$, there exists a true conditional trajectory distribution $P_\star(\cdot\mid d)$ on $\Xcal$ that governs the mobility behavior of individuals in group $d$.
\end{definition}

\begin{definition}[Regional Partition and Demographic Composition]
Individuals are partitioned into $G$ geographic regions (e.g., Census Block Groups), indexed by $g\in\{1,\dots,G\}$. For each region $g$, we assume access to a \emph{known} demographic composition vector $p_g\in\DeltaK$ obtained from census-like sources, where $p_g(d) = p(d\mid g)$ denotes the population share of demographic group $d$ in region $g$, and $\DeltaK:=\{p\in\mathbb{R}_+^K:\sum_{d=1}^K p(d)=1\}$ is the $(K-1)$-dimensional probability simplex.
Furthermore, for each region $g$, $Q_\star(\cdot\mid g)$ on $\Xcal$ denotes the true trajectory distribution of individuals in region $g$.
\end{definition}



\begin{definition}[Aggregate Feature Map]
\label{def:agg-feature}
Let $\phi:\Xcal\to\mathbb{R}^m$ be a bounded measurable feature map that extracts statistics from a trajectory (e.g., POI visit counts). 
Then, for each region $g$, define the region-level aggregate
\[
\nu_\star(g) := \mathbb{E}_{X\sim Q_\star(\cdot\mid g)}[\phi(X)]\in\mathbb{R}^m.
\]

\end{definition}

\noindent \textbf{Problem statement.}
Given (i) observed individual trajectories without demographic labels $\{x_i \in \mathcal{X}\}_{i=1}^n$, (ii) known demographic compositions $\{p(\cdot\mid g)\}_{g=1}^G$ for each region, and (iii) observed regional aggregated features $\{\nu_\star(g)\}_{g=1}^G$, our goal is to learn demographic-conditioned trajectory generators $\{P_\theta(\cdot\mid d)\}_{d=1}^K$ that are close to the true demographic-conditioned trajectory distributions $\{P_\star(\cdot\mid d)\}_{d=1}^K$, without access to individual-level demographic labels.


\subsection{Overview of \NAME{}} \label{sec:method}
Our method \NAME{} proceeds in two phases as shown in Fig.~\ref{fig:overview}: 
\begin{enumerate}[leftmargin=*, itemsep=0.2em]
  \item \textbf{Baseline.} Train the generative model on the individual trajectories $x_i$ without demographic labels. 
  While the baseline model is not conditioned on demographics, it may be conditioned on other individual features $z$, such as home or work locations, which are often provided in trajectory datasets (Section~\ref{sec:app-datasets}).
  Through this baseline training, the model learns $P_\theta(\cdot \mid z)$, providing a strong spatiotemporal backbone.
  \item \textbf{Aggregate supervision.} Extend the model to learn $P_\theta(\cdot \mid d, z)$, by fine-tuning the model using the region-level demographic compositions $p_g$ and region-level aggregate features $\nu_\star(g)$.
\end{enumerate}

In phase 2, we conduct aggregate supervision through the feature map $\phi:\Xcal\to\mathbb{R}^m$ that extracts summary statistics from trajectories.
For each region $g$, we compute the region-level aggregates $\nu_\theta(g)$ implied by model parameters $\theta$ by sampling from the model's demographic-conditioned distribution $P_\theta(\cdot \mid d, z)$, based on that region's known demographic composition (see Section~\ref{sec:model-training} for details). 
We seek model parameters $\theta$ that minimize the distance between the model-implied and ground-truth region-level aggregates:
\[
\min_\theta \;\; \mathcal{L}(\theta) = \sum_{g=1}^G \operatorname{dist}\bigl(\nu_\theta(g), \nu_\star(g)\bigr),
\]
where $\operatorname{dist}$ is a distributional distance metric. 
Importantly, \NAME{} is model-agnostic: any generative model that supports conditional sampling can use the same aggregate-supervision objective (e.g., diffusion models, autoregressive transformers, LLMs, VAEs, GANs).
The key requirement is the ability to sample from $P_\theta(\cdot\mid d,z)$ and compute feature expectations for gradient-based optimization.

In the following sections, we present model-agnostic theoretical analyses of \NAME{} (Section~\ref{sec:theory}), describe how we instantiated \NAME{} with a specific dataset and model architecture (Section~\ref{sec:empirical}), and conduct experiments with real trajectory data demonstrating the effectiveness of \NAME{} and testing the conditions identified in our theoretical analyses (Section~\ref{sec:experiments}).
\section{Theoretical Foundation}
\label{sec:theory}

In this section, we provide theoretical foundations for when and why our method \NAME{} succeeds.
Our analysis identifies two ideal conditions under which \NAME{} is guaranteed to recover the true demographic-conditioned trajectory distributions. 
While these ideal conditions may not hold in practice, they clarify the key factors that influence the success of \NAME{}: (i)~diversity in demographic compositions across regions and (ii)~the informativeness of the chosen feature map $\phi$.
These principles can be applied to understand when and why \NAME{} works, with instances closer to the ideal conditions resulting in greater success, as we show in Section~\ref{sec:experiments}.

\textbf{Preliminaries.}
In the theoretical analyses, we drop the additional individual features $z$ for simplicity and focus on demographic conditioning, but our analyses can be extended to incorporate these additional features. 
Consequently, we assume that $P_\star(X \mid d, g) = P_\star(X \mid d)$.
By the law of total expectation, the trajectory distribution for region $g$, $Q_\star(\cdot\mid g)$, becomes simply a weighted sum over demographic-conditioned distributions:
\[
Q_\star(\cdot\mid g) ~=~ \sum_{d=1}^K p(d\mid g)\, P_\star(\cdot\mid d).
\]
Similarly, the region-level aggregates $\nu_\star(g)$ becomes
\[
    \nu_\star(g) = \sum_{d=1}^K p(d\mid g)\,\mu_\star(d),
\]
where $\mu_\star(d) := \mathbb{E}_{X\sim P_\star(\cdot\mid d)}[\phi(X)]\in\mathbb{R}^m$ denotes the group-level feature mean.
Both $\nu_\star(g)$ and $\mu_\star(d)$ are expected values of the feature map $\phi$; we refer to the region-level quantities as \emph{aggregates} because they aggregate across demographic groups within a region, while the group-level quantities are \emph{feature means} for a single demographic group.

To simplify notation, we organize the vectors into matrix form.
Let $M_\star \in \mathbb{R}^{m \times K}$ denote the matrix of group-level feature means with columns $\mu_\star(d)$, and let $V_\star \in \mathbb{R}^{G \times m}$ denote the matrix of region-level aggregates with rows $\nu_\star(g)^\top$.
The demographic composition matrix $P \in \mathbb{R}^{G \times K}$ has rows $p_g^\top$, representing the demographic composition of each region; thus, $V_\star = P M_\star^\top$.
We use analogous notation $M_\theta$ and $V_\theta$ for model-implied quantities (see Table~\ref{tab:notation} in Appendix for a summary).


\subsection{Demographic Diversity Across Regions}
We begin by stating a key structural condition that enables us to recover group-level feature means from regional aggregates.

\begin{condition}[Demographic diversity across regions]
\label{ass:full-rank}
The demographic composition matrix $P \in \mathbb{R}^{G \times K}$ with rows $p_g^\top$ has full column rank:
\[
  \operatorname{rank}(P) = K.
\]
Equivalently, the set of regional demographic vectors $\{p_g\}_{g=1}^G$ contains $K$ linearly independent vectors.
\end{condition}
Intuitively, Condition~\ref{ass:full-rank} requires sufficient demographic heterogeneity across regions.
If all regions have nearly identical demographic compositions, $P$ is ill-conditioned and it becomes impossible to disentangle the contributions of different demographic groups from aggregates alone.
This condition is directly verifiable since $P$ is observed in our problem setting. 
Under Condition~\ref{ass:full-rank} alone, we can establish strong results about recovering demographic feature means from regional aggregates.

\begin{lemma}[Uniqueness of group-level feature means]
\label{lem:moment-identifiability}
Suppose Condition~\ref{ass:full-rank} holds. If the model-implied and true regional aggregates match, $\nu_\theta(g) = \nu_\star(g)$ for all regions $g$, then the demographic group-level feature means coincide:
\[
  \mu_\theta(d) = \mu_\star(d) \quad \text{for all } d.
\]
\end{lemma}

\begin{proof}[Proof sketch]
By construction, $V_\theta = P M_\theta^\top$ and $V_\star = P M_\star^\top$, where $M_\theta, M_\star\in\mathbb{R}^{m\times K}$ have columns $\mu_\theta(d), \mu_\star(d)$.
The condition $V_\theta = V_\star$ implies $P(M_\theta - M_\star)^\top = 0$.
Since $P$ has full column rank by Condition~\ref{ass:full-rank}, its null space is trivial, yielding $M_\theta = M_\star$.
\end{proof}

Lemma~\ref{lem:moment-identifiability} establishes that matching regional aggregates uniquely determines the group-level feature means $\{\mu_\star(d)\}$, provided the demographic composition matrix has sufficient diversity.
However, in practice we may not be able to match the regional aggregates perfectly. 
Next, we formalize how remaining distance in regional aggregates bound the distance in group-level feature means.

\begin{lemma}[Stability under aggregate perturbations]
\label{lem:stability}
Suppose Condition~\ref{ass:full-rank} holds. Then for any $V_1,V_2\in\mathbb R^{G\times m}$,
if $M_i^\top := P^\dagger V_i$, we have
\[
\|M_1 - M_2\|_F \le \frac{1}{\sigma_{\min}(P)}\|V_1 - V_2\|_F.
\]
\end{lemma}


\begin{proof}[Proof sketch]
The result follows from standard linear algebra. Since $M^\top = P^\dagger V$, where $P^\dagger$ is the Moore-Penrose pseudoinverse, the error in $M$ is bounded by $\|P^\dagger\|_2 \|V_1 - V_2\|_F$. The spectral norm $\|P^\dagger\|_2$ is exactly $1/\sigma_{\min}(P)$.
\end{proof}

\subsubsection{Finite sample bounds}
In practice, we do not have access to the true population expectations $V_*$. Instead, we observe empirical aggregates computed from finite samples in each region. We now quantify how sampling noise propagates to the recovery of demographic group-level feature means.

Let $n_g$ be the number of samples observed in region $g$. Let $\{x_{g,i}\}_{i=1}^{n_g} \stackrel{i.i.d.}{\sim} Q_*^{(g)}$ be the observed trajectories. The empirical regional feature vector is given by:
\begin{equation}
    \hat{v}(g) := \frac{1}{n_g} \sum_{i=1}^{n_g} \phi(x_{g,i}).
\end{equation}
We stack these vectors into the empirical matrix $\hat{V} \in \mathbb{R}^{G \times m}$. We assume the feature map is bounded, i.e., $\|\phi(x)\|_2 \le B$ almost surely.

\begin{lemma}[Finite Sample Error Bound]
\label{lem:finite-sample-bound}
    Suppose Condition~\ref{ass:full-rank} holds. Let $n_{\min} = \min_{g} n_g$ be the minimum sample size across regions. For any $\delta \in (0, 1)$, with probability at least $1 - \delta$, the error between the recovered demographic group-level feature means $\hat{M}$ (derived from $\hat{V}$) and the true means $M_*$ satisfies:
    \begin{equation}
        \|\hat{M} - M_*\|_F \le \frac{1}{\sigma_{\min}(P)} \left( \frac{B (\sqrt{m} + \sqrt{2\log(G/\delta)})}{\sqrt{n_{\min}}} \right) \cdot \sqrt{G},
    \end{equation}
    where $\sigma_{\min}(P)$ is the smallest singular value of the demographic composition matrix $P$.
\end{lemma}


\begin{theorem}[Overall bound (optimization + sampling)]
\label{thm:overall-bound}
Suppose Condition~\ref{ass:full-rank} holds.
Let $\hat V$ denote the observed (empirical) region-level aggregates and define
\[
\varepsilon_{\mathrm{opt}} := \|V_\theta - \hat V\|_F,
\qquad
\varepsilon_{\mathrm{samp}} :=
\left(
\frac{B(\sqrt{m}+\sqrt{2\log(G/\delta)})}{\sqrt{n_{\min}}}
\right)\sqrt{G}.
\]
Then for any $\delta\in(0,1)$, with probability at least $1-\delta$,
\begin{equation}
\label{eq:overall-bound}
\|M_\theta - M_\star\|_F
\;\le\;
\frac{\varepsilon_{\mathrm{opt}} + \varepsilon_{\mathrm{samp}}}{\sigma_{\min}(P)}.
\end{equation}
That is, the total error in recovering demographic group-level feature means is bounded by the sum of optimization error and sampling error in the regional aggregates, amplified by $1/\sigma_{\min}(P)$.
\end{theorem}

The bound highlights a fundamental constraint: recovery success is a joint function of data quality and regional structure. While increasing the sample size ($n_{\min}$) or improving the model fit ($\varepsilon_{\mathrm{opt}}$) reduces error, these improvements are linearly amplified by $1/\sigma_{\min}(P)$. This implies that even with "infinite" data and perfect optimization, demographic recovery is theoretically impossible if regions lack sufficient heterogeneity ($\sigma_{\min}(P) \to 0$).

\begin{remark}[Implications]
The results in this subsection depend only on Condition~\ref{ass:full-rank}, which is verifiable from observed data. Given a specific partitioning of regions and their demographic compositions:
\begin{itemize}
  \item One can compute $\sigma_{\min}(P)$ to assess whether the partition has sufficient demographic diversity.
  \item Lemma~\ref{lem:finite-sample-bound} provides a concrete bound on the error in recovering $\{\mu_\star(d)\}$ as a function of sample sizes $\{n_g\}$ and the feature bound $B$.
  \item These bounds can guide experimental design: e.g., choosing region granularity to maximize $\sigma_{\min}(P)$ while maintaining adequate sample sizes per region.
\end{itemize}
\end{remark}

\subsection{From Feature Means to Distributions}

Now, we move from matching group-level feature means to recovering the underlying group-level trajectory distributions.

\subsubsection{Recovery in the feature-induced metric without additional conditions.}
We begin by asking: what is the strongest notion of recovery that can be guaranteed
when supervision is provided only through aggregate feature constraints?
Previously, we analyzed the group-level feature means, showing that if Condition~\ref{ass:full-rank} holds, matching region-level aggregates ensures equality of group-level feature means.
However, without further assumptions on $\phi$, equality of feature
means does \emph{not} imply equality of the underlying distributions.
Thus, any recovery statement must be formulated relative to the
information encoded by $\phi$ itself.

To formalize this idea, we introduce a distributional metric that
measures discrepancies only through functions of $\phi$.

\begin{definition}[Feature-Induced Integral Probability Metric]
Let $\phi:\mathcal{X}\to\mathbb{R}^m$ be a measurable feature map.
Define the function class
\[
\mathcal{F}_\phi
=
\left\{
f_w(x) = \langle w,\phi(x)\rangle : \|w\|_2 \le 1
\right\}.
\]
The corresponding integral probability metric (IPM) is
\[
d_\phi(P,Q)
:=
\sup_{f\in\mathcal{F}_\phi}
\left|
\mathbb{E}_{X\sim P}[f(X)]
-
\mathbb{E}_{X\sim Q}[f(X)]
\right|.
\]
\end{definition}

This IPM compares distributions only through linear functionals of the
features $\phi(X)$, and therefore captures exactly the behavioral
information controlled by aggregate supervision.





We now state the formal recovery guarantee under aggregate supervision.

\begin{theorem}[Recovery up to $\phi$-IPM]
\label{thm:recovery-ipm}
Fix a feature map $\phi$.
If the model matches population-level regional aggregates,
$V_\theta^\phi = V_\star^\phi$,
then for every demographic group $d$,
\[
d_\phi\!\left(P_\theta(\cdot\mid d), P_\star(\cdot\mid d)\right)=0.
\]
\end{theorem}


Theorem~\ref{thm:recovery-ipm} characterizes what can be recovered from aggregate supervision
without additional identifiability assumptions on $\phi$.
Rather than claiming equality of full trajectory distributions, it
formalizes that matching regional aggregates ensures the model and true distributions
are indistinguishable under any linear functional \citep{rynne2008linear} of $\phi$.
Importantly, many empirical evaluation metrics (e.g., travel
distance) are themselves functionals of features similar to $\phi$, making this
characterization practically relevant.


\subsubsection{Full identifiability}
If we want to fully recover the group-level distributions, then we require an additional condition about the identifiability of the distributions from the feature map $\phi$.

\begin{condition}[$\phi$-identifiability]
\label{ass:phi-identifiable}
The feature map $\phi$ is rich enough that the conditional distributions are identifiable from their feature expectations: if two families
$\{P_d\}_{d=1}^K$ and $\{P'_d\}_{d=1}^K$ satisfy
\[
  \mathbb{E}_{X \sim P_d}\bigl[\phi(X)\bigr]
  \;=\;
  \mathbb{E}_{X \sim P'_d}\bigl[\phi(X)\bigr]
  \quad \text{for all } d \in \{1,\dots,K\},
\]
then $P_d = P'_d$ for all $d$.
\end{condition}


\begin{theorem}[Equivalence of aggregate matching and demographic recovery]
\label{thm:consistency-aggregates}
Fix a feature map $\phi$ and let $V_\theta, V_\star \in \mathbb{R}^{G\times m}$ denote the
population region-level aggregates with
\[
v_\theta(g)=\sum_{d=1}^K p_{g,d}\,\mu_\theta(d),
\qquad
v_\star(g)=\sum_{d=1}^K p_{g,d}\,\mu_\star(d),
\]
where $\mu_\theta(d)=\mathbb{E}_{X\sim P_\theta(\cdot\mid d)}[\phi(X)]$ and similarly for $\mu_\star(d)$.
Assume:
(i) $P$ has full column rank (Condition~\ref{ass:full-rank}), and
(ii) $\phi$ is identifying on the family of conditional distributions (Condition~\ref{ass:phi-identifiable}).
Then the following are equivalent:
\begin{enumerate}[label=(\roman*), leftmargin=*]
\item $V_\theta = V_\star$ (equivalently, $\nu_\theta(g)=\nu_\star(g)$ for all $g$);
\item $\mu_\theta(d)=\mu_\star(d)$ for all $d\in\{1,\dots,K\}$;
\item $P_\theta(\cdot\mid d)=P_\star(\cdot\mid d)$ for all $d\in\{1,\dots,K\}$.
\end{enumerate}
\end{theorem}

\begin{proof}[Proof sketch]
(i)$\Rightarrow$(ii): By Lemma~\ref{lem:moment-identifiability}.
(ii)$\Rightarrow$(iii): By Condition~\ref{ass:phi-identifiable}.
(iii)$\Rightarrow$(ii): Immediate from equality of distributions implies equality of expectations.
(ii)$\Rightarrow$(i): Plug $\mu_\theta=\mu_\star$ into $V=PM^\top$.
\end{proof}

Theorem~\ref{thm:consistency-aggregates} clarifies that matching regional aggregates (and thus, matching group-level feature means, under Condition~\ref{ass:full-rank}) recovers the group-level distributions if $\phi$ is identifying for the family of group-level trajectory distributions, $\mathcal{P}_j$.
While this is unlikely to hold for arbitrary distributions, identifiability is feasible if demographic heterogeneity is expressed primarily through low-order behavioral differences.
For example, if the group-level distribution takes the form $P_\star(x \mid d) \propto P_\star(x) \cdot \Pi_{t=1}^T \delta_{x_t,d}$, where $\delta_{x_t,d}$ are group-specific POI scaling factors and $P_\star(x)$ is the unconditional baseline, then POI visit counts form sufficient statistics and uniquely identify the demographic-conditioned distribution.

In the Appendix, we provide full proofs; additional identifiability examples, including minimal exponential families (Appendix~\ref{sec:app-identifiability-examples}); and an interpretation of our two-phase training procedure as an I-projection of the baseline, finding demographic-specific distributions that match the regional aggregate features but otherwise produce minimal changes to the baseline (Appendix~\ref{sec:i-projection}).

\subsection{Practical Implications}

Our theoretical results provide actionable guidance:

\textbf{Demographic diversity matters.} Performance depends on $\sigma_{\min}(P)$, which is directly computable from the demographic composition matrix. Partitions with diverse demographic compositions yield larger $\sigma_{\min}(P)$ and more stable recovery; practitioners can use this diagnostic to assess or optimize their regional design.
We test this in the following sections with experiments on real trajectories that we partition in different ways (Section~\ref{sec:demo-diversity-results}).

\textbf{Feature choice is critical.} 
The choice of feature map $\phi$ determines what aspects of the trajectories can be learned.
Furthermore, opportunities for identifiability are improved if we can assume that demographic heterogeneity manifests through low-order behavioral differences.
For a practitioner, this means that they should consider whether the key characteristics of trajectories that matter for their application can be captured through the feature map $\phi$ or functions of $\phi$.
We test this in the following sections with varying choices of $\phi$ that capture different aspects of the trajectories, such as POI marginals and category transitions (Section~\ref{sec:phi-results}).

\textbf{Importantly, we do not expect these ideal conditions to hold perfectly in practice.}
Rather, they characterize the theoretical limits of aggregate-based learning.
As we demonstrate empirically in Section~\ref{sec:phi-results}, \NAME{} performs increasingly better as the data moves closer to satisfying these conditions, and even when the conditions are not fully met, \NAME{} still improves substantially over baselines without demographic conditioning.


\section{Empirical Set-Up}
\label{sec:empirical}

In this section, we describe our concrete instantiation of \NAME{}, with a specific mobility dataset, model architecture, and model training procedure.
We provide additional details in Appendix~\ref{sec:app-empirical}.

\subsection{Mobility Dataset}
We use the Embee dataset \citep{bouzaghrane2025tracking}, which fuses passive POI check-in data from the SimilarWeb smartphone panel with longitudinal survey responses (six waves, Aug 2020--Sep 2022), providing self-reported demographics.
Check-ins are preprocessed into geographic ``places'' via DBSCAN clustering.
We focus on $K=8$ demographic groups ($4$ age bins $\times$ $2$ genders).
We also include individual home and work locations, which are obfuscated within a 1km radius for privacy protection, as additional individual features $z$ that the model conditions on during trajectory generation.
In our experiments, we use Embee data from two U.S. states---\textbf{Virginia (3,745 individuals, 76,999 trajectories)} and \textbf{California (9,114 individuals, 99,191 trajectories)}---enabling us to test \NAME{} across different environments and populations.
Since we have individual-level trajectories and demographic information, we can systematically test \NAME{} under different conditions, such as different regional partitions of individuals and different choices of aggregate feature.
Further details on data preprocessing and trajectory statistics are provided in Appendix~\ref{sec:app-datasets}.

\subsection{Model Architecture}
\label{sec:model-architecture}
We instantiate the trajectory generation model using a latent diffusion framework that decomposes generation into two phases: trajectory encoding and latent denoising. 
In the encoding step, we follow recent advances in latent diffusion models for sequential data \citep{lovelace2023latent,guo2025leveraging} and employ a BART autoencoder \citep{lewis2020bart} to map variable-length POI sequences into fixed-length continuous latent representations.
The encoded latent representations are then modeled via a Diffusion Transformer (DiT) \citep{peebles2023scalable}, which performs iterative denoising through a reverse diffusion process.
The DiT employs adaptive layer normalization to incorporate conditioning information, such as the home and work coordinates.
The architecture supports flexible conditioning, enabling the incorporation of demographic group embeddings during aggregate supervision. 

\subsection{Model Training}
\label{sec:model-training}
\subsubsection{Baseline training (phase 1).}
Following the two-phase approach outlined in Section~\ref{sec:method}, we first establish a baseline model that learns to generate trajectories conditioned on the individual's home and work locations, but not their demographics.
The BART autoencoder undergoes pretraining on individual trajectories via masked language modeling with span masking, learning to reconstruct trajectory sequences from partially observed inputs.
Subsequently, we train the DiT on latent representations extracted from the frozen autoencoder.
The DiT learns to reverse the forward diffusion process that progressively corrupts latent representations with Gaussian noise, using standard denoising objectives.

\subsubsection{Aggregate supervision (phase 2)}
To incorporate demographic conditioning, we fine-tune the baseline model using only region-level demographic compositions $\{p(\cdot\mid g)\}_{g=1}^G$ and region-level feature aggregates $\{\nu_\star(g)\}_{g=1}^G$.
For each batch, we: 
\begin{enumerate}[label=(\roman*), leftmargin=*, itemsep=0.1em]
  \item Sample a region $g$ uniformly from $\{1,\dots,G\}$;
  \item Sample demographic groups $d\sim p(\cdot\mid g)$ from the region's demographic composition;
  \item Generate synthetic trajectories $x\sim P_\theta(\cdot\mid d,z)$ via the reverse diffusion process, conditioned on demographic embedding $d$ and home/work coordinates $z$ sampled uniformly from training individuals in region $g$;
  \item Decode latent representations to POI sequences and compute empirical aggregates $\hat{\nu}_\theta(g)$ using the feature map $\phi$;
  \item Update $\theta$ via gradient descent to minimize the aggregate loss $\operatorname{dist}(\hat{\nu}_\theta(g), \nu_\star(g))$. We try both Jensen--Shannon \citep{lin2002divergence} and total variation divergences \citep{canonne2022short} in the aggregate loss.
\end{enumerate}

\subsubsection{Evaluation.}
We split trajectories at the user-level into 80\% train, 10\% validation, and 10\% test. In phase 1, we train the baseline model on trajectories from train users.
In phase 2, we use aggregated features within each region computed over trajectories from train users.
During evaluation, we compare our model's generated trajectories per demographic group to trajectories from test individuals in that demographic group.
\section{Experiments}
\label{sec:experiments}

We evaluate \NAME{} through extensive experiments designed to answer the following research questions:
\begin{itemize}[leftmargin=*, itemsep=0.2em]
    \item \textbf{RQ1:} How does diversity in demographic compositions across regions affect \NAME{} performance?
    \item \textbf{RQ2:} How does the choice of feature map $\phi$ affect what aspects of trajectory behavior can be learned from aggregates?
    \item \textbf{RQ3:} How well can \NAME{} improve on important downstream mobility tasks, such as next-POI prediction?
\end{itemize}

\subsection{Experimental Setup}


\subsubsection{Baselines and model variants}
We compare three training regimes, all sharing the same model architecture (Section~\ref{sec:model-architecture}):
\begin{itemize}[leftmargin=*, itemsep=0.1em]
  \item \textbf{Baseline}: The model from phase 1 of our procedure, conditioned on home/work coordinates but not demographics.
  \item \textbf{Strong}: The strongly supervised model trained directly on trajectory-demographic pairs (and home/work coordinates).
  \item \textbf{\NAME{} (ours)}: Fine-tuned from the baseline model using only region-level demographic mixtures $p_g$ and aggregate features.
\end{itemize}

\subsubsection{Evaluation metrics}
\label{sec:eval-metrics}
We evaluate trajectory quality using the following standard trajectory statistics \citep{zhu2023difftraj, guo2025leveraging}:
\begin{itemize}[leftmargin=*, itemsep=0.1em]
    \item Spatial: We partition the area into a uniform grid ($40 \times 40$ for Virginia, $100 \times 100$ for California to account for its larger size), then calculate the distribution of trajectory points within each grid cell.
    \item Travel Distance: For each trajectory, we compute the total Haversine distance traveled and bin the distances into a histogram.
    \item Trip: We use the same spatial grid as in Spatial, represent each trajectory as a trip from the origin cell to the destination cell, and compute the distribution over origin-destination cell pairs. 
    \item POI Frequency: We compute the distribution of visits across all POI tokens in the vocabulary.
\end{itemize}
For each statistic, we compute the Jensen--Shannon divergence (JSD) \citep{lin2002divergence} between the feature distribution from real vs. synthetic trajectories.
JSD is symmetric, where lower is better, and equals zero if and only if the distributions are equivalent.
All JSD metrics are computed per demographic group, comparing only to real trajectories from test individuals in that group, and we report JSDs per group and averaged over all groups.

\begin{figure*}[t]
    \centering
    \includegraphics[width=\linewidth]{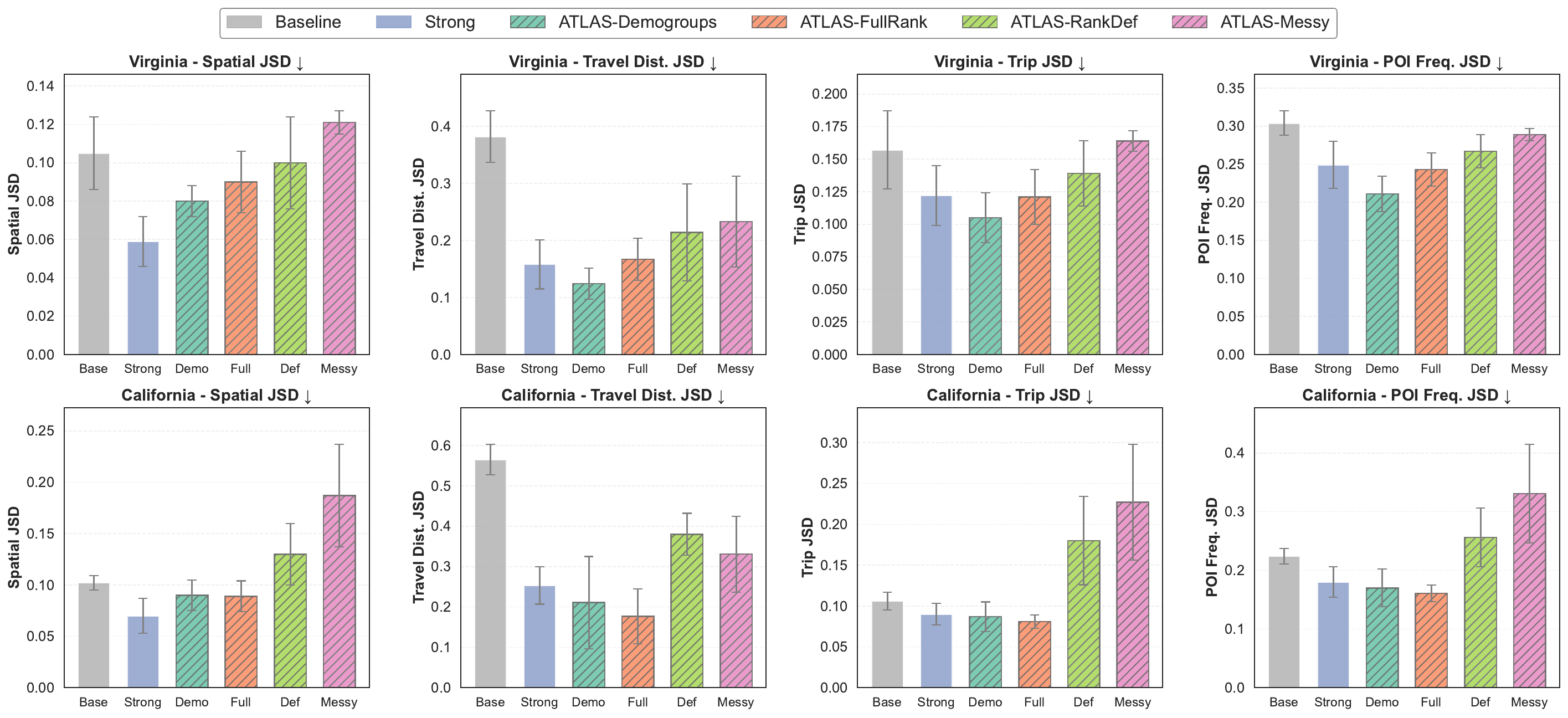}
    \caption{\textbf{Effect of demographic diversity on \NAME{} performance (RQ1).} JSD (lower is better) across four metrics on Virginia (top) and California (bottom). Bars indicate average JSD over 8 demographic groups; errors indicate standard deviation over groups. \NAME{} substantially improves over the baseline under well-conditioned regional partitions, often approaching strongly supervised performance, and degrades gracefully as partitions become ill-conditioned.}
    \label{fig:rq1-results}
\end{figure*}

\subsection{RQ1: Effect of Demographic Diversity}
\label{sec:demo-diversity-results}
First, we evaluate \NAME{} on four different regional partitions that result in varying degrees of diversity in the demographic composition matrix $P$ (following Condition~\ref{ass:full-rank}):
\begin{itemize}[leftmargin=*,
  itemsep=0.1em,
  topsep=0pt,
  parsep=0pt,
  partopsep=0pt]
  \item \textbf{Demogroups}: 8 regions, each with a single demographic group (best case, $P = I_8$, $\sigma_{\min}(P)=1$).
  \item \textbf{Full Rank}: 8 regions with balanced demographic pairs ($\operatorname{rank}(P)=8$, well-conditioned).
  \item \textbf{Rank-Deficient}: 8 regions with mixed demographics ($\operatorname{rank}(P)=7 < 8$, moderately ill-conditioned).
  \item \textbf{Messy}: 8 regions with severe demographic mixing across regions ($\operatorname{rank}(P)=4 \ll 8$, severely ill-conditioned).
\end{itemize}
We include the exact matrices used to instantiate these partitions in Appendix~\ref{sec:app-partitions}.
In these experiments, we use POI counts as the aggregate feature, but test out different features in RQ2 (Section~\ref{sec:phi-results}).


Figure~\ref{fig:rq1-results} summarizes performance across all partition setups for both Virginia and California, revealing how demographic composition affects \NAME{}' ability to recover group-specific trajectory distributions from regional aggregates.
We provide numerical results for each state in the appendix, including average JSDs (Tables~\ref{tab:main_results}, \ref{tab:ca_main_results}) and JSDs per demographic group (Tables~\ref{tab:spatial_jsd_pergroup}-\ref{tab:poi_freq_pergroup}, \ref{tab:ca_spatial_jsd_pergroup}-\ref{tab:ca_poi_freq_pergroup}).

\textbf{Performance on well-conditioned partitions.}
For the Demogroups partition, \NAME{} achieves the strongest overall performance, outperforming the baseline by 12-68\% with an average of 34\% (comparing averages in Tables~\ref{tab:main_results} and \ref{tab:ca_main_results}, range over states and trajectory statistics) and matching strong supervision on several metrics (Figure~\ref{fig:rq1-results}).
On the FullRank partition, \NAME{} maintains strong performance across metrics, outperforming the baseline by 13-69\%, with an average of 31\%.
While performance degrades slightly compared to Demogroups, the model continues to recover meaningful demographic-specific distributions, showing robustness to moderate increases in composition complexity.

\textbf{Progressive degradation with ill-conditioned partitions.}
As predicted by Theorem~\ref{thm:overall-bound}, performance degrades systematically as the demographic composition matrix $P$ becomes less well-conditioned.
For the Rank-Deficient partition, \NAME{} shows modest improvement on spatial metrics but maintains stronger gains on travel distance and POI frequency.
The Messy partition exhibits the most uneven recovery, consistent with severe ill-conditioning preventing reliable disentanglement (Condition~\ref{ass:full-rank}). However, travel distance remains reliably improved relative to the Baseline even in this setting. A plausible explanation is that matching regional aggregates constrains \emph{where} mass is placed in the POI vocabulary: enforcing realistic visit-frequency structure tends to concentrate synthetic trajectories on frequently visited POIs and reduce spurious long-range excursions. In contrast, spatial and trip fidelity degrade under severe ill-conditioning.

\subsection{RQ2: Choice of Feature Map $\phi$}
\label{sec:phi-results}
The choice of aggregate feature $\phi$ determines what aspects of trajectory behavior can be recovered from the feature means, as predicted by our theoretical analysis (Theorem~\ref{thm:recovery-ipm} and Condition~\ref{ass:phi-identifiable}).
Here we consider three types of aggregate features, all normalized histograms over trajectories in each region:
\begin{itemize}[leftmargin=*,
  itemsep=0.1em,
  topsep=0pt,
  parsep=0pt,
  partopsep=0pt]
  \item \textbf{POI-Histogram}: a normalized histogram of POI counts.
  \item \textbf{Cate}: a normalized histogram of POI category counts.
  \item \textbf{Cate-Trans}: a normalized histogram over consecutive POI category bigrams.
\end{itemize}
These features capture essential aspects of demographic-specific mobility behavior: which places are visited, what types of activities are performed, and how activities are sequenced.
To isolate the effect of feature choice, here we fix the regional partition (using Demogroups) and focus on varying the feature.

\begin{table}[h]
\centering
\caption{\textbf{Effect of feature choice on \NAME{} performance (RQ2).} Average JSD$\downarrow$ across 8 demographic groups ($\pm$ std across groups), for Virginia (VA) and California (CA).}
\label{tab:rq2_summary}
\resizebox{\columnwidth}{!}{%
\begin{tabular}{llcccc}
\toprule
State & Feature Map $\phi$ & Spatial & Travel Dist. & Trip & POI Freq. \\
\midrule
\multirow{5}{*}{VA}
& Baseline & $0.105 \pm \scriptstyle 0.019$ & $0.382 \pm \scriptstyle 0.045$ & $0.157 \pm \scriptstyle 0.030$ & $0.304 \pm \scriptstyle 0.016$ \\
& Strong & $0.059 \pm \scriptstyle 0.013$ & $0.158 \pm \scriptstyle 0.043$ & $0.122 \pm \scriptstyle 0.023$ & $0.249 \pm \scriptstyle 0.031$ \\
& POI-Histogram & $0.080 \pm \scriptstyle 0.008$ & $0.124 \pm \scriptstyle 0.027$ & $0.105 \pm \scriptstyle 0.019$ & $0.211 \pm \scriptstyle 0.023$ \\
& Cate-Trans & $0.109 \pm \scriptstyle 0.028$ & $0.227 \pm \scriptstyle 0.039$ & $0.129 \pm \scriptstyle 0.028$ & $0.257 \pm \scriptstyle 0.027$ \\
& Cate & $0.103 \pm \scriptstyle 0.026$ & $0.344 \pm \scriptstyle 0.119$ & $0.149 \pm \scriptstyle 0.036$ & $0.278 \pm \scriptstyle 0.040$ \\
\midrule
\multirow{5}{*}{CA}
& Baseline & $0.102 \pm \scriptstyle 0.007$ & $0.565 \pm \scriptstyle 0.038$ & $0.106 \pm \scriptstyle 0.011$ & $0.224 \pm \scriptstyle 0.013$ \\
& Strong & $0.070 \pm \scriptstyle 0.017$ & $0.253 \pm \scriptstyle 0.046$ & $0.090 \pm \scriptstyle 0.013$ & $0.180 \pm \scriptstyle 0.026$ \\
& POI-Histogram & $0.090 \pm \scriptstyle 0.015$ & $0.211 \pm \scriptstyle 0.114$ & $0.087 \pm \scriptstyle 0.018$ & $0.170 \pm \scriptstyle 0.032$ \\
& Cate-Trans & $0.119 \pm \scriptstyle 0.027$ & $0.326 \pm \scriptstyle 0.063$ & $0.123 \pm \scriptstyle 0.028$ & $0.225 \pm \scriptstyle 0.034$ \\
& Cate & $0.178 \pm \scriptstyle 0.031$ & $0.371 \pm \scriptstyle 0.116$ & $0.308 \pm \scriptstyle 0.036$ & $0.321 \pm \scriptstyle 0.029$ \\
\bottomrule
\end{tabular}}
\end{table}

\textbf{POI features substantially outperform category alternatives.}
Table~\ref{tab:rq2_summary} summarizes performance across three levels of feature granularity, with full per-group breakdowns in Tables~\ref{tab:js_cate_trans_pergroup} and~\ref{tab:ca_cate_ablation_pergroup}.
We find that POI-Histogram achieves the best performance across all evaluation metrics, indicating that fine-grained POI identity provides the most informative aggregate constraints for demographic recovery.
POI-level features encode fine-grained demographic differences in mobility behavior (e.g., which specific stores, gyms, or restaurants each group frequents) that are essential for recovering group-specific trajectory distributions.
Category-level features, by aggregating over all POIs of the same type, discard the spatial and behavioral distinctions that differentiate demographic groups.
Category transitions add some sequential structure, but the loss of POI identity still limits what demographic-conditioned distributions can be recovered from aggregates alone.

\textbf{Imperfect identifiability still helps.}
Importantly, even POI counts do not \emph{fully} satisfy the identifiability requirements of Condition~\ref{ass:phi-identifiable}, yet they still yield improvements over the unconditioned Baseline.
This demonstrates that aggregate supervision provides value even when Condition~\ref{ass:phi-identifiable} is only partially satisfied: richer features provide stronger constraints that push the model closer to demographic-specific patterns, consistent with our theoretical framing of conditions as ideals that guide progressive improvement rather than strict requirements.


\subsection{RQ3: Utility in Downstream Tasks}
Lastly, we evaluate whether improvements in demographic-specific trajectory generation translate into a key downstream task, next POI prediction.
We compare a next-POI predictor trained on synthetic trajectories generated by \NAME{} to one trained on synthetic trajectories from the baseline model and one trained on real trajectories, then evaluate their performance on held-out real trajectories within each demographic group (see Appendix~\ref{sec:app-downstream} for details).

\begin{table}[h]
  \centering
  \caption{\textbf{Next-POI prediction (RQ3).} Accuracy$\uparrow$ and GeoError$\downarrow$ (km) per group for Virginia (VA) and California (CA).}
  \label{tab:rq3_summary}
  \resizebox{\columnwidth}{!}{%
  \begin{tabular}{llccccccccc}
  \toprule
  & Setup & {$<$30, M} & {$<$30, F} & {30--40, M} & {30--40, F} & {40--50, M} & {40--50, F} & {$>$50, M} & {$>$50, F} & Avg \\
  \midrule
  \multicolumn{11}{l}{\textit{Accuracy $\uparrow$}} \\
  \multirow{3}{*}{VA}
  & Real    & 0.587 & 0.638 & 0.549 & 0.568 & 0.570 & 0.523 & 0.565 & 0.521 & 0.565 \\
  & Baseline & 0.480 & 0.473 & 0.520 & 0.555 & 0.474 & 0.440 & 0.417 & 0.437 & 0.475 \\
  & \NAME{} & 0.578 & 0.627 & 0.539 & 0.560 & 0.548 & 0.511 & 0.535 & 0.508 & 0.551 \\
  \cmidrule{2-11}
  \multirow{3}{*}{CA}
  & Real    & 0.610 & 0.657 & 0.604 & 0.618 & 0.620 & 0.554 & 0.596 & 0.581 & 0.605 \\
  & Baseline & 0.599 & 0.625 & 0.596 & 0.562 & 0.587 & 0.539 & 0.530 & 0.509 & 0.581 \\
  & \NAME{} & 0.599 & 0.645 & 0.597 & 0.613 & 0.617 & 0.540 & 0.583 & 0.560 & 0.594 \\
  \midrule
  \multicolumn{11}{l}{\textit{GeoError (km) $\downarrow$}} \\
  \multirow{3}{*}{VA}
  & Real    & 54.5 & 41.8 & 54.5 & 42.9 & 44.3 & 53.6 & 37.9 & 63.9 & 49.2 \\
  & Baseline & 62.4 & 51.2 & 61.5 & 43.8 & 49.2 & 58.1 & 46.3 & 71.0 & 55.5 \\
  & \NAME{} & 56.5 & 41.9 & 56.7 & 42.1 & 49.2 & 53.8 & 41.6 & 66.0 & 51.0 \\
  \cmidrule{2-11}
  \multirow{3}{*}{CA}
  & Real    & 117.3 & 97.7 & 114.1 & 122.0 & 113.6 & 142.7 & 132.3 & 130.2 & 121.2 \\
  & Baseline & 121.4 & 106.0 & 119.8 & 128.8 & 118.5 & 149.7 & 141.8 & 138.0 & 128.0 \\
  & \NAME{} & 119.4 & 101.8 & 115.8 & 123.6 & 114.4 & 147.4 & 136.9 & 134.9 & 124.3 \\
  \bottomrule
  \end{tabular}}
  \end{table}

\textbf{Downstream gains from recovered trajectories.}
Table~\ref{tab:rq3_summary} summarizes results for each state and demographic group, reporting next-POI accuracy and geographic error (in km), with additional metrics (HR@10 and NDCG@10) reported Tables~\ref{tab:next_poi_prediction} and~\ref{tab:next_poi_prediction_ca}.
Training on \NAME{} synthetic data substantially improves next-POI utility compared to the Baseline on both states.
In Virginia, most of the gap in accuracy from Baseline (0.475) to Real (0.565) is closed by \NAME{} (0.551), and GeoError decreases from 55.5 km (Baseline) to 51.0 km (\NAME{}) to 49.2 km (Real).
In California, the same pattern holds: Accuracy improves from 0.581 (Baseline) to 0.594 (\NAME{}) toward 0.605 (Real), and GeoError decreases from 128.0\,km to 124.3\,km toward 121.2\,km.
These consistent improvements demonstrate that aggregate supervision recovers meaningful demographic-specific patterns that transfer to downstream prediction tasks.

\section{Related Work}

\textbf{Mobility trajectory generation.}
Trajectory generation has emerged as a critical tool for privacy-preserving mobility analysis and simulation.
Prior work spans autoregressive models including RNN \citep{kulkarni2017generating}, Transformers \citep{hsu2024trajgpt}, and LLMs \citep{li2024geo,li2024poi,wang2024urban}, as well as non-autoregressive approaches including VAEs \citep{long2023practical,huang2019variational}, GANs \citep{jiang2023continuous}, and diffusion models \citep{zhu2023difftraj,song2024controllable,zhu2024controltraj,wei2024diff,guo2025leveraging} for road- or POI-level trajectory synthesis.
These methods balance validity constraints (e.g., geographic coherence, temporal consistency) with distributional fidelity.

However, most trajectory generation models fail to condition on demographics, despite large demographic heterogeneity in mobility patterns and important consequences of such heterogeneity, such as disparate health or access to resources \citep{chang2021mobility,nilforoshan2023segregation}.
One work that does incorporate demographic conditioning into trajectory generation \citep{song2024controllable} requires strongly supervised learning, with access to individual-level demographic labels.
However, individual-level labels are often unavailable or cannot be released due to privacy concerns.
Differentially private generative models \citep{abadi2016deep,dockhorn2022differentially,igamberdiev-habernal-2023-dp,ponomareva-etal-2022-training} offer formal privacy guarantees but still require access to individual-level demographic data during training.
Our work provides a complementary direction: when individual demographic labels are not available, we provide an alternate path for learning demographic-conditioned trajectory generation from more available, aggregated data. 
To the best of our knowledge, our work is the first to propose learning conditional trajectory generation with aggregate supervision. 

\textbf{Learning from aggregates / label proportions.}
Learning from label proportions (LLP) \citep{quadrianto2009llp} addresses settings where only group-level label statistics are observed rather than individual labels.
A closely related problem is ecological inference \citep{schuessler1999eco,king2004eco,wakefield2008two}, which seeks to infer individual-level behavior from aggregate regional statistics---particularly voting patterns and demographic data in social science applications.
Related work in computer science seeks to infer transition matrices or networks from aggregated information (e.g., node-level marginals, timeseries data) \citep{kumar2015inverting,hallac2017network,maystre2017choicerank,chang2024ipf}.
Both LLP and ecological inference have focused predominantly on discriminative or regression tasks \citep{zhang2020learning,chen2024general,wei2023universal}.
However, the generative modeling counterpart remains unexplored.
Our work bridges this gap by extending aggregate learning principles from discriminative to generative settings.

\section{Discussion}
In this work, we introduced \NAME{}, a novel method for learning demographic-conditioned trajectory generation with aggregate supervision, using region-level demographic compositions and aggregated mobility features.
Our theoretical foundations formalized data conditions that enable \NAME{} to work well, with actionable guidance for practitioners around regional partitions and feature choice. 
Our experiments with real individual trajectories and demographics demonstrated the efficacy of \NAME{}, showing that it outperforms baselines without demographic conditioning and approaches the performance for a strongly supervised model. 

\textbf{Future work.} A strength of \NAME{} is that it is model-agnostic.
While we instantiated \NAME{} with a diffusion model, future work could apply \NAME{} to other model architectures, such as LLMs \citep{li2024geo,wang2024urban} or VAEs \citep{long2023practical}.
Furthermore, our aggregate supervision framework could be extended to generative models for other data types, beyond mobility trajectories, and/or to other types of conditioning, beyond demographics.
Finally, future work could explore learning conditioned distributions that share weights (e.g., between ``< 30, M'' and ``< 30, F'') or scaling \NAME{} to more trajectory data and to more regions, including those outside of the U.S.

\section{Acknowledgements}
We thank Carlos Guirado and Prof.~Joan Walker for providing access to the Embee dataset \cite{bouzaghrane2025tracking} and for helpful discussions about the data and its collection. This work was supported in part by the Google Research Scholar Program and by compute resources at UC Berkeley's Center for Human-Compatible AI (CHAI). 

\bibliography{references}
\bibliographystyle{unsrt}

\appendix

\counterwithin{figure}{section}
\counterwithin{table}{section}
\renewcommand\thefigure{\thesection\arabic{figure}}
\renewcommand\thetable{\thesection\arabic{table}}

\section{Extended Theory}
\label{sec:app-extended-theory}

This appendix provides additional theoretical details complementary to Section~\ref{sec:theory}. In particular, we (i)~provide the core notation, and (ii)~provide full proofs for results stated as proof sketches in the main text.

\subsection{Notation}
\label{sec:app-theory-notation}
We use the same notation as Section~\ref{sec:theory}. Let $K$ denote the number of demographic groups and $G$ the number of regions. Let $\phi:\mathcal{X}\to\mathbb{R}^m$ be the aggregate feature map. By construction, $V_\star = P M_\star^\top$ and $V_\theta = P M_\theta^\top$.

\begin{table}[ht]
\centering
\caption{Key Notation.}
\label{tab:notation}
\small
\begin{tabular}{ll}
\toprule
\textbf{Symbol} & \textbf{Definition} \\
\midrule
$p_g\in\Delta^{K-1}$ & region-$g$ demographic composition (row of $P$) \\
$P\in\mathbb{R}^{G\times K}$ & demographic composition matrix with rows $p_g^\top$ \\
$\mu_\star(d)\in\mathbb{R}^m$ & group-$d$ feature mean $\mathbb{E}_{P_\star(\cdot\mid d)}[\phi(X)]$ \\
$\mu_\theta(d)\in\mathbb{R}^m$ & model group-$d$ feature mean $\mathbb{E}_{P_\theta(\cdot\mid d)}[\phi(X)]$ \\
$M_\star\in\mathbb{R}^{m\times K}$ & columns are $\mu_\star(d)$ \\
$M_\theta\in\mathbb{R}^{m\times K}$ & columns are $\mu_\theta(d)$ \\
$\nu_\star(g)\in\mathbb{R}^m$ & region-$g$ aggregate $\sum_d p_g(d)\mu_\star(d)$ \\
$\nu_\theta(g)\in\mathbb{R}^m$ & model region-$g$ aggregate $\sum_d p_g(d)\mu_\theta(d)$ \\
$V_\star\in\mathbb{R}^{G\times m}$ & rows are $\nu_\star(g)^\top$ \\
$V_\theta\in\mathbb{R}^{G\times m}$ & rows are $\nu_\theta(g)^\top$ \\
\bottomrule
\end{tabular}
\normalsize
\end{table}

\subsection{Proofs omitted from the main text}
\label{sec:app-theory-proofs}

\subsubsection{Proof of Lemma~\ref{lem:moment-identifiability}}
\begin{lemma*}[Uniqueness of group-level feature means]
Suppose Condition~\ref{ass:full-rank} holds. If the model-implied and true regional aggregates match, $\nu_\theta(g) = \nu_\star(g)$ for all regions $g$, then the demographic group-level feature means coincide:
\[
  \mu_\theta(d) = \mu_\star(d) \quad \text{for all } d.
\]
\end{lemma*}

\begin{proof}
The assumption $\nu_\theta(g)=\nu_\star(g)$ for all $g$ is equivalent to $V_\theta=V_\star$.
Using $V_\theta = P M_\theta^\top$ and $V_\star = P M_\star^\top$, we have
\[
P(M_\theta - M_\star)^\top = 0.
\]
Since $P$ has full column rank (Condition~\ref{ass:full-rank}), the only vector $u\in\mathbb{R}^K$ satisfying $Pu=0$ is $u=0$.
Applying this argument column-wise to $(M_\theta-M_\star)^\top$ yields $M_\theta=M_\star$, i.e., $\mu_\theta(d)=\mu_\star(d)$ for all $d$.
\end{proof}

\subsubsection{Proof of Lemma~\ref{lem:stability}}
\begin{lemma*}[Stability under aggregate perturbations]
Suppose Condition~\ref{ass:full-rank} holds. Then for any $V_1,V_2\in\mathbb R^{G\times m}$,
if $M_i^\top := P^\dagger V_i$, we have
\[
\|M_1 - M_2\|_F \le \frac{1}{\sigma_{\min}(P)}\|V_1 - V_2\|_F.
\]
\end{lemma*}
\begin{proof}
Let $V_1,V_2\in\mathbb{R}^{G\times m}$ and define $M_i^\top := P^\dagger V_i$.
Then
\[
M_1^\top - M_2^\top = P^\dagger (V_1 - V_2).
\]
Taking Frobenius norms and using $\|AX\|_F \le \|A\|_2\|X\|_F$ gives
\[
\|M_1 - M_2\|_F = \|M_1^\top - M_2^\top\|_F \le \|P^\dagger\|_2\,\|V_1-V_2\|_F.
\]
For a full-column-rank matrix $P$, $\|P^\dagger\|_2 = 1/\sigma_{\min}(P)$, yielding the claim.
\end{proof}

\subsubsection{Proof of Lemma~\ref{lem:finite-sample-bound} and Theorem~\ref{thm:overall-bound}}
\begin{lemma*}[Finite sample error bound]
    Suppose Condition~\ref{ass:full-rank} holds. Let $n_{\min} = \min_{g} n_g$ be the minimum sample size across regions. For any $\delta \in (0, 1)$, with probability at least $1 - \delta$, the error between the recovered demographic group-level feature means $\hat{M}$ (derived from $\hat{V}$) and the true means $M_*$ satisfies:
    \[
        \|\hat{M} - M_*\|_F \le \frac{1}{\sigma_{\min}(P)} \left( \frac{B (\sqrt{m} + \sqrt{2\log(G/\delta)})}{\sqrt{n_{\min}}} \right) \cdot \sqrt{G},
    \]
    where $\sigma_{\min}(P)$ is the smallest singular value of the demographic composition matrix $P$.
\end{lemma*}

\begin{theorem*}[Overall bound (optimization + sampling)]
Suppose Condition~\ref{ass:full-rank} holds.
Let $\hat V$ denote the observed (empirical) region-level aggregates and define
\[
\varepsilon_{\mathrm{opt}} := \|V_\theta - \hat V\|_F,
\qquad
\varepsilon_{\mathrm{samp}} :=
\left(
\frac{B(\sqrt{m}+\sqrt{2\log(G/\delta)})}{\sqrt{n_{\min}}}
\right)\sqrt{G}.
\]
Then for any $\delta\in(0,1)$, with probability at least $1-\delta$,
\[
\|M_\theta - M_\star\|_F
\;\le\;
\frac{\varepsilon_{\mathrm{opt}} + \varepsilon_{\mathrm{samp}}}{\sigma_{\min}(P)}.
\]
That is, the total error in recovering demographic group-level feature means is bounded by the sum of optimization error and sampling error in the regional aggregates, amplified by $1/\sigma_{\min}(P)$.
\end{theorem*}
\begin{proof}
We follow the notation in Section~\ref{sec:theory}. Let $\hat V$ be the empirical region-level aggregates formed by averaging $\phi(X)$ over $n_g$ samples in each region $g$.
Let $V_\star$ denote the population aggregates.
By Lemma~\ref{lem:stability},
\[
\|\hat M - M_\star\|_F \le \frac{1}{\sigma_{\min}(P)}\|\hat V - V_\star\|_F.
\]
It remains to upper-bound $\|\hat V - V_\star\|_F$.
Since $\|\phi(X)\|_2\le B$ almost surely, each coordinate of $\hat v(g)-v_\star(g)$ is an average of bounded random variables, and a standard vector concentration bound (e.g., coordinate-wise Hoeffding plus a union bound over $G$ regions and $m$ coordinates) yields that with probability at least $1-\delta$,
\[
\|\hat V - V_\star\|_F
\;\le\;
\left(\frac{B(\sqrt{m}+\sqrt{2\log(G/\delta)})}{\sqrt{n_{\min}}}\right)\sqrt{G}.
\]
Combining the two displays gives Lemma~\ref{lem:finite-sample-bound}.

For Theorem~\ref{thm:overall-bound}, we decompose
\begin{align*}
\|M_\theta - M_\star\|_F
&\;\le\;
\|M_\theta-\hat M\|_F + \|\hat M - M_\star\|_F \\
&\;\le\;
\frac{1}{\sigma_{\min}(P)}\|V_\theta-\hat V\|_F + \frac{1}{\sigma_{\min}(P)}\|\hat V-V_\star\|_F,
\end{align*}
where the second inequality again uses Lemma~\ref{lem:stability} (applied to $(V_\theta,\hat V)$ and $(\hat V,V_\star)$).
Identifying $\varepsilon_{\mathrm{opt}}:=\|V_\theta-\hat V\|_F$ and the high-probability bound on $\|\hat V-V_\star\|_F$ as $\varepsilon_{\mathrm{samp}}$ yields \eqref{eq:overall-bound}.
\end{proof}

\subsubsection{Proof of Theorem~\ref{thm:recovery-ipm}}
\begin{theorem*}[Recovery up to $\phi$-IPM]
Fix a feature map $\phi$.
If the model matches population-level regional aggregates,
$V_\theta^\phi = V_\star^\phi$,
then for every demographic group $d$,
\[
d_\phi\!\left(P_\theta(\cdot\mid d), P_\star(\cdot\mid d)\right)=0.
\]
\end{theorem*}
\begin{proof}
Aggregate matching $V_\theta=V_\star$ implies $\mu_\theta(d)=\mu_\star(d)$ for all $d$ by Lemma~\ref{lem:moment-identifiability}.
For the feature-induced IPM $d_\phi$, the closed form is
\begin{align*}
d_\phi\!\left(P_\theta(\cdot\mid d),P_\star(\cdot\mid d)\right)
&=
\left\|\mathbb{E}_{P_\theta(\cdot\mid d)}[\phi(X)]-\mathbb{E}_{P_\star(\cdot\mid d)}[\phi(X)]\right\|_2 \\
&=
\|\mu_\theta(d)-\mu_\star(d)\|_2
=0,
\end{align*}
which holds for each $d$.
\end{proof}

\subsubsection{Proof of Theorem~\ref{thm:consistency-aggregates}}
\begin{theorem*}[Equivalence of aggregate matching and demographic recovery]
Fix a feature map $\phi$ and let $V_\theta, V_\star \in \mathbb{R}^{G\times m}$ denote the
population region-level aggregates with
\[
v_\theta(g)=\sum_{d=1}^K p_{g,d}\,\mu_\theta(d),
\qquad
v_\star(g)=\sum_{d=1}^K p_{g,d}\,\mu_\star(d),
\]
where $\mu_\theta(d)=\mathbb{E}_{X\sim P_\theta(\cdot\mid d)}[\phi(X)]$ and similarly for $\mu_\star(d)$.
Assume:
(i) $P$ has full column rank (Condition~\ref{ass:full-rank}), and
(ii) $\phi$ is identifying on the family of conditional distributions (Condition~\ref{ass:phi-identifiable}).
Then the following are equivalent:
\begin{enumerate}[label=(\roman*), leftmargin=*]
\item $V_\theta = V_\star$ (equivalently, $\nu_\theta(g)=\nu_\star(g)$ for all $g$);
\item $\mu_\theta(d)=\mu_\star(d)$ for all $d\in\{1,\dots,K\}$;
\item $P_\theta(\cdot\mid d)=P_\star(\cdot\mid d)$ for all $d\in\{1,\dots,K\}$.
\end{enumerate}
\end{theorem*}
\begin{proof}
(i)$\Rightarrow$(ii) follows from Lemma~\ref{lem:moment-identifiability}.
(ii)$\Rightarrow$(iii) follows from Condition~\ref{ass:phi-identifiable}, which states that equality of feature expectations for all groups implies equality of the conditional distributions.
(iii)$\Rightarrow$(ii) is immediate by taking expectations of $\phi(X)$ under both sides.
(ii)$\Rightarrow$(i) follows by plugging $\mu_\theta=\mu_\star$ into $V=PM^\top$.
\end{proof}

\subsection{Identifiability examples}
\label{sec:app-identifiability-examples}

\subsubsection{Example 1: Minimal exponential families.}
A canonical example of a low-order structured family satisfying Condition~\ref{ass:phi-identifiable} is the minimal exponential family.
Let
\[
\mathcal{P}_j
=
\left\{
P_\eta(x)
=
\exp\!\left(\langle\eta,\phi_j(x)\rangle - A(\eta)\right)\rho(x)
: \eta \in \Omega
\right\},
\]
where $\rho(x)$ is a base measure, $\Omega \subset \mathbb{R}^m$ is the natural parameter space,
and the family is minimal (i.e., no nontrivial affine dependence among the components of $\phi_j$).

In a minimal exponential family,
the log-partition function $A(\eta)$ is strictly convex.
The moment map satisfies
\[
\nabla_\eta A(\eta)=\mathbb{E}_{P_\eta}[\phi_j(X)],
\]
which is therefore injective by the strict convexity of $A$.
Equality of expectations implies equality of parameters,
hence equality of distributions.

\subsubsection{Example 2: Scaling the unconditioned distribution.}
In the \NAME{} setting, we have access to individual trajectories without demographic information, allowing us to learn the unconditional trajectory distribution $P_\star(X)$.
Even if $P_\star(X)$ is arbitrarily complex, we can recover the demographic-conditioned distribution from regional aggregates, provided that demographic effects enter through a low-order perturbation.
Let $x = x^{(1)}, x^{(2)}, \cdots, x^{(T)}$ represent a trajectory of length $T$.
If the group-level distribution takes the form
\begin{align}
    P_\star(x | d) \propto P_\star(x) \cdot \Pi_{t=1}^T \delta_{x_t,d},
\end{align}
where $\delta_{x_t,d}$ represents group-specific POI scaling factors, then POI visit counts form sufficient statistics and choosing $\phi(x)$ to be the POI visit count vector uniquely identifies the demographic-conditioned distribution.

\subsubsection{Connection between examples.}
Examples 1 and 2 are unified by the I-projection framework (Section~\ref{sec:i-projection} below).
The I-projection has the form $Q_d^\dagger(x) \propto \exp(\langle\lambda_d,\phi(x)\rangle) P_\star(x)$, which is an exponential family (Example 1) with base measure $\rho = P_\star$.
In Example 2, the scaling factors $\Pi_{t=1}^T \delta_{x_t,d}$ can be written as $\exp(\langle\lambda_d, \phi(x)\rangle)$ where $\phi(x)$ is the POI visit count vector and $\lambda_d$ encodes the log-scaling factors.
This reveals that both examples instantiate the same mathematical structure: demographic conditioning introduces an exponential tilt of the unconditional baseline $P_\star$, parameterized by feature expectations.

\subsection{I-projection view of learning from aggregates}
\label{sec:i-projection}
This section formalizes a population-level selection principle that complements the identifiability results: when the aggregate constraints do not uniquely specify $P_\star(\cdot\mid d)$, a natural target is the minimum-change distribution relative to a pretrained base.

\subsubsection{Constraint set and information projection}
Assume phase~1 training yields an unconditional (or non-demographic) base distribution $P_\star$ over trajectories.
For each demographic group $d$, define the constraint set
\[
\mathcal{C}_d := \left\{ Q \in \mathcal{P}(\mathcal{X}) : \mathbb{E}_{Q}[\phi(X)] = \mu_\star(d) \right\}.
\]
The information projection (I-projection) of $P_\star$ onto $\mathcal{C}_d$ is
\[
Q_d^\dagger := \arg\min_{Q\in\mathcal{C}_d} \mathrm{KL}(Q\|P_\star).
\]

\subsubsection{Exponential-tilt form}
\begin{theorem}[Exponential-tilt characterization of the I-projection]
\label{thm:iproj-tilt}
Assume $\mathcal{C}_d$ is nonempty and $\phi$ is finite-dimensional.
Then $Q_d^\dagger$ exists and has the form
\[
Q_d^\dagger(x) = \frac{\exp(\langle \lambda_d,\phi(x)\rangle)}{Z(\lambda_d)}\,P_\star(x),
\]
for some $\lambda_d\in\mathbb{R}^m$, where $Z(\lambda_d)=\mathbb{E}_{P_\star}[\exp(\langle \lambda_d,\phi(X)\rangle)]$.
\end{theorem}
\begin{proof}
This is a standard convex duality result for moment-constrained KL minimization. Introducing Lagrange multipliers for the moment constraints and normalization and taking the functional derivative with respect to $Q$ yields an optimizer proportional to $P_\star(x)\exp(\langle \lambda_d,\phi(x)\rangle)$, with $Z(\lambda_d)$ enforcing normalization.
\end{proof}

\subsubsection{Interpretation}
Theorem~\ref{thm:iproj-tilt} provides an information-theoretic interpretation of aggregate supervision: among all conditional distributions that match the demographic-level feature constraints, the I-projection selects the one that deviates minimally from the base mobility behavior encoded by $P_\star$.
This lens motivates pretraining and clarifies why richer feature maps $\phi$ lead to stronger, more specific conditional adaptations.

\subsubsection{Practical implication: a strong baseline model is useful}
The I-projection implies that identifiability can be achieved when demographic conditioning acts as a perturbation of a known baseline distribution.
In practice, this means that learning a strong unconditional model on available trajectory data---even without demographic labels---provides a valuable backbone for recovering group-specific distributions.
The fine-tuning phase can be interpreted as finding the group-level distributions that satisfy the regional aggregate constraints while producing minimal-change adaptations of the unconditional distribution.
This interpretation justifies our two-phase training approach in \NAME{} (Section~\ref{sec:model-training}).


\section{Empirical Details}
\label{sec:app-empirical}

This section provides additional details about our empirical set-up.
We describe the mobility dataset used throughout our experiments, the preprocessing pipeline that transforms raw visit logs into trajectory segments, and the model architecture and training protocol.
All experiments share the same data splits, model architecture, and evaluation metrics; the primary experimental variations are the demographic compositions resulting from the choice of regional partition, the divergence function used in the aggregate loss, and the choice of aggregate feature. See the following Section~\ref{sec:app-experiments} for experimental details.

\subsection{Mobility Datasets}
\label{sec:app-datasets}
\subsubsection{Data source.}
We use the proprietary Embee mobility dataset from a longitudinal COVID-19 panel study~\citep{bouzaghrane2025tracking}, which combines passively collected point-of-interest (POI) visit logs with self-reported demographics from repeated surveys administered between August 2020 and September 2022.
The collection and analysis of the Embee dataset were approved by the UC Berkeley Committee for Protection of Human Subjects (CPHS). 
The passive mobility data is collected via a mobile analytics platform that recruits a panel of US smartphone users; importantly, the POI records represent \emph{inferred check-ins} at discrete locations rather than continuous GPS traces.
Each record corresponds to a detected visit event and includes geographic coordinates (latitude and longitude), arrival and departure timestamps in local time, a POI name with brand and category metadata, the distance and travel time to reach the location, and flags indicating whether the location corresponds to a user's inferred home or workplace.
For privacy, home and work coordinates are obfuscated within a fixed radius by the data provider.

We evaluate \NAME{} on two geographic U.S. states covered in this dataset: Virginia and California.
Evaluating on both states enables us to test \NAME{} across different urban environments, population densities, and mobility patterns.
Virginia provides a mix of urban, suburban, and rural areas concentrated in the Mid-Atlantic region, while California offers a larger, more diverse population with distinct metropolitan areas (San Francisco Bay Area, Los Angeles Basin).
Analyses in both states are restricted to panelists with complete demographic labels for the attributes used in this study.

\subsubsection{Demographic attributes.}
The longitudinal survey captures a rich set of sociodemographic variables, including age, gender, household income, race/ethnicity, education level, and household size.
For simplicity and to ensure sufficient sample sizes per group, we use only age and gender in this work, discretizing demographics into $K=8$ groups defined by four age bins crossed with two genders: $\{\text{<30},\,\text{30--40},\,\text{40--50},\,\text{>50}\}\times\{\text{M},\text{F}\}$.
These labels are used exclusively for evaluation and for constructing controlled region partitions; our aggregate-supervision setting assumes that individual demographic labels are unobserved during training.

\begin{figure}[!ht]
\centering
\includegraphics[width=\columnwidth]{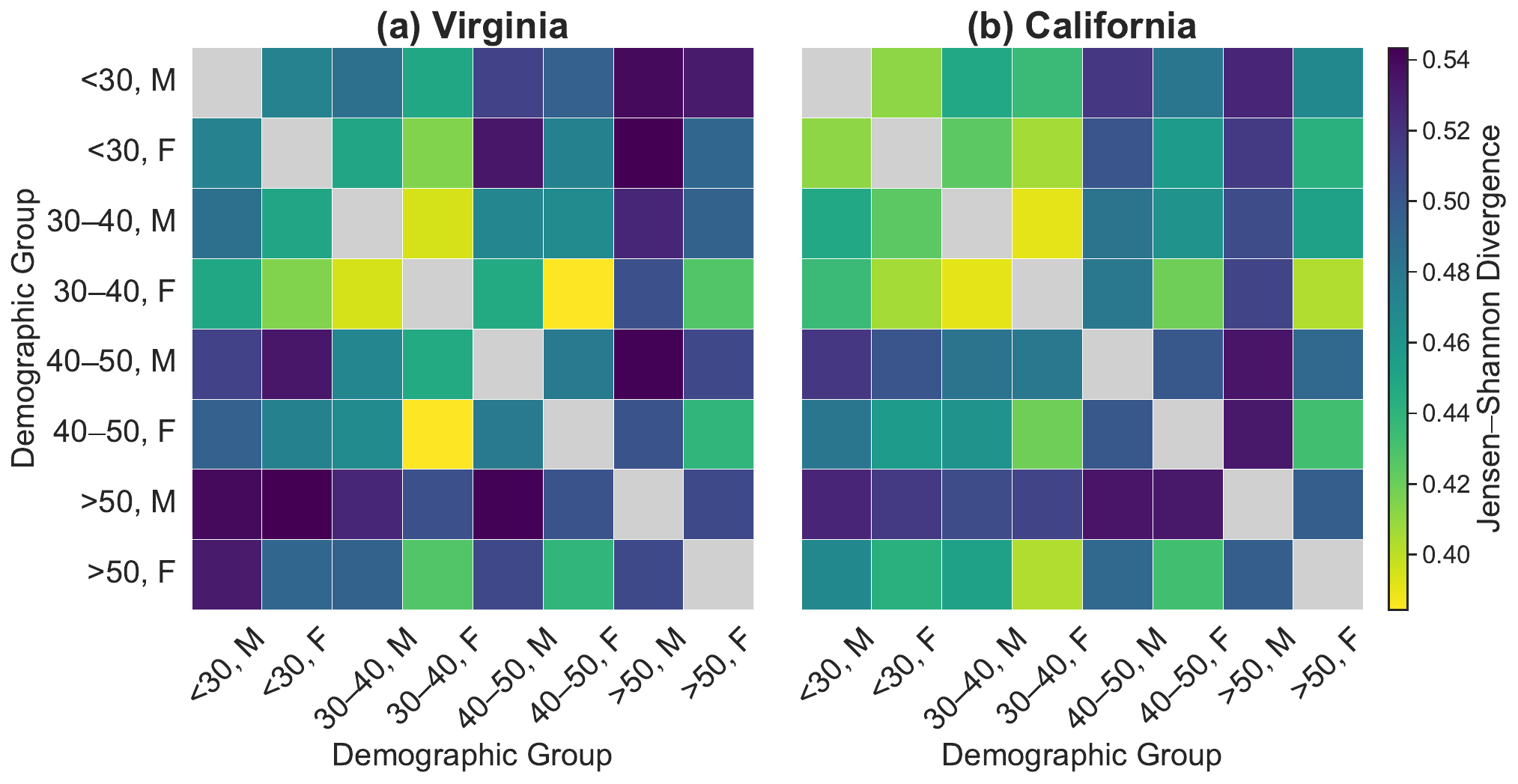}
\caption{\textbf{Pairwise JSD between demographic groups.}
Heatmaps showing pairwise Jensen--Shannon divergence (JSD) of POI visit distributions between $K{=}8$ age$\times$gender groups for Virginia (left) and California (right). JSD values range from approximately 0.40 to 0.54, with warmer colors indicating larger divergence.}
\label{fig:appendix-heat-jsd}
\end{figure}

Figure~\ref{fig:appendix-heat-jsd} quantifies the \textbf{distinguishability of demographic groups based on their POI visit patterns}.
JSD values across all group pairs range from approximately 0.40 to 0.54, indicating substantial but not extreme separation---groups are distinguishable yet share common mobility structures.

\textbf{High-similarity pairs.}
The smallest JSDs tend to occur between nearby age brackets, especially within the 30--40 vs 40--50 range (e.g., 30--40 Male vs 30--40 Female; 30--40 Female vs 40--50 Female). This indicates that mid-life groups have relatively similar POI visit distributions, with gradual changes across adjacent life stages rather than abrupt shifts.

\textbf{High-divergence pairs.}
The largest JSDs are concentrated in comparisons involving the oldest group ($>$50, especially $>$50 Male) against younger and mid-age groups. More broadly, cross-age contrasts (youngest vs oldest) are consistently among the most divergent, reflecting distinct POI preference profiles at opposite life stages.

\textbf{Cross-region consistency.}
Virginia and California show highly similar qualitative structure: (i) mid-age adjacency is relatively similar, (ii) youngest--oldest contrasts are most different, and (iii) the $>$50 Male row/column stands out as especially divergent in both panels. This consistency supports the interpretation that demographic groups exhibit systematic, measurable differences in mobility behavior that generalize across states—motivating demographic-conditioned trajectory generation.

\subsubsection{Preprocessing pipeline.}
We apply a preprocessing pipeline to transform raw visit logs into fixed-length trajectory segments suitable for generative modeling.
The core steps are shared across both states: we remove entries with missing coordinates or flagged as non-POI visits, and split overnight visits at midnight to enable day-level quality assessment.
Each POI is assigned a unique identifier, which we then use to construct the model vocabulary: home and work locations are mapped to dedicated special tokens, and all remaining POIs outside the top-$K$ most frequent locations are collapsed into a single out-of-vocabulary token to mitigate extreme sparsity.

For each user-day, we compute a daily travel distance (approximated by the bounding-box diagonal of visited locations, converted to kilometers) and a temporal coverage ratio (total observed minutes divided by 1440, capped at 1).
Days with daily distance exceeding 800\,km are flagged as spatially unreliable, while days with coverage ratio at least 0.2 are flagged as high-quality.
We extract 14-day trajectory windows; a window is retained only if at most 10\% of its days are spatially unreliable and at least 80\% are high-quality.
For Virginia, we use overlapping windows with a stride of 7 days; for California, we use non-overlapping windows (stride of 14 days).

To reduce sequence length and emphasize behavioral transitions over dwell-time repetitions, we apply run-length encoding: consecutive identical POI tokens are collapsed into a single token.
Finally, we restrict to users with both home and work coordinates available, ensuring consistent spatial conditioning across all models.

\subsubsection{Data splits and statistics.}
To prevent information leakage across segments from the same user, we partition users (not segments) into training, validation, and test sets using an 80/10/10 split based on a deterministic hash of the user identifier.

\begin{itemize}[leftmargin=*, itemsep=0.1em]
    \item \textbf{Virginia.}
    The Virginia dataset contains 4,576 users with complete home/work locations, split into 3,621/484/471 (train/val/test) users, yielding 69,120/9,366/8,995 trajectory segments.
    Of these, 3,745 users have demographic labels, producing 60,978/8,333/7,688 segments with demographic information.
    The median sequence length after run-length encoding is 32 tokens (p95: 68).
    The vocabulary contains 6,981 POI tokens.
    \item \textbf{California.}
    The California dataset contains 11,143 users, split into 8,931/1,097/1,115 (train/val/test) users, yielding 91,037 /11,326 /10,937 trajectory segments.
    Of these, 9,114 users have demographic labels, producing 79,884/10,032/9,275 (train/val/test) segments with demographic information.
    The median number of visits per segment is 43 (p95: 74).
    The vocabulary contains 9,506 POI tokens.
\end{itemize}

\subsubsection{Overview of public mobility datasets.}
We briefly survey three widely used mobility datasets to contextualize our data and highlight the general absence of ground-truth demographic labels in publicly available trajectory datasets.

\begin{itemize}[leftmargin=*, itemsep=0.1em]
    \item \textbf{GeoLife}~\citep{zheng1geolife} is a GPS trajectory dataset collected by Microsoft Research Asia from 178 users over more than four years (April 2007 to October 2011).
    It contains 17,621 trajectories totaling approximately 1.25 million kilometers and 48,000 hours of travel, primarily in Beijing, China.
    Each record is a raw GPS point with latitude, longitude, altitude, and timestamp; 91\% of the data is sampled at high frequency (every 1--5 seconds or 5--10 meters).
    Transportation mode labels (walk, bike, bus, car, etc.) are available for a subset of 69 users, but no demographic attributes (age, gender, income) are provided.
    \item \textbf{YJMob100K}~\citep{yabe2024yjmob100k} is an open-source, anonymized dataset of 100,000 individuals' mobility trajectories collected from mobile phone GPS data in an undisclosed Japanese metropolitan area over 75 days.
    To preserve privacy, locations are discretized into 500\,m $\times$ 500\,m grid cells, timestamps are binned into 30-minute intervals, and the city name and exact coordinates are withheld.
    The dataset includes auxiliary POI category counts per grid cell (85 categories) but explicitly excludes individual-level attributes: gender, age, and occupation are unknown by design.
    \item \textbf{Veraset}~\citep{veraset2022visits} is a commercial mobility data product that provides POI visit logs inferred from smartphone location signals across the United States.
    Each record includes a device identifier, POI name, category (NAICS code), brand, dwell time, and census block group.
    While Veraset data is widely used in urban analytics and epidemiology, demographic attributes of individual devices are not released; any demographic analysis must rely on census-level proxies derived from home-location inference.
\end{itemize}
The absence of ground-truth demographics in these datasets reflects realistic deployment constraints: privacy regulations and data-provider policies typically prohibit linking trajectories to individual-level sociodemographic labels.
Our aggregate-supervision setting is designed precisely for this scenario, where individual demographics are unavailable but region-level demographic compositions and aggregate mobility statistics remain accessible.

\subsection{Model Architecture}
\label{sec:app-architecture}

Our trajectory generation framework builds on the segment-level latent diffusion architecture introduced by Cardiff~\citep{guo2025leveraging}, which decomposes trajectory synthesis into a discrete-to-continuous autoencoder followed by a diffusion transformer operating in latent space.
We adopt the first phase of this cascaded pipeline (segment-level latent diffusion) without the second-phase GPS-level refinement, as our POI-based trajectories do not require continuous coordinate generation.

\subsubsection{Autoencoder.}
We use a BART-style~\citep{lewis2020bart} autoencoder to compress discrete POI token sequences into continuous latent representations.
The encoder is a bidirectional transformer that maps a variable-length POI sequence into a fixed-length latent sequence; the decoder is an autoregressive transformer that reconstructs the original token sequence from the latent.

\subsubsection{Diffusion Transformer (DiT).}
The generative backbone is a Diffusion Transformer~\citep{peebles2023scalable} that denoises latent trajectories using a cosine noise schedule.
The DiT applies a stack of transformer blocks with adaptive layer normalization (adaLN) to condition on the diffusion timestep and auxiliary features.

\subsubsection{Conditioning.}
The DiT accepts several conditioning signals:
(i)~\emph{home/work coordinates}: latitude and longitude of the user's home and workplace, embedded via a small MLP and added to the timestep embedding;
(ii)~\emph{demographic attributes}: age bin and gender, each embedded via a learnable lookup table and concatenated to the condition vector (only used during aggregate-supervised finetuning, not during unconditional pretraining).
During aggregate-supervised training, demographic conditioning is sampled from the region's demographic composition $\pi_g$ to enable learning from aggregates.

\subsection{Model Training}
\label{sec:app-training}
Training proceeds in two phases: (i)~unconditional pretraining of the autoencoder and DiT, followed by (ii)~aggregate-supervised finetuning described in the main text.

\subsubsection{Phase 1: Pretraining.}
The BART autoencoder is pretrained on the full training corpus with masked language modeling using the HuggingFace Trainer, which employs AdamW with weight decay 0.01, a linear warmup over 1000 steps, and BF16 mixed precision.
The DiT is then pretrained to denoise latent representations extracted from the frozen autoencoder, using the standard diffusion MSE objective (predicting $x_0$).
DiT pretraining uses the Adam optimizer with a linear warmup over 1000 steps and gradient clipping at norm 1.0.

\subsubsection{Phase 2: Aggregate-supervised finetuning.}
Starting from the pretrained DiT checkpoint, we finetune with a combination of the diffusion MSE loss and the aggregate divergence loss (Section~\ref{sec:method}).
The aggregate loss weight $\lambda_{\text{agg}}$ is set to 1.0, and we experiment with divergence functions including Jensen--Shannon (JS) and Total Variation (TV).
Finetuning uses the Adam optimizer with a lower learning rate ($10^{-5}$--$10^{-6}$), batch size 256--1024, and gradient clipping at norm 1.0.
Larger batch sizes are preferred during finetuning because each batch samples trajectories from a single region; a larger batch provides a more representative view of the region's population distribution, yielding more stable gradient estimates for the aggregate loss.
We use DDIM sampling~\citep{song2020denoising} with 50 steps for efficient generation during the aggregate loss computation.
To prevent the aggregate loss from over-fitting to high-frequency tokens (home, work, and the catch-all ``other'' token), we optionally mask these tokens from the POI histogram before computing the divergence.

\subsubsection{Implementation.}
All experiments are conducted on a single NVIDIA RTX A6000 GPU using Python 3.9.5, PyTorch 2.6.0, and CUDA 12.4. Table~\ref{tab:hyperparams} summarizes the key hyperparameters used in our experiments.

\begin{table}[t]
\centering
\caption{Hyperparameters of our trajectory generation model.}
\label{tab:hyperparams}
\small
\begin{tabular}{ll|ll}
\toprule
\multicolumn{4}{c}{\textbf{Shared Diffusion Settings}} \\
\midrule
Diffusion steps & 1000 & Noise schedule & Cosine \\
Optimizer & Adam & Prediction type & $x_0$ \\
\midrule
\multicolumn{2}{c|}{\textbf{BART}} & \multicolumn{2}{c}{\textbf{DiT}} \\
\midrule
Encoder layers & 4 & DiT depth & 8--12 \\
Decoder layers & 2 & Hidden size & 128--512 \\
Model dimension & 256 & Attention heads & 4 \\
FFN dimension & 1024 & Input length & 64 \\
Attention heads & 4 & Input dimension & 256 \\
Max sequence length & 64 & Demo embedding dim & 256 \\
\midrule
\multicolumn{2}{c|}{\textbf{Pretraining}} & \multicolumn{2}{c}{\textbf{Finetuning}} \\
\midrule
Learning rate & $10^{-4}$ & Learning rate & $10^{-5}$--$10^{-6}$ \\
Batch size & 512 & Batch size & 256--1024 \\
\bottomrule
\end{tabular}
\end{table}

\subsubsection{Evaluation protocol.}
We report results across multiple random seeds and compute per-demographic-group JSD metrics on the held-out test set.
For each seed, we generate synthetic trajectories by sampling from the finetuned DiT with DDIM (50 steps, $\eta=0$) and decode through the frozen autoencoder.
Metrics are computed by comparing the empirical distributions of synthetic and real trajectories within each demographic group.

\subsubsection{Training dynamics.}
Figures~\ref{fig:va-training-avg}--\ref{fig:ca-training-group} show the evolution of the validation loss and test metrics throughout finetuning for Virginia and California, respectively.
While we did not use the test metrics during model training, comparing how minimizing the training loss (the aggregate loss per region based on train individuals) relates to the validation loss (the aggregate loss per region based on validation individuals) and the test metrics (the JSD per \textit{demographic group} and trajectory statistic based on the test individuals) allows us to investigate how well the training signal, i.e., what the model observes, corresponds with the test metrics, i.e., what the model seeks to optimize. 
This correspondence from train loss to test metric is not guaranteed given the shift from region-level to group-level and the shift in metric (from aggregate feature loss to trajectory statistic), but luckily we see decent correspondence between the two, as described below, which explains the success of \NAME{}.

For Virginia (Figures~\ref{fig:va-training-avg}--\ref{fig:va-training-group}), finetuning runs for 30k steps. Across metrics, the average JSD decreases substantially over the first $\sim$10k steps, with additional gains around 15--20k steps; afterward, the curves largely stabilize with mild upward drift in spatial, trip, and POI frequency JSD. The per-group breakdown shows heterogeneous difficulty across demographic groups, with some groups (e.g., 30--40 F) tending to exhibit higher JSD and some groups displaying non-monotonic trajectories during early finetuning.

For California (Figures~\ref{fig:ca-training-avg}--\ref{fig:ca-training-group}), finetuning runs for 10k steps. The average JSD drops sharply during the first $\sim$1--2k steps for spatial, trip, and POI frequency metrics, while travel distance JSD continues improving through $\sim$4--6k steps before stabilizing. The per-group curves show consistent early improvement across all eight demographic groups, followed by group-dependent plateaus and mild rebounds in some metrics.

\begin{figure}[!ht]
\centering
\includegraphics[width=\columnwidth]{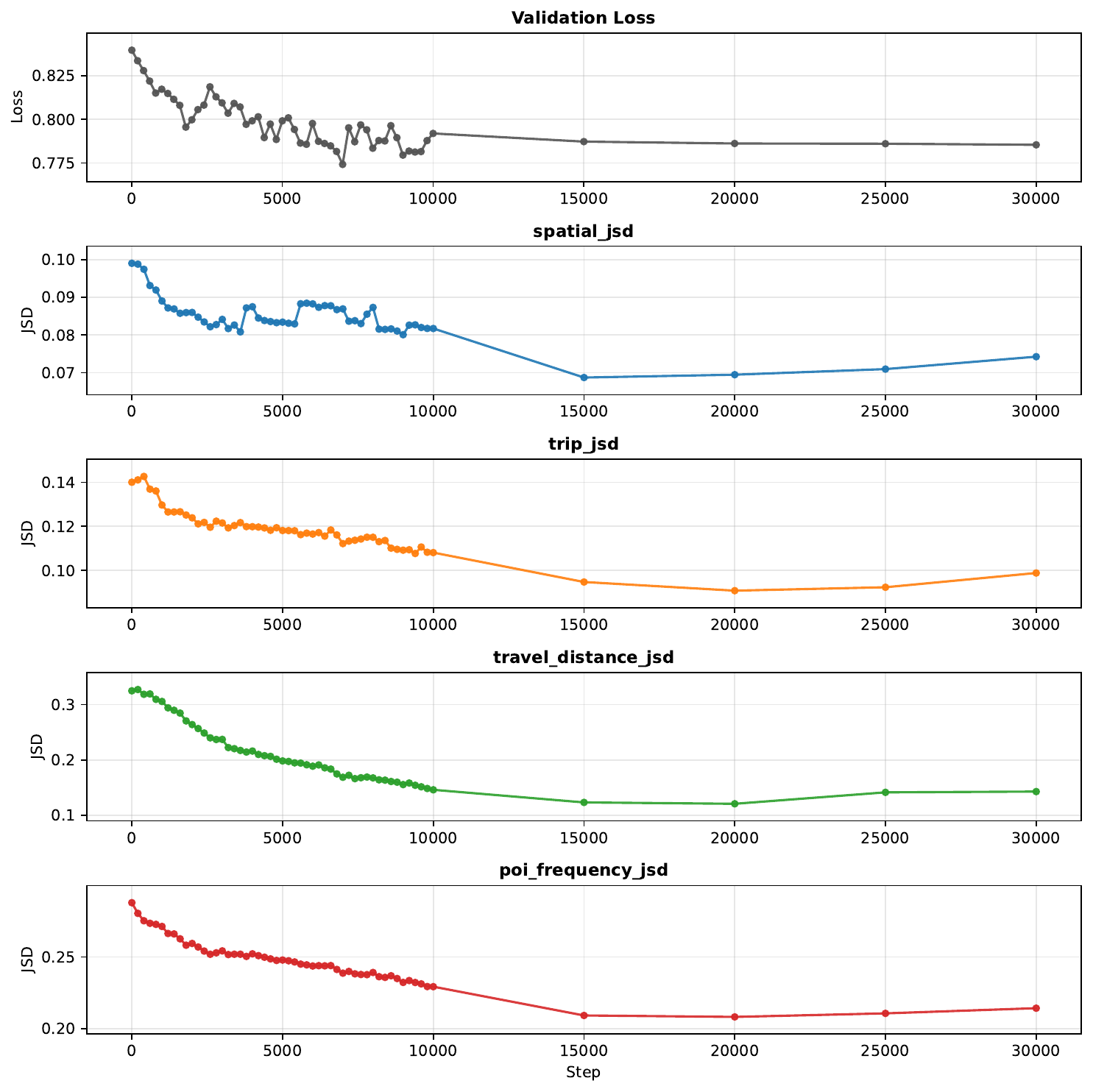}
\caption{\textbf{Virginia: average JSD during finetuning.} Validation loss (TV) and per-metric JSD (averaged across demographic groups) over training steps.}
\label{fig:va-training-avg}
\end{figure}

\begin{figure}[!ht]
\centering
\includegraphics[width=\columnwidth]{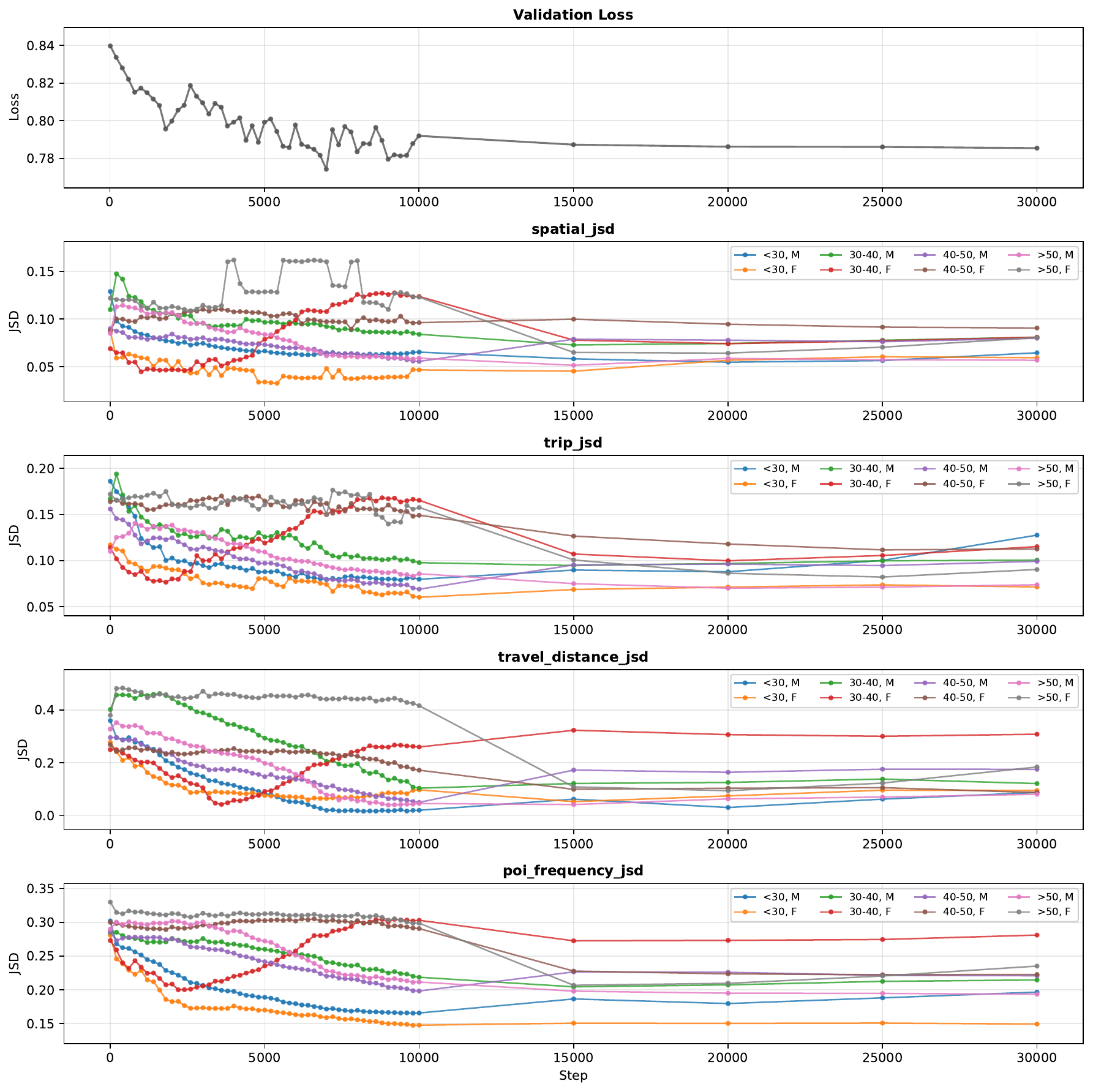}
\caption{\textbf{Virginia: per-group JSD during finetuning.} Per-metric JSD broken down by demographic group over training steps.}
\label{fig:va-training-group}
\end{figure}

\begin{figure}[!ht]
\centering
\includegraphics[width=\columnwidth]{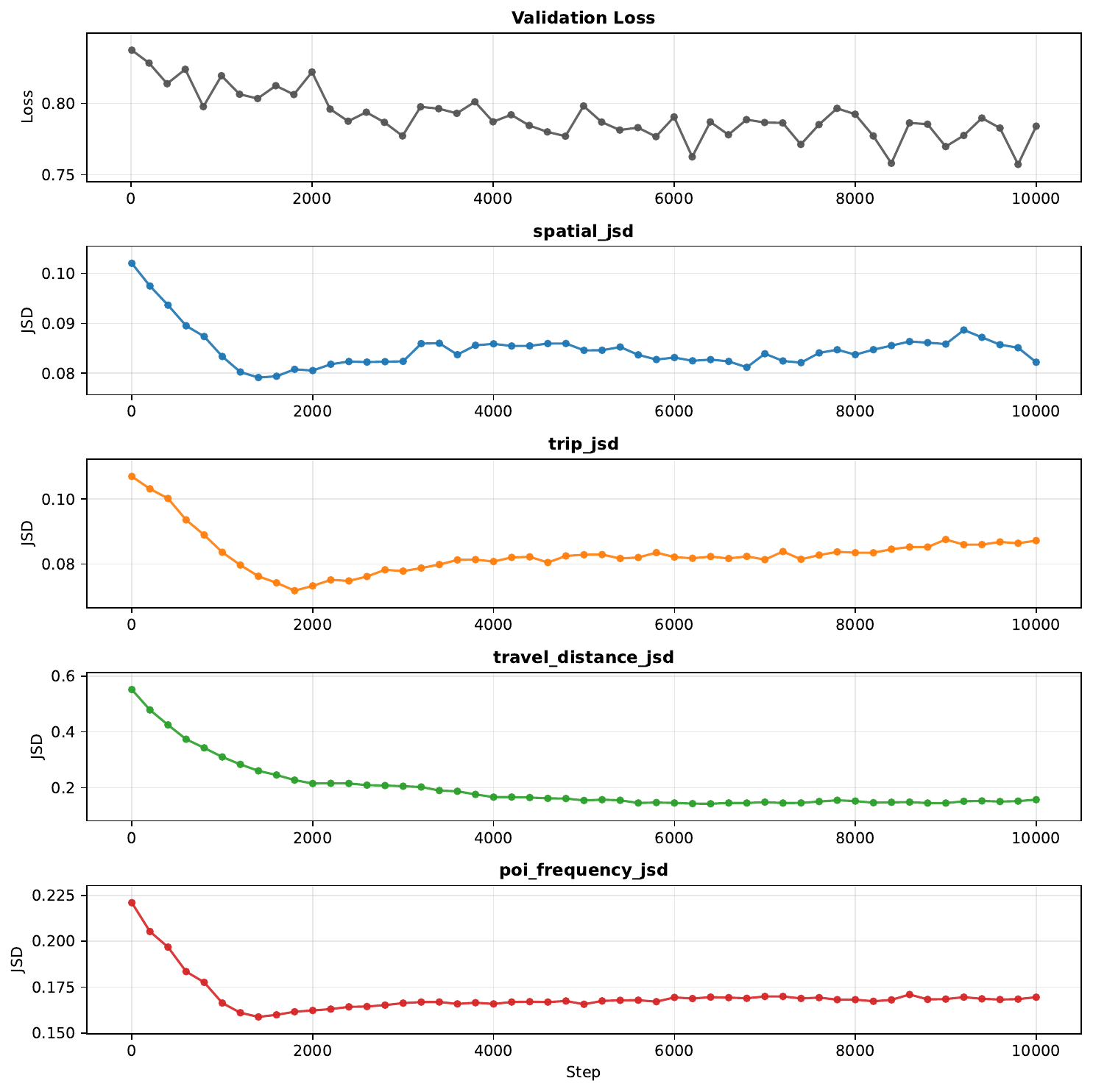}
\caption{\textbf{California: average JSD during finetuning.} Validation loss (TV) and per-metric test JSD (averaged across demographic groups) over training steps.}
\label{fig:ca-training-avg}
\end{figure}

\begin{figure}[!ht]
\centering
\includegraphics[width=\columnwidth]{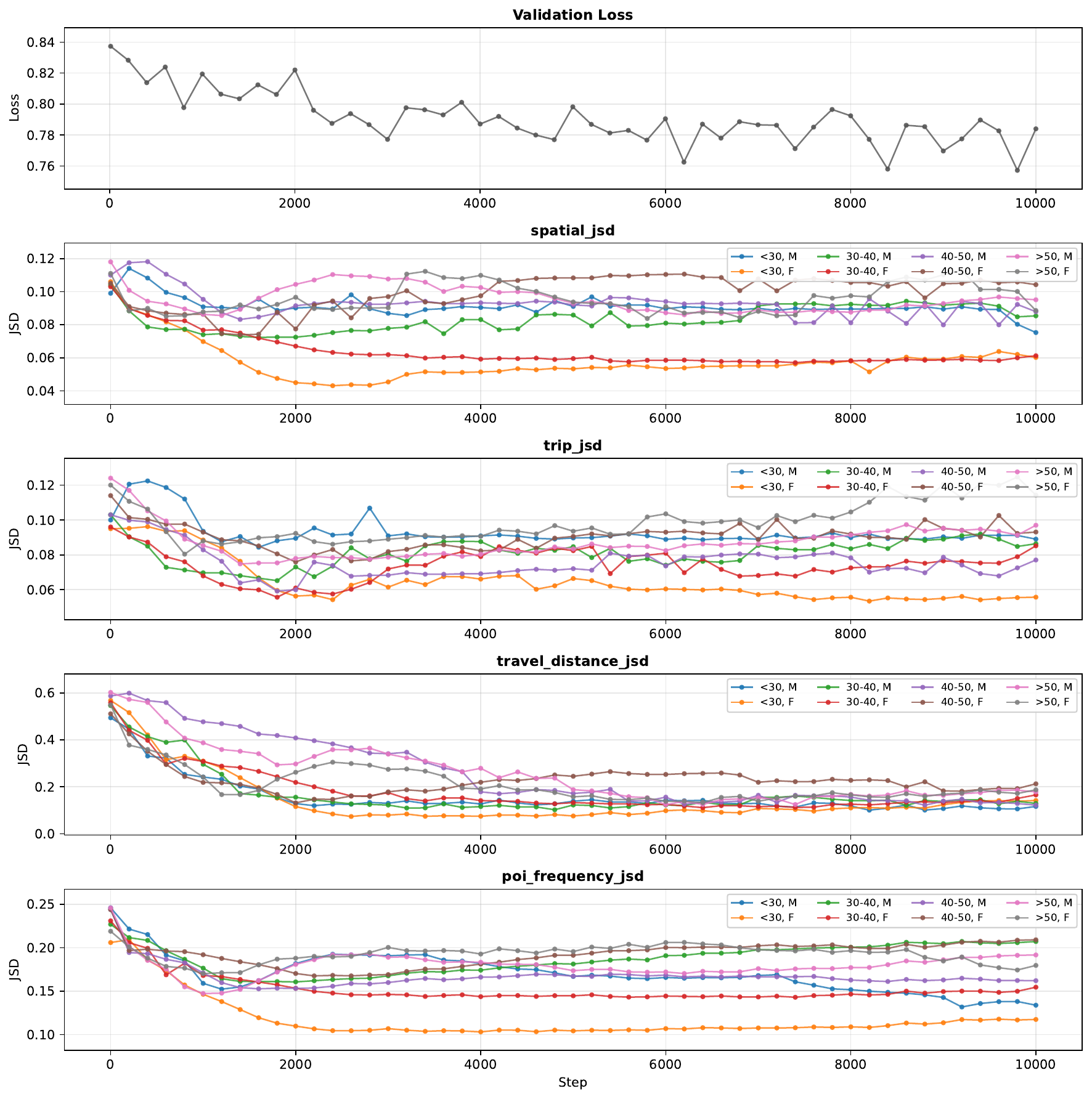}
\caption{\textbf{California: per-group JSD during finetuning.} Per-metric JSD broken down by demographic group over training steps.}
\label{fig:ca-training-group}
\end{figure}

\section{Experimental Results}
\label{sec:app-experiments}
In this section, we provide additional experimental details and results. In our first and second set of experiments, we use Jensen-Shannon divergence (JSD) \citep{lin2002divergence} to quantify the distance between the group-level synthetic trajectories (generated by \NAME{}, the baseline, or the strongly supervised model) and the real trajectories, by comparing their respective distributions over some discretized trajectory statistic, such as spatial grid cells or binned distances (Section~\ref{sec:eval-metrics}).
For discrete distributions $P$ and $Q$ over a finite space, JSD is defined as
\[
\mathrm{JSD}(P \| Q) := \frac{1}{2}\mathrm{KL}(P \| M) + \frac{1}{2}\mathrm{KL}(Q \| M),
\]
where $M = \frac{1}{2}(P + Q)$ is the mixture distribution and $\mathrm{KL}(P\|Q) = \sum_x P(x)\log\frac{P(x)}{Q(x)}$ is the Kullback--Leibler divergence \citep{kullback1951information}.

To quantify how much of the performance gap \NAME{} closes between the spatially-anchored baseline and strong supervision, we define:
\[
\text{Gap Closed} = 1 - \frac{\text{metric}(\NAME) - \text{metric}(\text{Strong})}{\text{metric}(\text{Baseline}) - \text{metric}(\text{Strong})}.
\]
A value of 1 indicates \NAME{} matches strong supervision; 0 indicates no improvement over baseline; negative values indicate worse-than-baseline performance.

\subsection{RQ1: Effect of Demographic Diversity}
\subsubsection{Partition construction matrices.}
\label{sec:app-partitions}
In our demographic-diversity experiments (Section~\ref{sec:demo-diversity-results}), we regroup the same labeled individuals into $G=8$ synthetic regions to obtain different demographic-mixture matrices $P$ with controlled rank/conditioning.
Concretely, we first construct an integer region-by-group count matrix $C$ (how many labeled individuals of each demographic group are assigned to each region), and then row-normalize to obtain the mixture matrix
\[
P_{g,d} \;=\; p(d\mid g) \;=\; \frac{C_{g,d}}{\sum_{d'=1}^K C_{g,d'}}.
\]

Below we report the resulting $P$ matrices (row-normalized proportions).
Each entry is a demographic-mixture proportion $p(d\mid g)$ for one region--group cell.
We index age by $a_i$ and gender by $g_j$, where $a_0<30$, $a_1=30$--$40$, $a_2=40$--$50$, $a_3>50$, and $g_0=M$, $g_1=F$.
Thus, the column order $(a_0g_0,a_0g_1,a_1g_0,a_1g_1,a_2g_0,a_2g_1,a_3g_0,a_3g_1)$ matches the demographic-group order used throughout our per-group JSD tables.
Values are shown to 6 decimal places; they are computed by exact row-normalization of the underlying integer count matrices (so near-equal mixtures are not rounded to identical values).
\begin{itemize}[leftmargin=*, itemsep=0.1em]
\item DemoGroups (rank $=8$): $P = I_8$, i.e., the region-level and group-level partitions are identical. 
    \item FullRank (rank $=8$).
\end{itemize}
\scriptsize
\begin{verbatim}
Row   a0_g0     a0_g1     a1_g0     a1_g1     a2_g0     a2_g1     a3_g0     a3_g1
R1: 0.000000  0.000000  0.000000  0.500101  0.499899  0.000000  0.000000  0.000000
R2: 0.000000  0.000000  0.500101  0.000000  0.000000  0.000000  0.000000  0.499899
R3: 0.000000  0.500101  0.000000  0.000000  0.000000  0.000000  0.499899  0.000000
R4: 0.500101  0.000000  0.000000  0.000000  0.000000  0.499899  0.000000  0.000000
R5: 0.500101  0.000000  0.000000  0.499899  0.000000  0.000000  0.000000  0.000000
R6: 0.000000  0.000000  0.500101  0.000000  0.000000  0.499899  0.000000  0.000000
R7: 0.000000  0.500101  0.000000  0.000000  0.499899  0.000000  0.000000  0.000000
R8: 0.000000  0.000000  0.000000  0.000000  0.000000  0.000000  0.500101  0.499899
\end{verbatim}
\normalsize

\begin{itemize}[leftmargin=*, itemsep=0.1em]
    \item RankDef (rank $=7$).
\end{itemize}
\scriptsize
\begin{verbatim}
Row   a0_g0     a0_g1     a1_g0     a1_g1     a2_g0     a2_g1     a3_g0     a3_g1
R0: 0.500000  0.000000  0.000000  0.500000  0.000000  0.000000  0.000000  0.000000
R1: 0.000000  0.500000  0.000000  0.000000  0.500000  0.000000  0.000000  0.000000
R2: 0.000000  0.000000  0.500000  0.000000  0.000000  0.000000  0.000000  0.500000
R3: 0.000000  0.000000  0.000000  0.000000  0.000000  0.500000  0.500000  0.000000
R4: 0.500000  0.000000  0.000000  0.000000  0.000000  0.500000  0.000000  0.000000
R5: 0.000000  0.000000  0.000000  0.500000  0.000000  0.000000  0.500000  0.000000
R6: 0.125130  0.125130  0.125130  0.125130  0.124870  0.124870  0.124870  0.124870
R7: 0.299948  0.100156  0.000000  0.250000  0.099896  0.150104  0.099896  0.000000
\end{verbatim}
\normalsize

\begin{itemize}[leftmargin=*, itemsep=0.1em]
    \item Messy (rank $=4$).
\end{itemize}
\scriptsize
\begin{verbatim}
Row   a0_g0     a0_g1     a1_g0     a1_g1     a2_g0     a2_g1     a3_g0     a3_g1
R0: 0.250000  0.250000  0.250000  0.250000  0.000000  0.000000  0.000000  0.000000
R1: 0.200000  0.000000  0.200000  0.200000  0.200000  0.200000  0.000000  0.000000
R2: 0.000000  0.000000  0.000000  0.250000  0.250000  0.250000  0.250000  0.000000
R3: 0.000000  0.200000  0.200000  0.000000  0.000000  0.200000  0.200000  0.200000
R4: 0.225000  0.125000  0.225000  0.225000  0.100000  0.100000  0.000000  0.000000
R5: 0.000000  0.100000  0.100000  0.125000  0.125000  0.225000  0.225000  0.100000
R6: 0.140000  0.060000  0.200000  0.140000  0.140000  0.200000  0.060000  0.060000
R7: 0.150000  0.150000  0.150000  0.250000  0.100000  0.100000  0.100000  0.000000
\end{verbatim}
\normalsize

\subsubsection{Results by partition type}
\begin{table*}[t]
  \centering
  \caption{Effect of demographic diversity (RQ1), Virginia. Average metrics across demographic groups for all partition structures (mean $\pm$ std over 8 groups). Lower is better for JSD and higher is better for Gap Closed. Bold indicates best \NAME{} variant within each partition type.}
  \label{tab:main_results}
  \resizebox{0.95\textwidth}{!}{%
  \begin{tabular}{lcccccccc}
  \toprule
   & \multicolumn{2}{c}{Spatial JSD$\downarrow$} & \multicolumn{2}{c}{Travel Dist. JSD$\downarrow$} & \multicolumn{2}{c}{Trip JSD$\downarrow$} & \multicolumn{2}{c}{POI Freq. JSD$\downarrow$} \\
  \cmidrule(lr){2-3}\cmidrule(lr){4-5}\cmidrule(lr){6-7}\cmidrule(lr){8-9}
  Setup & Value & Gap Closed (\%) & Value & Gap Closed (\%) & Value & Gap Closed (\%) & Value & Gap Closed (\%) \\
  \midrule
  Baseline (no demo) & $0.105 \pm \scriptstyle 0.019$ & -- & $0.382 \pm \scriptstyle 0.045$ & -- & $0.157 \pm \scriptstyle 0.030$ & -- & $0.304 \pm \scriptstyle 0.016$ & -- \\
  Strong (supervised) & $0.059 \pm \scriptstyle 0.013$ & 100 & $0.158 \pm \scriptstyle 0.043$ & 100 & $0.122 \pm \scriptstyle 0.023$ & 100 & $0.249 \pm \scriptstyle 0.031$ & 100 \\
  \midrule
  \textbf{Demogroups-JS} & \textbf{$0.080 \pm \scriptstyle 0.008$} & \textbf{$54$} & \textbf{$0.124 \pm \scriptstyle 0.027$} & \textbf{$115$} & \textbf{$0.105 \pm \scriptstyle 0.019$} & \textbf{$149$} & \textbf{$0.211 \pm \scriptstyle 0.023$} & \textbf{$169$} \\
  Demogroups-TV & $0.087 \pm \scriptstyle 0.018$ & $39$ & $0.131 \pm \scriptstyle 0.067$ & $112$ & $0.125 \pm \scriptstyle 0.017$ & $91$ & $0.238 \pm \scriptstyle 0.026$ & $120$ \\
  \midrule
  \textbf{FullRank-JS} & \textbf{$0.090 \pm \scriptstyle 0.016$} & \textbf{$33$} & \textbf{$0.167 \pm \scriptstyle 0.037$} & \textbf{$96$} & \textbf{$0.121 \pm \scriptstyle 0.021$} & \textbf{$103$} & \textbf{$0.243 \pm \scriptstyle 0.022$} & \textbf{$111$} \\
  FullRank-TV & $0.095 \pm \scriptstyle 0.035$ & $22$ & $0.174 \pm \scriptstyle 0.073$ & $93$ & $0.126 \pm \scriptstyle 0.042$ & $89$ & $0.258 \pm \scriptstyle 0.066$ & $84$ \\
  \midrule
  \textbf{RankDef-JS} & \textbf{$0.100 \pm \scriptstyle 0.024$} & \textbf{$11$} & \textbf{$0.214 \pm \scriptstyle 0.085$} & \textbf{$75$} & \textbf{$0.139 \pm \scriptstyle 0.025$} & \textbf{$51$} & \textbf{$0.267 \pm \scriptstyle 0.022$} & \textbf{$67$} \\
  RankDef-TV & $0.142 \pm \scriptstyle 0.036$ & $-80$ & $0.214 \pm \scriptstyle 0.069$ & $75$ & $0.182 \pm \scriptstyle 0.040$ & $-71$ & $0.331 \pm \scriptstyle 0.049$ & $-49$ \\
  \midrule
  \textbf{Messy-JS} & \textbf{$0.121 \pm \scriptstyle 0.006$} & \textbf{$-35$} & $0.233 \pm \scriptstyle 0.080$ & $67$ & \textbf{$0.164 \pm \scriptstyle 0.008$} & \textbf{$-20$} & \textbf{$0.289 \pm \scriptstyle 0.008$} & \textbf{$27$} \\
  Messy-TV & $0.153 \pm \scriptstyle 0.011$ & $-104$ & \textbf{$0.231 \pm \scriptstyle 0.057$} & \textbf{$67$} & $0.201 \pm \scriptstyle 0.014$ & $-126$ & $0.342 \pm \scriptstyle 0.036$ & $-69$ \\
  \bottomrule
  \end{tabular}}
  \end{table*}

\begin{table*}[t]
  \centering
  \caption{Effect of demographic diversity (RQ1), Virginia - \textit{spatial JSD} per demographic group. Format: mean $\pm$ std across 3 seeds. The Avg column reports mean $\pm$ std across 8 demographic groups (matching Table~\ref{tab:main_results}). Bold indicates best \NAME{} variant within each partition type.}
  \label{tab:spatial_jsd_pergroup}
  \resizebox{\textwidth}{!}{%
  \begin{tabular}{lccccccccc}
  \toprule
  Setup & <30, M & <30, F & 30--40, M & 30--40, F & 40--50, M & 40--50, F & >50, M & >50, F & Avg \\
  \midrule
  Baseline & $0.131 \pm \scriptstyle 0.002$ & $0.089 \pm \scriptstyle 0.002$ & $0.119 \pm \scriptstyle 0.009$ & $0.077 \pm \scriptstyle 0.008$ & $0.105 \pm \scriptstyle 0.015$ & $0.099 \pm \scriptstyle 0.011$ & $0.093 \pm \scriptstyle 0.008$ & $0.123 \pm \scriptstyle 0.001$ & $0.105 \pm \scriptstyle 0.019$ \\
  Strong & $0.075 \pm \scriptstyle 0.006$ & $0.051 \pm \scriptstyle 0.007$ & $0.045 \pm \scriptstyle 0.001$ & $0.048 \pm \scriptstyle 0.000$ & $0.058 \pm \scriptstyle 0.004$ & $0.068 \pm \scriptstyle 0.002$ & $0.051 \pm \scriptstyle 0.000$ & $0.079 \pm \scriptstyle 0.006$ & $0.059 \pm \scriptstyle 0.013$ \\
  \midrule
  \textbf{Demogroups-JS} & $0.072 \pm \scriptstyle 0.018$ & $0.060 \pm \scriptstyle 0.008$ & $0.092 \pm \scriptstyle 0.011$ & $0.057 \pm \scriptstyle 0.029$ & $0.086 \pm \scriptstyle 0.010$ & $0.093 \pm \scriptstyle 0.008$ & $0.091 \pm \scriptstyle 0.008$ & $0.089 \pm \scriptstyle 0.037$ & \textbf{$0.080 \pm \scriptstyle 0.008$} \\
  Demogroups-TV & $0.075 \pm \scriptstyle 0.025$ & $0.072 \pm \scriptstyle 0.030$ & $0.110 \pm \scriptstyle 0.045$ & $0.066 \pm \scriptstyle 0.022$ & $0.101 \pm \scriptstyle 0.051$ & $0.095 \pm \scriptstyle 0.018$ & $0.106 \pm \scriptstyle 0.049$ & $0.071 \pm \scriptstyle 0.012$ & $0.087 \pm \scriptstyle 0.018$ \\
  \midrule
  \textbf{FullRank-JS} & $0.105 \pm \scriptstyle 0.028$ & $0.076 \pm \scriptstyle 0.033$ & $0.097 \pm \scriptstyle 0.011$ & $0.058 \pm \scriptstyle 0.006$ & $0.098 \pm \scriptstyle 0.045$ & $0.099 \pm \scriptstyle 0.014$ & $0.102 \pm \scriptstyle 0.009$ & $0.084 \pm \scriptstyle 0.001$ & \textbf{$0.090 \pm \scriptstyle 0.016$} \\
  FullRank-TV & $0.074 \pm \scriptstyle 0.005$ & $0.054 \pm \scriptstyle 0.006$ & $0.102 \pm \scriptstyle 0.017$ & $0.057 \pm \scriptstyle 0.009$ & $0.108 \pm \scriptstyle 0.035$ & $0.099 \pm \scriptstyle 0.004$ & $0.164 \pm \scriptstyle 0.042$ & $0.101 \pm \scriptstyle 0.013$ & $0.095 \pm \scriptstyle 0.035$ \\
  \midrule
  \textbf{RankDef-JS} & $0.092 \pm \scriptstyle 0.017$ & $0.052 \pm \scriptstyle 0.004$ & $0.104 \pm \scriptstyle 0.005$ & $0.122 \pm \scriptstyle 0.081$ & $0.099 \pm \scriptstyle 0.050$ & $0.125 \pm \scriptstyle 0.011$ & $0.116 \pm \scriptstyle 0.047$ & $0.087 \pm \scriptstyle 0.008$ & \textbf{$0.100 \pm \scriptstyle 0.024$} \\
  RankDef-TV & $0.200 \pm \scriptstyle 0.098$ & $0.119 \pm \scriptstyle 0.044$ & $0.160 \pm \scriptstyle 0.044$ & $0.087 \pm \scriptstyle 0.051$ & $0.144 \pm \scriptstyle 0.093$ & $0.122 \pm \scriptstyle 0.033$ & $0.178 \pm \scriptstyle 0.114$ & $0.124 \pm \scriptstyle 0.070$ & $0.142 \pm \scriptstyle 0.036$ \\
  \midrule
  \textbf{Messy-JS} & $0.068 \pm \scriptstyle 0.011$ & $0.069 \pm \scriptstyle 0.020$ & $0.202 \pm \scriptstyle 0.005$ & $0.083 \pm \scriptstyle 0.056$ & $0.263 \pm \scriptstyle 0.081$ & $0.118 \pm \scriptstyle 0.024$ & $0.102 \pm \scriptstyle 0.054$ & $0.064 \pm \scriptstyle 0.001$ & $0.121 \pm \scriptstyle 0.006$ \\
  Messy-TV & $0.076 \pm \scriptstyle 0.023$ & $0.179 \pm \scriptstyle 0.100$ & $0.122 \pm \scriptstyle 0.036$ & $0.207 \pm \scriptstyle 0.015$ & $0.103 \pm \scriptstyle 0.080$ & $0.228 \pm \scriptstyle 0.027$ & $0.086 \pm \scriptstyle 0.038$ & $0.225 \pm \scriptstyle 0.109$ & $0.153 \pm \scriptstyle 0.011$ \\
  \bottomrule
  \end{tabular}}
  \end{table*}

\begin{table*}[t]
\centering
\caption{Effect of demographic diversity (RQ1), Virginia - \textit{travel distance JSD} per demographic group. Format: mean $\pm$ std across 3 seeds. The Avg column reports mean $\pm$ std across 8 demographic groups (matching Table~\ref{tab:main_results}). Bold indicates best \NAME{} variant within each partition type.}
\label{tab:travel_distance_jsd_pergroup}
\resizebox{\textwidth}{!}{%
\begin{tabular}{lccccccccc}
\toprule
Setup & <30, M & <30, F & 30--40, M & 30--40, F & 40--50, M & 40--50, F & >50, M & >50, F & Avg \\
\midrule
Baseline & $0.398 \pm \scriptstyle 0.038$ & $0.353 \pm \scriptstyle 0.075$ & $0.455 \pm \scriptstyle 0.053$ & $0.324 \pm \scriptstyle 0.074$ & $0.366 \pm \scriptstyle 0.070$ & $0.350 \pm \scriptstyle 0.081$ & $0.387 \pm \scriptstyle 0.059$ & $0.425 \pm \scriptstyle 0.045$ & $0.382 \pm \scriptstyle 0.045$ \\
Strong & $0.125 \pm \scriptstyle 0.002$ & $0.178 \pm \scriptstyle 0.003$ & $0.140 \pm \scriptstyle 0.003$ & $0.134 \pm \scriptstyle 0.002$ & $0.147 \pm \scriptstyle 0.003$ & $0.164 \pm \scriptstyle 0.008$ & $0.121 \pm \scriptstyle 0.004$ & $0.253 \pm \scriptstyle 0.020$ & $0.158 \pm \scriptstyle 0.043$ \\
\midrule
\textbf{Demogroups-JS} & $0.127 \pm \scriptstyle 0.078$ & $0.121 \pm \scriptstyle 0.072$ & $0.146 \pm \scriptstyle 0.066$ & $0.151 \pm \scriptstyle 0.013$ & $0.160 \pm \scriptstyle 0.096$ & $0.084 \pm \scriptstyle 0.029$ & $0.106 \pm \scriptstyle 0.070$ & $0.097 \pm \scriptstyle 0.005$ & \textbf{$0.124 \pm \scriptstyle 0.027$} \\
Demogroups-TV & $0.079 \pm \scriptstyle 0.043$ & $0.099 \pm \scriptstyle 0.040$ & $0.081 \pm \scriptstyle 0.054$ & $0.282 \pm \scriptstyle 0.045$ & $0.149 \pm \scriptstyle 0.015$ & $0.090 \pm \scriptstyle 0.046$ & $0.155 \pm \scriptstyle 0.087$ & $0.116 \pm \scriptstyle 0.062$ & $0.131 \pm \scriptstyle 0.067$ \\
\midrule
\textbf{FullRank-JS} & $0.180 \pm \scriptstyle 0.117$ & $0.169 \pm \scriptstyle 0.096$ & $0.105 \pm \scriptstyle 0.058$ & $0.144 \pm \scriptstyle 0.023$ & $0.234 \pm \scriptstyle 0.138$ & $0.173 \pm \scriptstyle 0.009$ & $0.178 \pm \scriptstyle 0.135$ & $0.150 \pm \scriptstyle 0.028$ & \textbf{$0.167 \pm \scriptstyle 0.037$} \\
FullRank-TV & $0.075 \pm \scriptstyle 0.003$ & $0.114 \pm \scriptstyle 0.021$ & $0.150 \pm \scriptstyle 0.009$ & $0.146 \pm \scriptstyle 0.054$ & $0.210 \pm \scriptstyle 0.106$ & $0.162 \pm \scriptstyle 0.084$ & $0.313 \pm \scriptstyle 0.110$ & $0.219 \pm \scriptstyle 0.061$ & $0.174 \pm \scriptstyle 0.073$ \\
\midrule
\textbf{RankDef-JS} & $0.277 \pm \scriptstyle 0.150$ & $0.044 \pm \scriptstyle 0.012$ & $0.259 \pm \scriptstyle 0.097$ & $0.151 \pm \scriptstyle 0.018$ & $0.297 \pm \scriptstyle 0.097$ & $0.195 \pm \scriptstyle 0.170$ & $0.283 \pm \scriptstyle 0.020$ & $0.202 \pm \scriptstyle 0.024$ & \textbf{$0.214 \pm \scriptstyle 0.085$} \\
RankDef-TV & $0.083 \pm \scriptstyle 0.046$ & $0.202 \pm \scriptstyle 0.113$ & $0.243 \pm \scriptstyle 0.153$ & $0.186 \pm \scriptstyle 0.030$ & $0.226 \pm \scriptstyle 0.191$ & $0.312 \pm \scriptstyle 0.056$ & $0.183 \pm \scriptstyle 0.199$ & $0.277 \pm \scriptstyle 0.098$ & $0.214 \pm \scriptstyle 0.069$ \\
\midrule
\textbf{Messy-JS} & $0.295 \pm \scriptstyle 0.105$ & $0.247 \pm \scriptstyle 0.157$ & $0.115 \pm \scriptstyle 0.027$ & $0.283 \pm \scriptstyle 0.132$ & $0.178 \pm \scriptstyle 0.002$ & $0.213 \pm \scriptstyle 0.097$ & $0.364 \pm \scriptstyle 0.068$ & $0.171 \pm \scriptstyle 0.068$ & \textbf{$0.233 \pm \scriptstyle 0.080$} \\
Messy-TV & $0.170 \pm \scriptstyle 0.130$ & $0.270 \pm \scriptstyle 0.069$ & $0.045 \pm \scriptstyle 0.022$ & $0.209 \pm \scriptstyle 0.155$ & $0.124 \pm \scriptstyle 0.057$ & $0.344 \pm \scriptstyle 0.161$ & $0.222 \pm \scriptstyle 0.178$ & $0.458 \pm \scriptstyle 0.021$ & $0.231 \pm \scriptstyle 0.057$ \\
\bottomrule
\end{tabular}}
\end{table*}

\begin{table*}[t]
  \centering
  \caption{Effect of demographic diversity (RQ1), Virginia - \textit{trip JSD} per demographic group. Format: mean $\pm$ std across 3 seeds. The Avg column reports mean $\pm$ std across 8 demographic groups (matching Table~\ref{tab:main_results}). Bold indicates best \NAME{} variant within each partition type.}
  \label{tab:trip_jsd_pergroup}
  \resizebox{\textwidth}{!}{%
  \begin{tabular}{lccccccccc}
  \toprule
  Setup & <30, M & <30, F & 30--40, M & 30--40, F & 40--50, M & 40--50, F & >50, M & >50, F & Avg \\
  \midrule
  Baseline & $0.192 \pm \scriptstyle 0.006$ & $0.122 \pm \scriptstyle 0.005$ & $0.177 \pm \scriptstyle 0.010$ & $0.119 \pm \scriptstyle 0.005$ & $0.165 \pm \scriptstyle 0.009$ & $0.170 \pm \scriptstyle 0.006$ & $0.128 \pm \scriptstyle 0.018$ & $0.186 \pm \scriptstyle 0.014$ & $0.157 \pm \scriptstyle 0.030$ \\
  Strong & $0.146 \pm \scriptstyle 0.000$ & $0.108 \pm \scriptstyle 0.016$ & $0.113 \pm \scriptstyle 0.002$ & $0.092 \pm \scriptstyle 0.004$ & $0.127 \pm \scriptstyle 0.007$ & $0.154 \pm \scriptstyle 0.000$ & $0.098 \pm \scriptstyle 0.001$ & $0.140 \pm \scriptstyle 0.020$ & $0.122 \pm \scriptstyle 0.023$ \\
  \midrule
  \textbf{Demogroups-JS} & $0.112 \pm \scriptstyle 0.027$ & $0.068 \pm \scriptstyle 0.008$ & $0.101 \pm \scriptstyle 0.008$ & $0.095 \pm \scriptstyle 0.045$ & $0.121 \pm \scriptstyle 0.029$ & $0.127 \pm \scriptstyle 0.014$ & $0.115 \pm \scriptstyle 0.017$ & $0.100 \pm \scriptstyle 0.058$ & \textbf{$0.105 \pm \scriptstyle 0.019$} \\
  Demogroups-TV & $0.133 \pm \scriptstyle 0.031$ & $0.092 \pm \scriptstyle 0.030$ & $0.138 \pm \scriptstyle 0.050$ & $0.105 \pm \scriptstyle 0.032$ & $0.142 \pm \scriptstyle 0.088$ & $0.131 \pm \scriptstyle 0.033$ & $0.131 \pm \scriptstyle 0.061$ & $0.127 \pm \scriptstyle 0.052$ & $0.125 \pm \scriptstyle 0.017$ \\
  \midrule
  \textbf{FullRank-JS} & $0.153 \pm \scriptstyle 0.040$ & $0.098 \pm \scriptstyle 0.024$ & $0.113 \pm \scriptstyle 0.013$ & $0.096 \pm \scriptstyle 0.014$ & $0.136 \pm \scriptstyle 0.052$ & $0.141 \pm \scriptstyle 0.028$ & $0.122 \pm \scriptstyle 0.029$ & $0.108 \pm \scriptstyle 0.014$ & \textbf{$0.121 \pm \scriptstyle 0.021$} \\
  FullRank-TV & $0.118 \pm \scriptstyle 0.006$ & $0.082 \pm \scriptstyle 0.022$ & $0.112 \pm \scriptstyle 0.008$ & $0.071 \pm \scriptstyle 0.014$ & $0.144 \pm \scriptstyle 0.054$ & $0.139 \pm \scriptstyle 0.026$ & $0.207 \pm \scriptstyle 0.034$ & $0.135 \pm \scriptstyle 0.013$ & $0.126 \pm \scriptstyle 0.042$ \\
  \midrule
  \textbf{RankDef-JS} & $0.159 \pm \scriptstyle 0.018$ & $0.125 \pm \scriptstyle 0.040$ & $0.148 \pm \scriptstyle 0.024$ & $0.112 \pm \scriptstyle 0.003$ & $0.102 \pm \scriptstyle 0.000$ & $0.150 \pm \scriptstyle 0.003$ & $0.136 \pm \scriptstyle 0.006$ & $0.177 \pm \scriptstyle 0.019$ & \textbf{$0.139 \pm \scriptstyle 0.025$} \\
  RankDef-TV & $0.252 \pm \scriptstyle 0.087$ & $0.182 \pm \scriptstyle 0.070$ & $0.204 \pm \scriptstyle 0.071$ & $0.112 \pm \scriptstyle 0.043$ & $0.186 \pm \scriptstyle 0.101$ & $0.171 \pm \scriptstyle 0.053$ & $0.192 \pm \scriptstyle 0.081$ & $0.155 \pm \scriptstyle 0.075$ & $0.182 \pm \scriptstyle 0.040$ \\
  \midrule
  \textbf{Messy-JS} & $0.091 \pm \scriptstyle 0.013$ & $0.080 \pm \scriptstyle 0.030$ & $0.283 \pm \scriptstyle 0.032$ & $0.120 \pm \scriptstyle 0.065$ & $0.336 \pm \scriptstyle 0.052$ & $0.180 \pm \scriptstyle 0.054$ & $0.138 \pm \scriptstyle 0.056$ & $0.083 \pm \scriptstyle 0.001$ & $0.164 \pm \scriptstyle 0.008$ \\
  Messy-TV & $0.114 \pm \scriptstyle 0.017$ & $0.225 \pm \scriptstyle 0.100$ & $0.138 \pm \scriptstyle 0.052$ & $0.262 \pm \scriptstyle 0.015$ & $0.150 \pm \scriptstyle 0.094$ & $0.322 \pm \scriptstyle 0.043$ & $0.126 \pm \scriptstyle 0.036$ & $0.273 \pm \scriptstyle 0.160$ & $0.201 \pm \scriptstyle 0.014$ \\
  \bottomrule
\end{tabular}}
\end{table*}

\begin{table*}[t]
\centering
\caption{Effect of demographic diversity (RQ1), Virginia - \textit{POI frequency JSD} per demographic group. Format: mean $\pm$ std across 3 seeds. The Avg column reports mean $\pm$ std across 8 demographic groups (matching Table~\ref{tab:main_results}). Bold indicates best \NAME{} variant within each partition type.}
\label{tab:poi_freq_pergroup}
\resizebox{\textwidth}{!}{%
\begin{tabular}{lccccccccc}
\toprule
Setup & <30, M & <30, F & 30--40, M & 30--40, F & 40--50, M & 40--50, F & >50, M & >50, F & Avg \\
\midrule
Baseline & $0.307 \pm \scriptstyle 0.005$ & $0.291 \pm \scriptstyle 0.010$ & $0.292 \pm \scriptstyle 0.007$ & $0.291 \pm \scriptstyle 0.018$ & $0.308 \pm \scriptstyle 0.022$ & $0.318 \pm \scriptstyle 0.018$ & $0.290 \pm \scriptstyle 0.000$ & $0.332 \pm \scriptstyle 0.002$ & $0.304 \pm \scriptstyle 0.016$ \\
Strong & $0.252 \pm \scriptstyle 0.005$ & $0.243 \pm \scriptstyle 0.001$ & $0.189 \pm \scriptstyle 0.001$ & $0.242 \pm \scriptstyle 0.001$ & $0.262 \pm \scriptstyle 0.006$ & $0.279 \pm \scriptstyle 0.001$ & $0.232 \pm \scriptstyle 0.003$ & $0.291 \pm \scriptstyle 0.001$ & $0.249 \pm \scriptstyle 0.031$ \\
\midrule
\textbf{Demogroups-JS} & $0.186 \pm \scriptstyle 0.018$ & $0.168 \pm \scriptstyle 0.020$ & $0.220 \pm \scriptstyle 0.016$ & $0.202 \pm \scriptstyle 0.024$ & $0.227 \pm \scriptstyle 0.037$ & $0.224 \pm \scriptstyle 0.000$ & $0.229 \pm \scriptstyle 0.031$ & $0.231 \pm \scriptstyle 0.029$ & \textbf{$0.211 \pm \scriptstyle 0.023$} \\
Demogroups-TV & $0.211 \pm \scriptstyle 0.049$ & $0.204 \pm \scriptstyle 0.077$ & $0.223 \pm \scriptstyle 0.028$ & $0.261 \pm \scriptstyle 0.057$ & $0.262 \pm \scriptstyle 0.059$ & $0.245 \pm \scriptstyle 0.028$ & $0.273 \pm \scriptstyle 0.080$ & $0.222 \pm \scriptstyle 0.012$ & $0.238 \pm \scriptstyle 0.026$ \\
\midrule
\textbf{FullRank-JS} & $0.262 \pm \scriptstyle 0.041$ & $0.207 \pm \scriptstyle 0.065$ & $0.229 \pm \scriptstyle 0.010$ & $0.228 \pm \scriptstyle 0.009$ & $0.277 \pm \scriptstyle 0.087$ & $0.250 \pm \scriptstyle 0.028$ & $0.255 \pm \scriptstyle 0.042$ & $0.235 \pm \scriptstyle 0.010$ & \textbf{$0.243 \pm \scriptstyle 0.022$} \\
FullRank-TV & $0.215 \pm \scriptstyle 0.014$ & $0.169 \pm \scriptstyle 0.013$ & $0.246 \pm \scriptstyle 0.029$ & $0.212 \pm \scriptstyle 0.029$ & $0.318 \pm \scriptstyle 0.055$ & $0.274 \pm \scriptstyle 0.016$ & $0.378 \pm \scriptstyle 0.042$ & $0.255 \pm \scriptstyle 0.034$ & $0.258 \pm \scriptstyle 0.066$ \\
\midrule
\textbf{RankDef-JS} & $0.272 \pm \scriptstyle 0.026$ & $0.268 \pm \scriptstyle 0.094$ & $0.267 \pm \scriptstyle 0.010$ & $0.236 \pm \scriptstyle 0.000$ & $0.234 \pm \scriptstyle 0.020$ & $0.276 \pm \scriptstyle 0.000$ & $0.280 \pm \scriptstyle 0.018$ & $0.301 \pm \scriptstyle 0.025$ & \textbf{$0.267 \pm \scriptstyle 0.022$} \\
RankDef-TV & $0.349 \pm \scriptstyle 0.096$ & $0.299 \pm \scriptstyle 0.099$ & $0.366 \pm \scriptstyle 0.114$ & $0.265 \pm \scriptstyle 0.044$ & $0.332 \pm \scriptstyle 0.133$ & $0.313 \pm \scriptstyle 0.077$ & $0.423 \pm \scriptstyle 0.131$ & $0.303 \pm \scriptstyle 0.098$ & $0.331 \pm \scriptstyle 0.049$ \\
\midrule
\textbf{Messy-JS} & $0.170 \pm \scriptstyle 0.006$ & $0.194 \pm \scriptstyle 0.050$ & $0.410 \pm \scriptstyle 0.045$ & $0.260 \pm \scriptstyle 0.082$ & $0.489 \pm \scriptstyle 0.116$ & $0.310 \pm \scriptstyle 0.096$ & $0.262 \pm \scriptstyle 0.086$ & $0.215 \pm \scriptstyle 0.003$ & $0.289 \pm \scriptstyle 0.008$ \\
Messy-TV & $0.192 \pm \scriptstyle 0.033$ & $0.403 \pm \scriptstyle 0.195$ & $0.253 \pm \scriptstyle 0.050$ & $0.504 \pm \scriptstyle 0.051$ & $0.255 \pm \scriptstyle 0.047$ & $0.485 \pm \scriptstyle 0.075$ & $0.242 \pm \scriptstyle 0.034$ & $0.404 \pm \scriptstyle 0.139$ & $0.342 \pm \scriptstyle 0.036$ \\
\bottomrule
\end{tabular}}
\end{table*}

\begin{table*}[t]
  \centering
  \caption{Effect of demographic diversity (RQ1), California. Average metrics across demographic groups for all partition structures (mean $\pm$ std over 8 groups). Lower is better for JSD and higher is better for Gap Closed. Bold indicates best \NAME{} variant within each partition type.}
  \label{tab:ca_main_results}
  \resizebox{0.95\textwidth}{!}{%
  \begin{tabular}{lcccccccc}
  \toprule
   & \multicolumn{2}{c}{Spatial JSD$\downarrow$} & \multicolumn{2}{c}{Travel Dist. JSD$\downarrow$} & \multicolumn{2}{c}{Trip JSD$\downarrow$} & \multicolumn{2}{c}{POI Freq. JSD$\downarrow$} \\
  \cmidrule(lr){2-3}\cmidrule(lr){4-5}\cmidrule(lr){6-7}\cmidrule(lr){8-9}
  Setup & Value & Gap Closed (\%) & Value & Gap Closed (\%) & Value & Gap Closed (\%) & Value & Gap Closed (\%) \\
  \midrule
  Baseline & $0.102 \pm \scriptstyle 0.007$ & -- & $0.565 \pm \scriptstyle 0.038$ & -- & $0.106 \pm \scriptstyle 0.011$ & -- & $0.224 \pm \scriptstyle 0.013$ & -- \\
  Strong & $0.070 \pm \scriptstyle 0.017$ & 100 & $0.253 \pm \scriptstyle 0.046$ & 100 & $0.090 \pm \scriptstyle 0.013$ & 100 & $0.180 \pm \scriptstyle 0.026$ & 100 \\
  \midrule
  \textbf{Demogroups-JS} & $0.090 \pm \scriptstyle 0.015$ & $38$ & $0.211 \pm \scriptstyle 0.114$ & $113$ & $0.087 \pm \scriptstyle 0.018$ & $119$ & $0.170 \pm \scriptstyle 0.032$ & $123$ \\
  Demogroups-TV & $0.087 \pm \scriptstyle 0.013$ & $47$ & $0.210 \pm \scriptstyle 0.080$ & $114$ & $0.100 \pm \scriptstyle 0.025$ & $37$ & $0.173 \pm \scriptstyle 0.042$ & $116$ \\
  \midrule
  \textbf{FullRank-JS} & $0.089 \pm \scriptstyle 0.015$ & $41$ & $0.177 \pm \scriptstyle 0.068$ & $124$ & $0.081 \pm \scriptstyle 0.008$ & $156$ & $0.161 \pm \scriptstyle 0.014$ & $143$ \\
  FullRank-TV & $0.087 \pm \scriptstyle 0.019$ & $47$ & $0.201 \pm \scriptstyle 0.092$ & $117$ & $0.094 \pm \scriptstyle 0.027$ & $75$ & $0.184 \pm \scriptstyle 0.021$ & $91$ \\
  \midrule
  \textbf{RankDef-JS} & $0.130 \pm \scriptstyle 0.030$ & $-88$ & $0.380 \pm \scriptstyle 0.052$ & $59$ & $0.180 \pm \scriptstyle 0.054$ & $-462$ & $0.256 \pm \scriptstyle 0.050$ & $-73$ \\
  RankDef-TV & $0.125 \pm \scriptstyle 0.032$ & $-72$ & $0.397 \pm \scriptstyle 0.103$ & $54$ & $0.192 \pm \scriptstyle 0.063$ & $-538$ & $0.244 \pm \scriptstyle 0.053$ & $-45$ \\
  \midrule
  \textbf{Messy-JS} & $0.187 \pm \scriptstyle 0.050$ & $-266$ & $0.331 \pm \scriptstyle 0.094$ & $75$ & $0.227 \pm \scriptstyle 0.071$ & $-756$ & $0.331 \pm \scriptstyle 0.084$ & $-243$ \\
  Messy-TV & $0.203 \pm \scriptstyle 0.022$ & $-316$ & $0.366 \pm \scriptstyle 0.074$ & $64$ & $0.278 \pm \scriptstyle 0.036$ & $-1075$ & $0.378 \pm \scriptstyle 0.037$ & $-350$ \\
  \bottomrule
  \end{tabular}}
  \end{table*}

\begin{table*}[t]
  \centering
  \caption{Effect of demographic diversity (RQ1), California - \textit{spatial JSD} per demographic group. Format: mean $\pm$ std across 3 seeds. The Avg column reports mean $\pm$ std across 8 demographic groups (matching Table~\ref{tab:main_results}). Bold indicates best \NAME{} variant within each partition type.}
  \label{tab:ca_spatial_jsd_pergroup}
  \resizebox{\textwidth}{!}{%
  \begin{tabular}{lccccccccc}
  \toprule
  Setup & <30, M & <30, F & 30--40, M & 30--40, F & 40--50, M & 40--50, F & >50, M & >50, F & Avg \\
  \midrule
  Baseline & $0.096 \pm \scriptstyle 0.003$ & $0.095 \pm \scriptstyle 0.011$ & $0.101 \pm \scriptstyle 0.003$ & $0.096 \pm \scriptstyle 0.007$ & $0.110 \pm \scriptstyle 0.000$ & $0.104 \pm \scriptstyle 0.001$ & $0.115 \pm \scriptstyle 0.003$ & $0.103 \pm \scriptstyle 0.008$ & $0.102 \pm \scriptstyle 0.007$ \\
  Strong & $0.056 \pm \scriptstyle 0.002$ & $0.052 \pm \scriptstyle 0.000$ & $0.067 \pm \scriptstyle 0.001$ & $0.060 \pm \scriptstyle 0.002$ & $0.105 \pm \scriptstyle 0.004$ & $0.072 \pm \scriptstyle 0.001$ & $0.079 \pm \scriptstyle 0.002$ & $0.069 \pm \scriptstyle 0.007$ & $0.070 \pm \scriptstyle 0.017$ \\
  \midrule
  \textbf{Demogroups-JS} & $0.102 \pm \scriptstyle 0.043$ & $0.077 \pm \scriptstyle 0.003$ & $0.078 \pm \scriptstyle 0.021$ & $0.073 \pm \scriptstyle 0.028$ & $0.085 \pm \scriptstyle 0.000$ & $0.111 \pm \scriptstyle 0.020$ & $0.111 \pm \scriptstyle 0.009$ & $0.084 \pm \scriptstyle 0.001$ & \textbf{$0.090 \pm \scriptstyle 0.015$} \\
  Demogroups-TV & $0.075 \pm \scriptstyle 0.002$ & $0.072 \pm \scriptstyle 0.001$ & $0.090 \pm \scriptstyle 0.005$ & $0.073 \pm \scriptstyle 0.017$ & $0.094 \pm \scriptstyle 0.007$ & $0.108 \pm \scriptstyle 0.009$ & $0.098 \pm \scriptstyle 0.019$ & $0.086 \pm \scriptstyle 0.002$ & $0.087 \pm \scriptstyle 0.013$ \\
  \midrule
  \textbf{FullRank-JS} & $0.085 \pm \scriptstyle 0.019$ & $0.085 \pm \scriptstyle 0.030$ & $0.123 \pm \scriptstyle 0.038$ & $0.080 \pm \scriptstyle 0.036$ & $0.083 \pm \scriptstyle 0.021$ & $0.084 \pm \scriptstyle 0.000$ & $0.098 \pm \scriptstyle 0.016$ & $0.073 \pm \scriptstyle 0.007$ & \textbf{$0.089 \pm \scriptstyle 0.015$} \\
  FullRank-TV & $0.060 \pm \scriptstyle 0.003$ & $0.073 \pm \scriptstyle 0.000$ & $0.085 \pm \scriptstyle 0.000$ & $0.066 \pm \scriptstyle 0.008$ & $0.109 \pm \scriptstyle 0.000$ & $0.111 \pm \scriptstyle 0.001$ & $0.095 \pm \scriptstyle 0.005$ & $0.098 \pm \scriptstyle 0.001$ & $0.087 \pm \scriptstyle 0.019$ \\
  \midrule
  \textbf{RankDef-JS} & $0.143 \pm \scriptstyle 0.047$ & $0.100 \pm \scriptstyle 0.041$ & $0.123 \pm \scriptstyle 0.020$ & $0.099 \pm \scriptstyle 0.017$ & $0.180 \pm \scriptstyle 0.014$ & $0.128 \pm \scriptstyle 0.018$ & $0.163 \pm \scriptstyle 0.034$ & $0.102 \pm \scriptstyle 0.017$ & $0.130 \pm \scriptstyle 0.030$ \\
  RankDef-TV & $0.122 \pm \scriptstyle 0.052$ & $0.096 \pm \scriptstyle 0.034$ & $0.103 \pm \scriptstyle 0.005$ & $0.082 \pm \scriptstyle 0.023$ & $0.146 \pm \scriptstyle 0.088$ & $0.122 \pm \scriptstyle 0.054$ & $0.177 \pm \scriptstyle 0.051$ & $0.149 \pm \scriptstyle 0.048$ & \textbf{$0.125 \pm \scriptstyle 0.032$} \\
  \midrule
  \textbf{Messy-JS} & $0.196 \pm \scriptstyle 0.067$ & $0.121 \pm \scriptstyle 0.014$ & $0.137 \pm \scriptstyle 0.052$ & $0.176 \pm \scriptstyle 0.069$ & $0.250 \pm \scriptstyle 0.058$ & $0.235 \pm \scriptstyle 0.035$ & $0.236 \pm \scriptstyle 0.107$ & $0.149 \pm \scriptstyle 0.065$ & \textbf{$0.187 \pm \scriptstyle 0.050$} \\
  Messy-TV & $0.212 \pm \scriptstyle 0.041$ & $0.186 \pm \scriptstyle 0.076$ & $0.233 \pm \scriptstyle 0.030$ & $0.185 \pm \scriptstyle 0.043$ & $0.218 \pm \scriptstyle 0.019$ & $0.184 \pm \scriptstyle 0.039$ & $0.229 \pm \scriptstyle 0.033$ & $0.182 \pm \scriptstyle 0.037$ & $0.203 \pm \scriptstyle 0.022$ \\
  \bottomrule
  \end{tabular}}
  \end{table*}

\begin{table*}[t]
  \centering
  \caption{Effect of demographic diversity (RQ1), California - \textit{travel distance JSD} per demographic group. Format: mean $\pm$ std across 3 seeds. The Avg column reports mean $\pm$ std across 8 demographic groups (matching Table~\ref{tab:main_results}). Bold indicates best \NAME{} variant within each partition type.}
  \label{tab:ca_travel_distance_jsd_pergroup}
  \resizebox{\textwidth}{!}{%
  \begin{tabular}{lccccccccc}
  \toprule
  Setup & <30, M & <30, F & 30--40, M & 30--40, F & 40--50, M & 40--50, F & >50, M & >50, F & Avg \\
  \midrule
  Baseline & $0.502 \pm \scriptstyle 0.008$ & $0.583 \pm \scriptstyle 0.015$ & $0.549 \pm \scriptstyle 0.003$ & $0.568 \pm \scriptstyle 0.010$ & $0.595 \pm \scriptstyle 0.008$ & $0.523 \pm \scriptstyle 0.012$ & $0.614 \pm \scriptstyle 0.012$ & $0.583 \pm \scriptstyle 0.034$ & $0.565 \pm \scriptstyle 0.038$ \\
  Strong & $0.207 \pm \scriptstyle 0.020$ & $0.232 \pm \scriptstyle 0.015$ & $0.227 \pm \scriptstyle 0.006$ & $0.285 \pm \scriptstyle 0.027$ & $0.342 \pm \scriptstyle 0.047$ & $0.205 \pm \scriptstyle 0.032$ & $0.279 \pm \scriptstyle 0.008$ & $0.248 \pm \scriptstyle 0.027$ & $0.253 \pm \scriptstyle 0.046$ \\
  \midrule
  \textbf{Demogroups-JS} & $0.285 \pm \scriptstyle 0.011$ & $0.203 \pm \scriptstyle 0.009$ & $0.145 \pm \scriptstyle 0.010$ & $0.051 \pm \scriptstyle 0.001$ & $0.359 \pm \scriptstyle 0.003$ & $0.156 \pm \scriptstyle 0.015$ & $0.366 \pm \scriptstyle 0.004$ & $0.126 \pm \scriptstyle 0.005$ & \textbf{$0.211 \pm \scriptstyle 0.114$} \\
  Demogroups-TV & $0.238 \pm \scriptstyle 0.128$ & $0.135 \pm \scriptstyle 0.058$ & $0.185 \pm \scriptstyle 0.101$ & $0.201 \pm \scriptstyle 0.059$ & $0.143 \pm \scriptstyle 0.021$ & $0.182 \pm \scriptstyle 0.013$ & $0.389 \pm \scriptstyle 0.199$ & $0.208 \pm \scriptstyle 0.074$ & $0.210 \pm \scriptstyle 0.080$ \\
  \midrule
  \textbf{FullRank-JS} & $0.206 \pm \scriptstyle 0.012$ & $0.107 \pm \scriptstyle 0.030$ & $0.178 \pm \scriptstyle 0.066$ & $0.086 \pm \scriptstyle 0.055$ & $0.116 \pm \scriptstyle 0.017$ & $0.261 \pm \scriptstyle 0.128$ & $0.261 \pm \scriptstyle 0.058$ & $0.201 \pm \scriptstyle 0.012$ & \textbf{$0.177 \pm \scriptstyle 0.068$} \\
  FullRank-TV & $0.119 \pm \scriptstyle 0.014$ & $0.185 \pm \scriptstyle 0.008$ & $0.191 \pm \scriptstyle 0.015$ & $0.106 \pm \scriptstyle 0.001$ & $0.396 \pm \scriptstyle 0.030$ & $0.157 \pm \scriptstyle 0.020$ & $0.254 \pm \scriptstyle 0.003$ & $0.197 \pm \scriptstyle 0.023$ & $0.201 \pm \scriptstyle 0.092$ \\
  \midrule
  \textbf{RankDef-JS} & $0.342 \pm \scriptstyle 0.138$ & $0.379 \pm \scriptstyle 0.158$ & $0.431 \pm \scriptstyle 0.151$ & $0.398 \pm \scriptstyle 0.234$ & $0.284 \pm \scriptstyle 0.188$ & $0.453 \pm \scriptstyle 0.141$ & $0.368 \pm \scriptstyle 0.137$ & $0.385 \pm \scriptstyle 0.154$ & $0.380 \pm \scriptstyle 0.052$ \\
  RankDef-TV & $0.210 \pm \scriptstyle 0.097$ & $0.339 \pm \scriptstyle 0.089$ & $0.316 \pm \scriptstyle 0.068$ & $0.471 \pm \scriptstyle 0.076$ & $0.458 \pm \scriptstyle 0.066$ & $0.462 \pm \scriptstyle 0.210$ & $0.397 \pm \scriptstyle 0.187$ & $0.525 \pm \scriptstyle 0.149$ & \textbf{$0.397 \pm \scriptstyle 0.103$} \\
  \midrule
  \textbf{Messy-JS} & $0.408 \pm \scriptstyle 0.175$ & $0.211 \pm \scriptstyle 0.082$ & $0.174 \pm \scriptstyle 0.044$ & $0.384 \pm \scriptstyle 0.298$ & $0.392 \pm \scriptstyle 0.041$ & $0.406 \pm \scriptstyle 0.144$ & $0.290 \pm \scriptstyle 0.100$ & $0.386 \pm \scriptstyle 0.282$ & \textbf{$0.331 \pm \scriptstyle 0.094$} \\
  Messy-TV & $0.294 \pm \scriptstyle 0.111$ & $0.406 \pm \scriptstyle 0.158$ & $0.234 \pm \scriptstyle 0.173$ & $0.319 \pm \scriptstyle 0.126$ & $0.447 \pm \scriptstyle 0.115$ & $0.396 \pm \scriptstyle 0.152$ & $0.425 \pm \scriptstyle 0.053$ & $0.405 \pm \scriptstyle 0.211$ & $0.366 \pm \scriptstyle 0.074$ \\
  \bottomrule
  \end{tabular}}
  \end{table*}

\begin{table*}[t]
  \centering
  \caption{Effect of demographic diversity (RQ1), California - \textit{trip JSD} per demographic group. Format: mean $\pm$ std across 3 seeds. The Avg column reports mean $\pm$ std across 8 demographic groups (matching Table~\ref{tab:main_results}). Bold indicates best \NAME{} variant within each partition type.}
  \label{tab:ca_trip_jsd_pergroup}
  \resizebox{\textwidth}{!}{%
  \begin{tabular}{lccccccccc}
  \toprule
  Setup & <30, M & <30, F & 30--40, M & 30--40, F & 40--50, M & 40--50, F & >50, M & >50, F & Avg \\
  \midrule
  Baseline & $0.099 \pm \scriptstyle 0.001$ & $0.094 \pm \scriptstyle 0.001$ & $0.102 \pm \scriptstyle 0.001$ & $0.095 \pm \scriptstyle 0.001$ & $0.103 \pm \scriptstyle 0.000$ & $0.113 \pm \scriptstyle 0.001$ & $0.123 \pm \scriptstyle 0.001$ & $0.119 \pm \scriptstyle 0.001$ & $0.106 \pm \scriptstyle 0.011$ \\
  Strong & $0.088 \pm \scriptstyle 0.002$ & $0.076 \pm \scriptstyle 0.001$ & $0.079 \pm \scriptstyle 0.001$ & $0.074 \pm \scriptstyle 0.001$ & $0.111 \pm \scriptstyle 0.003$ & $0.093 \pm \scriptstyle 0.003$ & $0.100 \pm \scriptstyle 0.002$ & $0.098 \pm \scriptstyle 0.001$ & $0.090 \pm \scriptstyle 0.013$ \\
  \midrule
  \textbf{Demogroups-JS} & $0.112 \pm \scriptstyle 0.002$ & $0.076 \pm \scriptstyle 0.001$ & $0.071 \pm \scriptstyle 0.000$ & $0.063 \pm \scriptstyle 0.001$ & $0.089 \pm \scriptstyle 0.002$ & $0.097 \pm \scriptstyle 0.002$ & $0.111 \pm \scriptstyle 0.000$ & $0.081 \pm \scriptstyle 0.009$ & \textbf{$0.087 \pm \scriptstyle 0.018$} \\
  Demogroups-TV & $0.079 \pm \scriptstyle 0.010$ & $0.060 \pm \scriptstyle 0.010$ & $0.132 \pm \scriptstyle 0.023$ & $0.109 \pm \scriptstyle 0.022$ & $0.085 \pm \scriptstyle 0.024$ & $0.127 \pm \scriptstyle 0.049$ & $0.092 \pm \scriptstyle 0.013$ & $0.114 \pm \scriptstyle 0.006$ & $0.100 \pm \scriptstyle 0.025$ \\
  \midrule
  \textbf{FullRank-JS} & $0.078 \pm \scriptstyle 0.011$ & $0.089 \pm \scriptstyle 0.017$ & $0.078 \pm \scriptstyle 0.003$ & $0.077 \pm \scriptstyle 0.025$ & $0.068 \pm \scriptstyle 0.009$ & $0.092 \pm \scriptstyle 0.011$ & $0.078 \pm \scriptstyle 0.006$ & $0.091 \pm \scriptstyle 0.005$ & \textbf{$0.081 \pm \scriptstyle 0.008$} \\
  FullRank-TV & $0.088 \pm \scriptstyle 0.015$ & $0.055 \pm \scriptstyle 0.001$ & $0.109 \pm \scriptstyle 0.034$ & $0.058 \pm \scriptstyle 0.002$ & $0.128 \pm \scriptstyle 0.051$ & $0.094 \pm \scriptstyle 0.001$ & $0.127 \pm \scriptstyle 0.050$ & $0.088 \pm \scriptstyle 0.006$ & $0.094 \pm \scriptstyle 0.027$ \\
  \midrule
  \textbf{RankDef-JS} & $0.252 \pm \scriptstyle 0.023$ & $0.125 \pm \scriptstyle 0.068$ & $0.199 \pm \scriptstyle 0.095$ & $0.116 \pm \scriptstyle 0.065$ & $0.243 \pm \scriptstyle 0.063$ & $0.151 \pm \scriptstyle 0.060$ & $0.217 \pm \scriptstyle 0.039$ & $0.140 \pm \scriptstyle 0.005$ & \textbf{$0.180 \pm \scriptstyle 0.054$} \\
  RankDef-TV & $0.183 \pm \scriptstyle 0.081$ & $0.099 \pm \scriptstyle 0.072$ & $0.161 \pm \scriptstyle 0.050$ & $0.131 \pm \scriptstyle 0.043$ & $0.224 \pm \scriptstyle 0.156$ & $0.197 \pm \scriptstyle 0.087$ & $0.276 \pm \scriptstyle 0.046$ & $0.268 \pm \scriptstyle 0.080$ & $0.192 \pm \scriptstyle 0.063$ \\
  \midrule
  \textbf{Messy-JS} & $0.292 \pm \scriptstyle 0.085$ & $0.123 \pm \scriptstyle 0.061$ & $0.152 \pm \scriptstyle 0.055$ & $0.209 \pm \scriptstyle 0.130$ & $0.311 \pm \scriptstyle 0.078$ & $0.305 \pm \scriptstyle 0.118$ & $0.235 \pm \scriptstyle 0.085$ & $0.189 \pm \scriptstyle 0.124$ & \textbf{$0.227 \pm \scriptstyle 0.071$} \\
  Messy-TV & $0.314 \pm \scriptstyle 0.064$ & $0.238 \pm \scriptstyle 0.114$ & $0.294 \pm \scriptstyle 0.020$ & $0.247 \pm \scriptstyle 0.076$ & $0.317 \pm \scriptstyle 0.018$ & $0.270 \pm \scriptstyle 0.109$ & $0.312 \pm \scriptstyle 0.050$ & $0.234 \pm \scriptstyle 0.112$ & $0.278 \pm \scriptstyle 0.036$ \\
  \bottomrule
  \end{tabular}}
  \end{table*}

\begin{table*}[t]
\centering
\caption{Effect of demographic diversity (RQ1), California - \textit{POI frequency JSD} per demographic group. Format: mean $\pm$ std across 3 seeds. The Avg column reports mean $\pm$ std across 8 demographic groups (matching Table~\ref{tab:main_results}). Bold indicates best \NAME{} variant within each partition type.}
\label{tab:ca_poi_freq_pergroup}
\resizebox{\textwidth}{!}{%
\begin{tabular}{lccccccccc}
\toprule
Setup & <30, M & <30, F & 30--40, M & 30--40, F & 40--50, M & 40--50, F & >50, M & >50, F & Avg \\
\midrule
Baseline & $0.216 \pm \scriptstyle 0.003$ & $0.198 \pm \scriptstyle 0.008$ & $0.222 \pm \scriptstyle 0.005$ & $0.218 \pm \scriptstyle 0.013$ & $0.232 \pm \scriptstyle 0.012$ & $0.233 \pm \scriptstyle 0.011$ & $0.236 \pm \scriptstyle 0.010$ & $0.234 \pm \scriptstyle 0.012$ & $0.224 \pm \scriptstyle 0.013$ \\
Strong & $0.160 \pm \scriptstyle 0.001$ & $0.135 \pm \scriptstyle 0.000$ & $0.172 \pm \scriptstyle 0.000$ & $0.168 \pm \scriptstyle 0.002$ & $0.216 \pm \scriptstyle 0.002$ & $0.194 \pm \scriptstyle 0.003$ & $0.202 \pm \scriptstyle 0.001$ & $0.190 \pm \scriptstyle 0.006$ & $0.180 \pm \scriptstyle 0.026$ \\
\midrule
\textbf{Demogroups-JS} & $0.211 \pm \scriptstyle 0.004$ & $0.130 \pm \scriptstyle 0.002$ & $0.159 \pm \scriptstyle 0.001$ & $0.134 \pm \scriptstyle 0.000$ & $0.177 \pm \scriptstyle 0.001$ & $0.166 \pm \scriptstyle 0.001$ & $0.217 \pm \scriptstyle 0.000$ & $0.164 \pm \scriptstyle 0.000$ & \textbf{$0.170 \pm \scriptstyle 0.032$} \\
Demogroups-TV & $0.132 \pm \scriptstyle 0.001$ & $0.106 \pm \scriptstyle 0.016$ & $0.223 \pm \scriptstyle 0.011$ & $0.185 \pm \scriptstyle 0.014$ & $0.154 \pm \scriptstyle 0.006$ & $0.225 \pm \scriptstyle 0.008$ & $0.175 \pm \scriptstyle 0.006$ & $0.187 \pm \scriptstyle 0.016$ & $0.173 \pm \scriptstyle 0.042$ \\
\midrule
\textbf{FullRank-JS} & $0.143 \pm \scriptstyle 0.015$ & $0.161 \pm \scriptstyle 0.040$ & $0.157 \pm \scriptstyle 0.013$ & $0.155 \pm \scriptstyle 0.021$ & $0.148 \pm \scriptstyle 0.007$ & $0.184 \pm \scriptstyle 0.021$ & $0.177 \pm \scriptstyle 0.007$ & $0.164 \pm \scriptstyle 0.014$ & \textbf{$0.161 \pm \scriptstyle 0.014$} \\
FullRank-TV & $0.160 \pm \scriptstyle 0.010$ & $0.157 \pm \scriptstyle 0.051$ & $0.184 \pm \scriptstyle 0.017$ & $0.172 \pm \scriptstyle 0.026$ & $0.209 \pm \scriptstyle 0.027$ & $0.210 \pm \scriptstyle 0.009$ & $0.199 \pm \scriptstyle 0.026$ & $0.184 \pm \scriptstyle 0.017$ & $0.184 \pm \scriptstyle 0.021$ \\
\midrule
\textbf{RankDef-JS} & $0.310 \pm \scriptstyle 0.054$ & $0.192 \pm \scriptstyle 0.084$ & $0.265 \pm \scriptstyle 0.063$ & $0.201 \pm \scriptstyle 0.061$ & $0.308 \pm \scriptstyle 0.067$ & $0.239 \pm \scriptstyle 0.020$ & $0.313 \pm \scriptstyle 0.063$ & $0.219 \pm \scriptstyle 0.043$ & $0.256 \pm \scriptstyle 0.050$ \\
RankDef-TV & $0.236 \pm \scriptstyle 0.085$ & $0.164 \pm \scriptstyle 0.073$ & $0.211 \pm \scriptstyle 0.013$ & $0.200 \pm \scriptstyle 0.056$ & $0.271 \pm \scriptstyle 0.127$ & $0.248 \pm \scriptstyle 0.085$ & $0.329 \pm \scriptstyle 0.068$ & $0.293 \pm \scriptstyle 0.085$ & \textbf{$0.244 \pm \scriptstyle 0.053$} \\
\midrule
\textbf{Messy-JS} & $0.376 \pm \scriptstyle 0.098$ & $0.175 \pm \scriptstyle 0.061$ & $0.261 \pm \scriptstyle 0.041$ & $0.340 \pm \scriptstyle 0.142$ & $0.412 \pm \scriptstyle 0.067$ & $0.400 \pm \scriptstyle 0.083$ & $0.397 \pm \scriptstyle 0.162$ & $0.284 \pm \scriptstyle 0.092$ & \textbf{$0.331 \pm \scriptstyle 0.084$} \\
Messy-TV & $0.385 \pm \scriptstyle 0.069$ & $0.326 \pm \scriptstyle 0.061$ & $0.431 \pm \scriptstyle 0.057$ & $0.363 \pm \scriptstyle 0.042$ & $0.402 \pm \scriptstyle 0.034$ & $0.361 \pm \scriptstyle 0.045$ & $0.419 \pm \scriptstyle 0.036$ & $0.341 \pm \scriptstyle 0.068$ & $0.378 \pm \scriptstyle 0.037$ \\
\bottomrule
\end{tabular}}
\end{table*}

In the following tables, we provide detailed per-group performance comparisons across states, trajectory statistics, partition structures, and divergence functions used in the aggregate loss:
Virginia, average JSDs over groups (Table~\ref{tab:main_results}) and JSDs per group and trajectory statistic (Tables~\ref{tab:spatial_jsd_pergroup}-\ref{tab:poi_freq_pergroup}); California, average JSDs over groups (Table~\ref{tab:ca_main_results}) and JSDs per group and trajectory statistic (Tables~\ref{tab:ca_spatial_jsd_pergroup}-\ref{tab:ca_poi_freq_pergroup}).
As described in Section~\ref{sec:demo-diversity-results}, \NAME{} performs very well when the partition is well-conditioned (DemoGroups and FullRank) and performance degrades as the partition becomes more ill-conditioned, as predicted by Condition~\ref{ass:full-rank} and Lemma~\ref{lem:moment-identifiability}.

Additionally, we observe a clear interaction between partition conditioning and the choice of distributional divergence measure used in the aggregate loss. 
Jensen-Shannon (JS) divergence remains comparatively stable as $P$ becomes ill-conditioned: on the Rank-Deficient partition, RankDef-JS retains reasonable performance.
In contrast, Total Variation (TV) degrades sharply under rank deficiency: RankDef-TV shows negative average gap closed on spatial ($-80\%$), trip ($-71\%$), and POI frequency ($-49\%$), while remaining moderately effective on travel distance ($75\%$). This pattern is consistent with noise amplification under ill-conditioned $P$: small errors in estimated regional aggregates $\hat{\nu}_\theta(g)$ are amplified through $(P^\top P)^{-1}$ at a rate proportional to $1/\sigma_{\min}(P)$ (Lemma~\ref{lem:stability}).
Under such noise, JS can yield more stable optimization, whereas TV, as an $\ell_1$-type discrepancy, can overweight sparse bins/tail events and become less stable.
Practically, the divergence choice matters primarily when $P$ is ill-conditioned, where optimization becomes more sensitive to estimation noise.

\subsection{RQ2: Effect of Feature Choice}
\begin{table*}[t]
  \centering
  \caption{Effect of feature choice (RQ2), Virginia. JSDs across trajectory statistics per demographic group. Format: mean $\pm$ std across 3 seeds. The Avg column reports mean $\pm$ std across 8 demographic groups.}
  \label{tab:js_cate_trans_pergroup}
  \resizebox{\textwidth}{!}{%
  \begin{tabular}{lccccccccc}
  \toprule
  Setup & <30, M & <30, F & 30--40, M & 30--40, F & 40--50, M & 40--50, F & >50, M & >50, F & Avg \\
  \midrule
  \multicolumn{10}{l}{\textit{Spatial JSD}} \\
  Baseline & $0.131 \pm \scriptstyle 0.002$ & $0.089 \pm \scriptstyle 0.002$ & $0.119 \pm \scriptstyle 0.009$ & $0.077 \pm \scriptstyle 0.008$ & $0.105 \pm \scriptstyle 0.015$ & $0.099 \pm \scriptstyle 0.011$ & $0.093 \pm \scriptstyle 0.008$ & $0.123 \pm \scriptstyle 0.001$ & $0.105$ \\
  Strong & $0.075 \pm \scriptstyle 0.006$ & $0.051 \pm \scriptstyle 0.007$ & $0.045 \pm \scriptstyle 0.001$ & $0.048 \pm \scriptstyle 0.000$ & $0.058 \pm \scriptstyle 0.004$ & $0.068 \pm \scriptstyle 0.002$ & $0.051 \pm \scriptstyle 0.000$ & $0.079 \pm \scriptstyle 0.006$ & $0.059$ \\
  \textbf{POI-Histogram-JS} & $0.072 \pm \scriptstyle 0.018$ & $0.060 \pm \scriptstyle 0.008$ & $0.092 \pm \scriptstyle 0.011$ & $0.057 \pm \scriptstyle 0.029$ & $0.086 \pm \scriptstyle 0.010$ & $0.093 \pm \scriptstyle 0.008$ & $0.091 \pm \scriptstyle 0.008$ & $0.089 \pm \scriptstyle 0.037$ & \textbf{$0.080$} \\
  POI-Histogram-TV & $0.075 \pm \scriptstyle 0.025$ & $0.072 \pm \scriptstyle 0.030$ & $0.110 \pm \scriptstyle 0.045$ & $0.066 \pm \scriptstyle 0.022$ & $0.101 \pm \scriptstyle 0.051$ & $0.095 \pm \scriptstyle 0.018$ & $0.106 \pm \scriptstyle 0.049$ & $0.071 \pm \scriptstyle 0.012$ & 0.087 \\
  \textbf{Cate-Trans-JS} & $0.158 \pm \scriptstyle 0.094$ & $0.094 \pm \scriptstyle 0.038$ & $0.138 \pm \scriptstyle 0.033$ & $0.088 \pm \scriptstyle 0.043$ & $0.078 \pm \scriptstyle 0.008$ & $0.106 \pm \scriptstyle 0.014$ & $0.131 \pm \scriptstyle 0.100$ & $0.077 \pm \scriptstyle 0.006$ & 0.109 \\
  Cate-Trans-TV & $0.126 \pm \scriptstyle 0.078$ & $0.165 \pm \scriptstyle 0.201$ & $0.196 \pm \scriptstyle 0.127$ & $0.141 \pm \scriptstyle 0.155$ & $0.066 \pm \scriptstyle 0.006$ & $0.129 \pm \scriptstyle 0.031$ & $0.151 \pm \scriptstyle 0.121$ & $0.093 \pm \scriptstyle 0.014$ & 0.133 \\
  \textbf{Cate-JS} & $0.097 \pm \scriptstyle 0.005$ & $0.085 \pm \scriptstyle 0.011$ & $0.111 \pm \scriptstyle 0.019$ & $0.048 \pm \scriptstyle 0.002$ & $0.109 \pm \scriptstyle 0.007$ & $0.141 \pm \scriptstyle 0.028$ & $0.101 \pm \scriptstyle 0.009$ & $0.129 \pm \scriptstyle 0.058$ & 0.103 \\
  Cate-TV & $0.238 \pm \scriptstyle 0.043$ & $0.170 \pm \scriptstyle 0.111$ & $0.194 \pm \scriptstyle 0.032$ & $0.128 \pm \scriptstyle 0.053$ & $0.119 \pm \scriptstyle 0.033$ & $0.119 \pm \scriptstyle 0.004$ & $0.105 \pm \scriptstyle 0.006$ & $0.240 \pm \scriptstyle 0.163$ & 0.164 \\
  \midrule
  \multicolumn{10}{l}{\textit{Travel Distance JSD}} \\
  Baseline & $0.398 \pm \scriptstyle 0.038$ & $0.353 \pm \scriptstyle 0.075$ & $0.455 \pm \scriptstyle 0.053$ & $0.324 \pm \scriptstyle 0.074$ & $0.366 \pm \scriptstyle 0.070$ & $0.350 \pm \scriptstyle 0.081$ & $0.387 \pm \scriptstyle 0.059$ & $0.425 \pm \scriptstyle 0.045$ & $0.382$  \\
    Strong & $0.125 \pm \scriptstyle 0.002$ & $0.178 \pm \scriptstyle 0.003$ & $0.140 \pm \scriptstyle 0.003$ & $0.134 \pm \scriptstyle 0.002$ & $0.147 \pm \scriptstyle 0.003$ & $0.164 \pm \scriptstyle 0.008$ & $0.121 \pm \scriptstyle 0.004$ & $0.253 \pm \scriptstyle 0.020$ & $0.158$ \\
  \textbf{POI-Histogram-JS} & $0.127 \pm \scriptstyle 0.078$ & $0.121 \pm \scriptstyle 0.072$ & $0.146 \pm \scriptstyle 0.066$ & $0.151 \pm \scriptstyle 0.013$ & $0.160 \pm \scriptstyle 0.096$ & $0.084 \pm \scriptstyle 0.029$ & $0.106 \pm \scriptstyle 0.070$ & $0.097 \pm \scriptstyle 0.005$ & \textbf{$0.124$} \\
  POI-Histogram-TV & $0.079 \pm \scriptstyle 0.043$ & $0.099 \pm \scriptstyle 0.040$ & $0.081 \pm \scriptstyle 0.054$ & $0.282 \pm \scriptstyle 0.045$ & $0.149 \pm \scriptstyle 0.015$ & $0.090 \pm \scriptstyle 0.046$ & $0.155 \pm \scriptstyle 0.087$ & $0.116 \pm \scriptstyle 0.062$ & $0.131$ \\
  \textbf{Cate-Trans-JS} & $0.261 \pm \scriptstyle 0.217$ & $0.189 \pm \scriptstyle 0.140$ & $0.274 \pm \scriptstyle 0.215$ & $0.183 \pm \scriptstyle 0.066$ & $0.215 \pm \scriptstyle 0.202$ & $0.174 \pm \scriptstyle 0.078$ & $0.276 \pm \scriptstyle 0.085$ & $0.241 \pm \scriptstyle 0.140$ & 0.227 \\
  Cate-Trans-TV & $0.166 \pm \scriptstyle 0.187$ & $0.123 \pm \scriptstyle 0.125$ & $0.357 \pm \scriptstyle 0.157$ & $0.274 \pm \scriptstyle 0.123$ & $0.206 \pm \scriptstyle 0.163$ & $0.182 \pm \scriptstyle 0.140$ & $0.225 \pm \scriptstyle 0.172$ & $0.117 \pm \scriptstyle 0.099$ & 0.206 \\
  \textbf{Cate-JS} & $0.187 \pm \scriptstyle 0.062$ & $0.218 \pm \scriptstyle 0.001$ & $0.476 \pm \scriptstyle 0.002$ & $0.178 \pm \scriptstyle 0.015$ & $0.401 \pm \scriptstyle 0.059$ & $0.423 \pm \scriptstyle 0.046$ & $0.397 \pm \scriptstyle 0.074$ & $0.472 \pm \scriptstyle 0.085$ & 0.344 \\
  Cate-TV & $0.305 \pm \scriptstyle 0.193$ & $0.338 \pm \scriptstyle 0.197$ & $0.572 \pm \scriptstyle 0.083$ & $0.265 \pm \scriptstyle 0.033$ & $0.356 \pm \scriptstyle 0.130$ & $0.312 \pm \scriptstyle 0.140$ & $0.291 \pm \scriptstyle 0.187$ & $0.317 \pm \scriptstyle 0.187$ & 0.345 \\
  \midrule
  \multicolumn{10}{l}{\textit{Trip JSD}} \\
  Baseline & $0.192 \pm \scriptstyle 0.006$ & $0.122 \pm \scriptstyle 0.005$ & $0.177 \pm \scriptstyle 0.010$ & $0.119 \pm \scriptstyle 0.005$ & $0.165 \pm \scriptstyle 0.009$ & $0.170 \pm \scriptstyle 0.006$ & $0.128 \pm \scriptstyle 0.018$ & $0.186 \pm \scriptstyle 0.014$ & $0.157$ \\
  Strong & $0.146 \pm \scriptstyle 0.000$ & $0.108 \pm \scriptstyle 0.016$ & $0.113 \pm \scriptstyle 0.002$ & $0.092 \pm \scriptstyle 0.004$ & $0.127 \pm \scriptstyle 0.007$ & $0.154 \pm \scriptstyle 0.000$ & $0.098 \pm \scriptstyle 0.001$ & $0.140 \pm \scriptstyle 0.020$ & $0.122$ \\
  \textbf{POI-Histogram-JS} & $0.112 \pm \scriptstyle 0.022$ & $0.068 \pm \scriptstyle 0.006$ & $0.101 \pm \scriptstyle 0.007$ & $0.095 \pm \scriptstyle 0.037$ & $0.121 \pm \scriptstyle 0.024$ & $0.127 \pm \scriptstyle 0.012$ & $0.115 \pm \scriptstyle 0.014$ & $0.100 \pm \scriptstyle 0.047$ & 0.105\\
  POI-Histogram-TV & $0.133 \pm \scriptstyle 0.026$ & $0.092 \pm \scriptstyle 0.024$ & $0.138 \pm \scriptstyle 0.041$ & $0.105 \pm \scriptstyle 0.026$ & $0.142 \pm \scriptstyle 0.072$ & $0.131 \pm \scriptstyle 0.027$ & $0.131 \pm \scriptstyle 0.050$ & $0.127 \pm \scriptstyle 0.043$ & 0.125 \\
  \textbf{Cate-Trans-JS} & $0.188 \pm \scriptstyle 0.069$ & $0.102 \pm \scriptstyle 0.017$ & $0.123 \pm \scriptstyle 0.022$ & $0.103 \pm \scriptstyle 0.021$ & $0.103 \pm \scriptstyle 0.016$ & $0.132 \pm \scriptstyle 0.023$ & $0.150 \pm \scriptstyle 0.097$ & $0.134 \pm \scriptstyle 0.040$ & 0.129 \\
  Cate-Trans-TV & $0.148 \pm \scriptstyle 0.073$ & $0.200 \pm \scriptstyle 0.209$ & $0.232 \pm \scriptstyle 0.174$ & $0.184 \pm \scriptstyle 0.181$ & $0.108 \pm \scriptstyle 0.006$ & $0.164 \pm \scriptstyle 0.028$ & $0.198 \pm \scriptstyle 0.170$ & $0.129 \pm \scriptstyle 0.027$ & 0.170 \\
  \textbf{Cate-JS} & $0.143 \pm \scriptstyle 0.015$ & $0.114 \pm \scriptstyle 0.003$ & $0.149 \pm \scriptstyle 0.020$ & $0.082 \pm \scriptstyle 0.004$ & $0.183 \pm \scriptstyle 0.006$ & $0.197 \pm \scriptstyle 0.002$ & $0.143 \pm \scriptstyle 0.002$ & $0.179 \pm \scriptstyle 0.033$ & 0.149 \\
  Cate-TV & $0.266 \pm \scriptstyle 0.052$ & $0.218 \pm \scriptstyle 0.117$ & $0.229 \pm \scriptstyle 0.041$ & $0.155 \pm \scriptstyle 0.057$ & $0.185 \pm \scriptstyle 0.058$ & $0.166 \pm \scriptstyle 0.031$ & $0.134 \pm \scriptstyle 0.024$ & $0.257 \pm \scriptstyle 0.149$ & 0.201 \\
  \midrule
  \multicolumn{10}{l}{\textit{POI Frequency JSD}} \\
    Baseline & $0.398 \pm \scriptstyle 0.038$ & $0.353 \pm \scriptstyle 0.075$ & $0.455 \pm \scriptstyle 0.053$ & $0.324 \pm \scriptstyle 0.074$ & $0.366 \pm \scriptstyle 0.070$ & $0.350 \pm \scriptstyle 0.081$ & $0.387 \pm \scriptstyle 0.059$ & $0.425 \pm \scriptstyle 0.045$ & $0.382$ \\
    Strong & $0.252 \pm \scriptstyle 0.005$ & $0.243 \pm \scriptstyle 0.001$ & $0.189 \pm \scriptstyle 0.001$ & $0.242 \pm \scriptstyle 0.001$ & $0.262 \pm \scriptstyle 0.006$ & $0.279 \pm \scriptstyle 0.001$ & $0.232 \pm \scriptstyle 0.003$ & $0.291 \pm \scriptstyle 0.001$ & $0.249$ \\
  \textbf{POI-Histogram-JS} & $0.186 \pm \scriptstyle 0.018$ & $0.168 \pm \scriptstyle 0.020$ & $0.220 \pm \scriptstyle 0.016$ & $0.202 \pm \scriptstyle 0.024$ & $0.227 \pm \scriptstyle 0.037$ & $0.224 \pm \scriptstyle 0.000$ & $0.229 \pm \scriptstyle 0.031$ & $0.231 \pm \scriptstyle 0.029$ & 0.211\\
  POI-Histogram-TV & $0.211 \pm \scriptstyle 0.049$ & $0.204 \pm \scriptstyle 0.077$ & $0.223 \pm \scriptstyle 0.028$ & $0.261 \pm \scriptstyle 0.057$ & $0.262 \pm \scriptstyle 0.059$ & $0.245 \pm \scriptstyle 0.028$ & $0.273 \pm \scriptstyle 0.080$ & $0.222 \pm \scriptstyle 0.012$ & 0.238 \\
  \textbf{Cate-Trans-JS} & $0.290 \pm \scriptstyle 0.120$ & $0.247 \pm \scriptstyle 0.079$ & $0.275 \pm \scriptstyle 0.061$ & $0.282 \pm \scriptstyle 0.084$ & $0.221 \pm \scriptstyle 0.025$ & $0.233 \pm \scriptstyle 0.018$ & $0.284 \pm \scriptstyle 0.141$ & $0.221 \pm \scriptstyle 0.014$ & 0.257 \\
  Cate-Trans-TV & $0.282 \pm \scriptstyle 0.152$ & $0.344 \pm \scriptstyle 0.260$ & $0.365 \pm \scriptstyle 0.173$ & $0.385 \pm \scriptstyle 0.253$ & $0.232 \pm \scriptstyle 0.035$ & $0.272 \pm \scriptstyle 0.044$ & $0.346 \pm \scriptstyle 0.174$ & $0.217 \pm \scriptstyle 0.011$ & 0.305 \\
  \textbf{Cate-JS} & $0.248 \pm \scriptstyle 0.054$ & $0.223 \pm \scriptstyle 0.021$ & $0.278 \pm \scriptstyle 0.016$ & $0.225 \pm \scriptstyle 0.019$ & $0.308 \pm \scriptstyle 0.034$ & $0.333 \pm \scriptstyle 0.008$ & $0.290 \pm \scriptstyle 0.020$ & $0.320 \pm \scriptstyle 0.061$ & 0.278 \\
  Cate-TV & $0.404 \pm \scriptstyle 0.086$ & $0.331 \pm \scriptstyle 0.120$ & $0.373 \pm \scriptstyle 0.014$ & $0.357 \pm \scriptstyle 0.102$ & $0.313 \pm \scriptstyle 0.086$ & $0.304 \pm \scriptstyle 0.043$ & $0.279 \pm \scriptstyle 0.038$ & $0.343 \pm \scriptstyle 0.069$ & 0.338 \\
  \bottomrule
  \end{tabular}}
  \end{table*}

\begin{table*}[t]
  \centering
  \caption{Effect of feature choice (RQ2), California. JSDs across trajectory statistics per demographic group. Format: mean $\pm$ std across 3 seeds. The Avg column reports mean $\pm$ std across 8 demographic groups.}
  \label{tab:ca_cate_ablation_pergroup}
  \resizebox{\textwidth}{!}{%
  \begin{tabular}{lccccccccc}
  \toprule
  Setup & <30, M & <30, F & 30--40, M & 30--40, F & 40--50, M & 40--50, F & >50, M & >50, F & Avg \\
  \midrule
  \multicolumn{10}{l}{\textit{Spatial JSD}} \\
  Baseline & $0.096 \pm \scriptstyle 0.003$ & $0.095 \pm \scriptstyle 0.011$ & $0.101 \pm \scriptstyle 0.003$ & $0.096 \pm \scriptstyle 0.007$ & $0.110 \pm \scriptstyle 0.000$ & $0.104 \pm \scriptstyle 0.001$ & $0.115 \pm \scriptstyle 0.003$ & $0.103 \pm \scriptstyle 0.008$ & $0.102$ \\
  Strong & $0.056 \pm \scriptstyle 0.002$ & $0.052 \pm \scriptstyle 0.000$ & $0.067 \pm \scriptstyle 0.001$ & $0.060 \pm \scriptstyle 0.002$ & $0.105 \pm \scriptstyle 0.004$ & $0.072 \pm \scriptstyle 0.001$ & $0.079 \pm \scriptstyle 0.002$ & $0.069 \pm \scriptstyle 0.007$ & $0.070$ \\
  \textbf{POI-Histogram-JS} & $0.102 \pm \scriptstyle 0.043$ & $0.077 \pm \scriptstyle 0.003$ & $0.078 \pm \scriptstyle 0.021$ & $0.073 \pm \scriptstyle 0.028$ & $0.085 \pm \scriptstyle 0.000$ & $0.111 \pm \scriptstyle 0.020$ & $0.111 \pm \scriptstyle 0.009$ & $0.084 \pm \scriptstyle 0.001$ & $0.090$ \\
  POI-Histogram-TV & $0.075 \pm \scriptstyle 0.002$ & $0.072 \pm \scriptstyle 0.001$ & $0.090 \pm \scriptstyle 0.005$ & $0.073 \pm \scriptstyle 0.017$ & $0.094 \pm \scriptstyle 0.007$ & $0.108 \pm \scriptstyle 0.009$ & $0.098 \pm \scriptstyle 0.019$ & $0.086 \pm \scriptstyle 0.002$ & \textbf{$0.087$} \\
  \textbf{Cate-Trans-JS} & $0.094 \pm \scriptstyle 0.016$ & $0.091 \pm \scriptstyle 0.038$ & $0.094 \pm \scriptstyle 0.017$ & $0.120 \pm \scriptstyle 0.067$ & $0.117 \pm \scriptstyle 0.015$ & $0.118 \pm \scriptstyle 0.040$ & $0.162 \pm \scriptstyle 0.066$ & $0.155 \pm \scriptstyle 0.020$ & $0.119 $ \\
  Cate-Trans-TV & $0.103 \pm \scriptstyle 0.021$ & $0.085 \pm \scriptstyle 0.028$ & $0.079 \pm \scriptstyle 0.008$ & $0.079 \pm \scriptstyle 0.027$ & $0.098 \pm \scriptstyle 0.020$ & $0.126 \pm \scriptstyle 0.032$ & $0.119 \pm \scriptstyle 0.014$ & $0.125 \pm \scriptstyle 0.033$ & $0.102$ \\
  \textbf{Cate-JS} & $0.171 \pm \scriptstyle 0.079$ & $0.231 \pm \scriptstyle 0.189$ & $0.126 \pm \scriptstyle 0.026$ & $0.164 \pm \scriptstyle 0.110$ & $0.194 \pm \scriptstyle 0.102$ & $0.168 \pm \scriptstyle 0.060$ & $0.170 \pm \scriptstyle 0.054$ & $0.202 \pm \scriptstyle 0.081$ & $0.178$ \\
  Cate-TV & $0.113 \pm \scriptstyle 0.001$ & $0.173 \pm \scriptstyle 0.036$ & $0.114 \pm \scriptstyle 0.040$ & $0.179 \pm \scriptstyle 0.003$ & $0.112 \pm \scriptstyle 0.008$ & $0.147 \pm \scriptstyle 0.022$ & $0.132 \pm \scriptstyle 0.020$ & $0.138 \pm \scriptstyle 0.029$ & $0.138$ \\
  \midrule
  \multicolumn{10}{l}{\textit{Travel Distance JSD}} \\
  Baseline & $0.502 \pm \scriptstyle 0.008$ & $0.583 \pm \scriptstyle 0.015$ & $0.549 \pm \scriptstyle 0.003$ & $0.568 \pm \scriptstyle 0.010$ & $0.595 \pm \scriptstyle 0.008$ & $0.523 \pm \scriptstyle 0.012$ & $0.614 \pm \scriptstyle 0.012$ & $0.583 \pm \scriptstyle 0.034$ & $0.565$ \\
  Strong & $0.207 \pm \scriptstyle 0.020$ & $0.232 \pm \scriptstyle 0.015$ & $0.227 \pm \scriptstyle 0.006$ & $0.285 \pm \scriptstyle 0.027$ & $0.342 \pm \scriptstyle 0.047$ & $0.205 \pm \scriptstyle 0.032$ & $0.279 \pm \scriptstyle 0.008$ & $0.248 \pm \scriptstyle 0.027$ & $0.253$ \\
  \textbf{POI-Histogram-JS} & $0.285 \pm \scriptstyle 0.011$ & $0.203 \pm \scriptstyle 0.009$ & $0.145 \pm \scriptstyle 0.010$ & $0.051 \pm \scriptstyle 0.001$ & $0.359 \pm \scriptstyle 0.003$ & $0.156 \pm \scriptstyle 0.015$ & $0.366 \pm \scriptstyle 0.004$ & $0.126 \pm \scriptstyle 0.005$ & $0.211 $ \\
  POI-Histogram-TV & $0.238 \pm \scriptstyle 0.128$ & $0.135 \pm \scriptstyle 0.058$ & $0.185 \pm \scriptstyle 0.101$ & $0.201 \pm \scriptstyle 0.059$ & $0.143 \pm \scriptstyle 0.021$ & $0.182 \pm \scriptstyle 0.013$ & $0.389 \pm \scriptstyle 0.199$ & $0.208 \pm \scriptstyle 0.074$ & \textbf{$0.210$} \\
  \textbf{Cate-Trans-JS} & $0.276 \pm \scriptstyle 0.181$ & $0.364 \pm \scriptstyle 0.144$ & $0.288 \pm \scriptstyle 0.192$ & $0.371 \pm \scriptstyle 0.069$ & $0.244 \pm \scriptstyle 0.162$ & $0.302 \pm \scriptstyle 0.180$ & $0.322 \pm \scriptstyle 0.127$ & $0.441 \pm \scriptstyle 0.191$ & $0.326 $ \\
  Cate-Trans-TV & $0.193 \pm \scriptstyle 0.115$ & $0.392 \pm \scriptstyle 0.149$ & $0.153 \pm \scriptstyle 0.138$ & $0.409 \pm \scriptstyle 0.066$ & $0.243 \pm \scriptstyle 0.152$ & $0.346 \pm \scriptstyle 0.123$ & $0.411 \pm \scriptstyle 0.128$ & $0.315 \pm \scriptstyle 0.092$ & $0.308$ \\
  \textbf{Cate-JS} & $0.264 \pm \scriptstyle 0.029$ & $0.251 \pm \scriptstyle 0.113$ & $0.381 \pm \scriptstyle 0.088$ & $0.407 \pm \scriptstyle 0.069$ & $0.320 \pm \scriptstyle 0.025$ & $0.303 \pm \scriptstyle 0.126$ & $0.605 \pm \scriptstyle 0.170$ & $0.439 \pm \scriptstyle 0.180$ & $0.371$ \\
  Cate-TV & $0.303 \pm \scriptstyle 0.149$ & $0.326 \pm \scriptstyle 0.012$ & $0.396 \pm \scriptstyle 0.091$ & $0.381 \pm \scriptstyle 0.076$ & $0.398 \pm \scriptstyle 0.048$ & $0.355 \pm \scriptstyle 0.023$ & $0.448 \pm \scriptstyle 0.187$ & $0.316 \pm \scriptstyle 0.128$ & $0.365$ \\
  \midrule
  \multicolumn{10}{l}{\textit{Trip JSD}} \\
  Baseline & $0.099 \pm \scriptstyle 0.001$ & $0.094 \pm \scriptstyle 0.001$ & $0.102 \pm \scriptstyle 0.001$ & $0.095 \pm \scriptstyle 0.001$ & $0.103 \pm \scriptstyle 0.000$ & $0.113 \pm \scriptstyle 0.001$ & $0.123 \pm \scriptstyle 0.001$ & $0.119 \pm \scriptstyle 0.001$ & $0.106$ \\
  Strong & $0.088 \pm \scriptstyle 0.002$ & $0.076 \pm \scriptstyle 0.001$ & $0.079 \pm \scriptstyle 0.001$ & $0.074 \pm \scriptstyle 0.001$ & $0.111 \pm \scriptstyle 0.003$ & $0.093 \pm \scriptstyle 0.003$ & $0.100 \pm \scriptstyle 0.002$ & $0.098 \pm \scriptstyle 0.001$ & $0.090$ \\
  \textbf{POI-Histogram-JS} & $0.112 \pm \scriptstyle 0.002$ & $0.076 \pm \scriptstyle 0.001$ & $0.071 \pm \scriptstyle 0.000$ & $0.063 \pm \scriptstyle 0.001$ & $0.089 \pm \scriptstyle 0.002$ & $0.097 \pm \scriptstyle 0.002$ & $0.111 \pm \scriptstyle 0.000$ & $0.081 \pm \scriptstyle 0.009$ & \textbf{$0.087$} \\
  POI-Histogram-TV & $0.079 \pm \scriptstyle 0.010$ & $0.060 \pm \scriptstyle 0.010$ & $0.132 \pm \scriptstyle 0.023$ & $0.109 \pm \scriptstyle 0.022$ & $0.085 \pm \scriptstyle 0.024$ & $0.127 \pm \scriptstyle 0.049$ & $0.092 \pm \scriptstyle 0.013$ & $0.114 \pm \scriptstyle 0.006$ & $0.100$ \\
  \textbf{Cate-Trans-JS} & $0.097 \pm \scriptstyle 0.017$ & $0.102 \pm \scriptstyle 0.063$ & $0.086 \pm \scriptstyle 0.010$ & $0.118 \pm \scriptstyle 0.057$ & $0.124 \pm \scriptstyle 0.007$ & $0.148 \pm \scriptstyle 0.049$ & $0.153 \pm \scriptstyle 0.057$ & $0.160 \pm \scriptstyle 0.041$ & $0.123 $ \\
  Cate-Trans-TV & $0.106 \pm \scriptstyle 0.029$ & $0.107 \pm \scriptstyle 0.055$ & $0.077 \pm \scriptstyle 0.009$ & $0.090 \pm \scriptstyle 0.027$ & $0.094 \pm \scriptstyle 0.027$ & $0.157 \pm \scriptstyle 0.048$ & $0.099 \pm \scriptstyle 0.015$ & $0.117 \pm \scriptstyle 0.021$ & $0.106 $ \\
  \textbf{Cate-JS} & $0.282 \pm \scriptstyle 0.043$ & $0.315 \pm \scriptstyle 0.104$ & $0.255 \pm \scriptstyle 0.020$ & $0.304 \pm \scriptstyle 0.044$ & $0.340 \pm \scriptstyle 0.114$ & $0.349 \pm \scriptstyle 0.059$ & $0.272 \pm \scriptstyle 0.043$ & $0.347 \pm \scriptstyle 0.035$ & $0.308$ \\
  Cate-TV & $0.203 \pm \scriptstyle 0.152$ & $0.219 \pm \scriptstyle 0.152$ & $0.199 \pm \scriptstyle 0.158$ & $0.244 \pm \scriptstyle 0.135$ & $0.192 \pm \scriptstyle 0.109$ & $0.239 \pm \scriptstyle 0.177$ & $0.212 \pm \scriptstyle 0.153$ & $0.194 \pm \scriptstyle 0.084$ & $0.213 $ \\
  \midrule
  \multicolumn{10}{l}{\textit{POI Frequency JSD}} \\
  Baseline & $0.216 \pm \scriptstyle 0.003$ & $0.198 \pm \scriptstyle 0.008$ & $0.222 \pm \scriptstyle 0.005$ & $0.218 \pm \scriptstyle 0.013$ & $0.232 \pm \scriptstyle 0.012$ & $0.233 \pm \scriptstyle 0.011$ & $0.236 \pm \scriptstyle 0.010$ & $0.234 \pm \scriptstyle 0.012$ & $0.224$ \\
  Strong & $0.160 \pm \scriptstyle 0.001$ & $0.135 \pm \scriptstyle 0.000$ & $0.172 \pm \scriptstyle 0.000$ & $0.168 \pm \scriptstyle 0.002$ & $0.216 \pm \scriptstyle 0.002$ & $0.194 \pm \scriptstyle 0.003$ & $0.202 \pm \scriptstyle 0.001$ & $0.190 \pm \scriptstyle 0.006$ & $0.180$ \\
  \textbf{POI-Histogram-JS} & $0.211 \pm \scriptstyle 0.004$ & $0.130 \pm \scriptstyle 0.002$ & $0.159 \pm \scriptstyle 0.001$ & $0.134 \pm \scriptstyle 0.000$ & $0.177 \pm \scriptstyle 0.001$ & $0.166 \pm \scriptstyle 0.001$ & $0.217 \pm \scriptstyle 0.000$ & $0.164 \pm \scriptstyle 0.000$ & \textbf{$0.170$} \\
  POI-Histogram-TV & $0.132 \pm \scriptstyle 0.001$ & $0.106 \pm \scriptstyle 0.016$ & $0.223 \pm \scriptstyle 0.011$ & $0.185 \pm \scriptstyle 0.014$ & $0.154 \pm \scriptstyle 0.006$ & $0.225 \pm \scriptstyle 0.008$ & $0.175 \pm \scriptstyle 0.006$ & $0.187 \pm \scriptstyle 0.016$ & $0.173$ \\
  \textbf{Cate-Trans-JS} & $0.194 \pm \scriptstyle 0.031$ & $0.171 \pm \scriptstyle 0.053$ & $0.199 \pm \scriptstyle 0.018$ & $0.242 \pm \scriptstyle 0.088$ & $0.229 \pm \scriptstyle 0.029$ & $0.239 \pm \scriptstyle 0.057$ & $0.256 \pm \scriptstyle 0.050$ & $0.269 \pm \scriptstyle 0.040$ & $0.225 $ \\
  Cate-Trans-TV & $0.193 \pm \scriptstyle 0.034$ & $0.163 \pm \scriptstyle 0.035$ & $0.176 \pm \scriptstyle 0.023$ & $0.189 \pm \scriptstyle 0.034$ & $0.188 \pm \scriptstyle 0.036$ & $0.237 \pm \scriptstyle 0.014$ & $0.232 \pm \scriptstyle 0.015$ & $0.238 \pm \scriptstyle 0.025$ & $0.202$ \\
  \textbf{Cate-JS} & $0.332 \pm \scriptstyle 0.106$ & $0.364 \pm \scriptstyle 0.221$ & $0.270 \pm \scriptstyle 0.044$ & $0.324 \pm \scriptstyle 0.131$ & $0.308 \pm \scriptstyle 0.072$ & $0.304 \pm \scriptstyle 0.050$ & $0.313 \pm \scriptstyle 0.055$ & $0.351 \pm \scriptstyle 0.085$ & $0.321$ \\
  Cate-TV & $0.198 \pm \scriptstyle 0.028$ & $0.240 \pm \scriptstyle 0.061$ & $0.233 \pm \scriptstyle 0.034$ & $0.290 \pm \scriptstyle 0.037$ & $0.215 \pm \scriptstyle 0.009$ & $0.260 \pm \scriptstyle 0.053$ & $0.243 \pm \scriptstyle 0.016$ & $0.259 \pm \scriptstyle 0.012$ & $0.242$ \\
  \bottomrule
  \end{tabular}}
\end{table*}

In our feature choice experiments (Section~\ref{sec:phi-results}), we experiment with different choices of aggregate feature $\phi$.
We present results per demographic group across states, trajectory statistics, feature choices, and divergence functions used in the aggregate loss, with Table~\ref{tab:js_cate_trans_pergroup} for Virginia and Table~\ref{tab:ca_cate_ablation_pergroup} for California.

\subsection{RQ3: Next-POI Prediction}
\label{sec:app-downstream}
\begin{table*}[t]
\centering
\caption{Next POI prediction performance by demographic group (Virginia). Metrics: Accuracy (Acc), Hit Rate at 10 (HR@10), Normalized Discounted Cumulative Gain at 10 (NDCG@10), and Geographic Error (GeoError in km). We report the mean across demographic groups in the last column.}
\label{tab:next_poi_prediction}
\resizebox{0.72\textwidth}{!}{%
\begin{tabular}{lccccccccc}
\toprule
Setup & <30, M & <30, F & 30--40, M & 30--40, F & 40--50, M & 40--50, F & >50, M & >50, F & Avg \\
\midrule
\multicolumn{10}{l}{\textit{Accuracy}} \\
Real & 0.5866 & 0.6377 & 0.5491 & 0.5677 & 0.5702 & 0.5225 & 0.5650 & 0.5212 & 0.5650 \\
Baseline & 0.4803 & 0.4729 & 0.5197 & 0.5553 & 0.4740 & 0.4404 & 0.4168 & 0.4371 & 0.4746 \\
FullRank-JS & 0.5777 & 0.6272 & 0.5386 & 0.5601 & 0.5482 & 0.5108 & 0.5348 & 0.5084 & 0.5507 \\
FullRank-TV & 0.5134 & 0.6292 & 0.4690 & 0.5637 & 0.5502 & 0.5030 & 0.4846 & 0.5087 & 0.5277 \\
\midrule
\multicolumn{10}{l}{\textit{HR@10}} \\
Real & 0.7109 & 0.7454 & 0.6512 & 0.6765 & 0.6675 & 0.6179 & 0.6690 & 0.6662 & 0.6756 \\
Baseline & 0.7080 & 0.7401 & 0.6474 & 0.6774 & 0.6655 & 0.6195 & 0.6690 & 0.6456 & 0.6716 \\
FullRank-JS & 0.7083 & 0.7403 & 0.6490 & 0.6838 & 0.6664 & 0.6189 & 0.7944 & 0.6435 & 0.6881 \\
FullRank-TV & 0.7082 & 0.7399 & 0.6602 & 0.6761 & 0.6665 & 0.6179 & 0.6740 & 0.6432 & 0.6733 \\
\midrule
\multicolumn{10}{l}{\textit{NDCG@10}} \\
Real & 0.6644 & 0.7045 & 0.6123 & 0.6362 & 0.6307 & 0.5826 & 0.6305 & 0.6065 & 0.6335 \\
Baseline & 0.6240 & 0.6414 & 0.6002 & 0.6320 & 0.5941 & 0.5528 & 0.5758 & 0.5681 & 0.5986 \\
FullRank-JS & 0.6600 & 0.6985 & 0.6078 & 0.6367 & 0.6220 & 0.5786 & 0.6971 & 0.5935 & 0.6368 \\
FullRank-TV & 0.6358 & 0.6990 & 0.5858 & 0.6346 & 0.6225 & 0.5753 & 0.6025 & 0.5935 & 0.6186 \\
\midrule
\multicolumn{10}{l}{\textit{GeoError (km)}} \\
Real & 54.450 & 41.809 & 54.522 & 42.904 & 44.313 & 53.611 & 37.896 & 63.881 & 49.173 \\
Baseline & 62.444 & 51.200 & 61.521 & 43.848 & 49.217 & 58.119 & 46.336 & 71.006 & 55.461 \\
FullRank-JS & 56.462 & 41.939 & 56.676 & 42.105 & 49.152 & 53.823 & 41.634 & 66.009 & 50.975 \\
FullRank-TV & 59.974 & 43.213 & 57.643 & 43.741 & 44.788 & 56.812 & 43.273 & 65.108 & 51.819 \\
\bottomrule
\end{tabular}}
\end{table*}

\begin{table*}[t]
\centering
\caption{Next POI prediction performance by demographic group (California). Metrics: Accuracy (Acc), Hit Rate at 10 (HR@10), Normalized Discounted Cumulative Gain at 10 (NDCG@10), and Geographic Error (GeoError in km). We report the mean across demographic groups in the last column.}
\label{tab:next_poi_prediction_ca}
\resizebox{0.72\textwidth}{!}{%
\begin{tabular}{lccccccccc}
\toprule
Setup & <30, M & <30, F & 30--40, M & 30--40, F & 40--50, M & 40--50, F & >50, M & >50, F & Avg \\
\midrule
\multicolumn{10}{l}{\textit{Accuracy}} \\
Real & 0.6100 & 0.6568 & 0.6043 & 0.6179 & 0.6202 & 0.5535 & 0.5955 & 0.5809 & 0.6049 \\
Baseline & 0.5990 & 0.6249 & 0.5964 & 0.5619 & 0.5867 & 0.5386 & 0.5301 & 0.5088 & 0.5808 \\
FullRank-JS & 0.5991 & 0.6453 & 0.5965 & 0.6126 & 0.6173 & 0.5395 & 0.5826 & 0.5598 & 0.5941 \\
FullRank-TV & 0.5995 & 0.6436 & 0.5966 & 0.6128 & 0.6168 & 0.5386 & 0.5820 & 0.5596 & 0.5937 \\
\midrule
\multicolumn{10}{l}{\textit{HR@10}} \\
Real & 0.7621 & 0.8250 & 0.7461 & 0.7668 & 0.7595 & 0.7450 & 0.7330 & 0.7582 & 0.7620 \\
Baseline & 0.7365 & 0.8051 & 0.7356 & 0.7098 & 0.7031 & 0.6820 & 0.6885 & 0.7031 & 0.7205 \\
FullRank-JS & 0.7570 & 0.8254 & 0.7360 & 0.7595 & 0.7535 & 0.7320 & 0.7287 & 0.7434 & 0.7544 \\
FullRank-TV & 0.7574 & 0.8251 & 0.7358 & 0.7599 & 0.7531 & 0.7317 & 0.7281 & 0.7433 & 0.7543 \\
\midrule
\multicolumn{10}{l}{\textit{NDCG@10}} \\
Real & 0.7047 & 0.7629 & 0.6913 & 0.7101 & 0.7068 & 0.6715 & 0.6809 & 0.6901 & 0.7023 \\
Baseline & 0.6583 & 0.7286 & 0.6841 & 0.6551 & 0.6827 & 0.6105 & 0.6336 & 0.6353 & 0.6610 \\
FullRank-JS & 0.6986 & 0.7588 & 0.6843 & 0.7052 & 0.7031 & 0.6608 & 0.6746 & 0.6755 & 0.6951 \\
FullRank-TV & 0.6988 & 0.7579 & 0.6842 & 0.7054 & 0.7028 & 0.6604 & 0.6742 & 0.6754 & 0.6949 \\
\midrule
\multicolumn{10}{l}{\textit{GeoError (km)}} \\
Real & 117.343 & 97.673 & 114.092 & 121.952 & 113.627 & 142.743 & 132.345 & 130.166 & 121.243 \\
Baseline & 121.405 & 105.970 & 119.804 & 128.754 & 118.470 & 149.748 & 141.835 & 138.043 & 128.004 \\
FullRank-JS & 119.350 & 101.833 & 115.804 & 123.597 & 114.396 & 147.401 & 136.890 & 134.893 & 124.271 \\
FullRank-TV & 119.466 & 101.912 & 115.774 & 123.620 & 114.626 & 147.552 & 137.118 & 134.895 & 124.370 \\
\bottomrule
\end{tabular}}
\end{table*}

\subsubsection{Training next-POI predictors.}
We evaluate whether improvements in aggregate distributional fidelity translate into downstream utility by training a next-POI predictor on synthetic trajectories from each setup, then evaluating on held-out real trajectories within each demographic group.
We train an LSTM-based next-POI predictor~\citep{hochreiter1997long} for each experimental setup.
The model takes a variable-length history of visited POIs as input, embeds each POI token via a learnable embedding layer (dimension 256), processes the sequence through a 2-layer LSTM (hidden dimension 256, dropout 0.2), and predicts the next POI from the final hidden state via a linear output layer over the full POI vocabulary.
Training uses cross-entropy loss with the Adam optimizer (learning rate $10^{-4}$) and early stopping based on the validation loss on a held-out validation split.

We compare training the next-POI predictor on synthetic trajectories generated by \NAME{} to training it on synthetic trajectories generated by the Baseline model (not conditioned on demographics) and training it on \textit{real} trajectories (``Real'').
For this \NAME{} experiment, we use the FullRank partition and POI counts as the aggregate feature, and we experiment with the two different divergence measures in the aggregate loss (JS and TV). 
For Baseline and \NAME{}, we generate synthetic trajectories using training-set conditions: 30{,}000 trajectories for Virginia and 50{,}000 for California.
Each synthetic trajectory is generated conditioned on a demographic label $d$ and home/work coordinates $z$ sampled from training individuals.
The generated trajectories are then partitioned by their demographic label into $K=8$ demographic groups (4 age bins $\times$ 2 genders).
For each group, we train a separate LSTM on the synthetic trajectories assigned to each demographic group.
For Real, we train a separate LSTM on real trajectories from each demographic group, only using trajectories from train individuals (from the user-level train split described in Section~\ref{sec:model-training}).

At evaluation time, we test each group-specific LSTM on held-out \emph{real} trajectories from the corresponding demographic group (from the test split described in Section~\ref{sec:model-training}).
This protocol measures whether the demographic-specific patterns captured in the synthetic data transfer to predicting real mobility behavior within each group.

\subsubsection{Next-POI prediction results.}
Tables~\ref{tab:next_poi_prediction} and~\ref{tab:next_poi_prediction_ca} provide the full per-group breakdowns for Virginia and California, respectively.
We report standard ranking metrics: Accuracy (top-1 prediction), Hit Rate at 10 (HR@10), Normalized Discounted Cumulative Gain at 10 (NDCG@10), and geographic error (GeoError in km; lower is better).

In Virginia (Table~\ref{tab:next_poi_prediction}), FullRank-JS substantially improves Accuracy from 0.475 (Baseline) to 0.551, nearly matching the Real upper bound of 0.565.
GeoError decreases from 55.5\,km to 51.0\,km (Real: 49.2\,km), indicating that recovered demographic-specific patterns improve geographic prediction quality.
HR@10 is high across all setups (0.67--0.69), suggesting that placing the true next POI in the top-10 is relatively easy; the meaningful differences emerge in top-1 Accuracy and ranking quality (NDCG@10), where \NAME{} shows the most substantial gains over Baseline.

In California (Table~\ref{tab:next_poi_prediction_ca}), the same pattern holds despite the larger geographic scale and higher baseline GeoError distances.
FullRank-JS improves Accuracy from 0.581 (Baseline) to 0.594, approaching Real (0.605), and reduces GeoError from 128.0\,km to 124.3\,km (Real: 121.2\,km).
HR@10 and NDCG@10 both show consistent improvements across demographic groups, with FullRank-JS closing approximately 80\% of the Baseline-to-Real gap on NDCG@10 (0.661 $\to$ 0.695 vs.\ 0.702).
The gains are remarkably consistent between JS and TV divergence variants, suggesting that the downstream benefits stem primarily from the aggregate supervision framework rather than the specific divergence choice.

\end{document}